\documentclass[11pt]{article}

\usepackage[a4paper,margin=1in]{geometry}

\usepackage[utf8]{inputenc} 
\usepackage[T1]{fontenc}    

\usepackage{microtype}
\usepackage{graphicx}
\usepackage{subcaption}
\usepackage{booktabs} 
\usepackage{algorithm}
\usepackage{algpseudocode} 

\usepackage{url}            
\usepackage{booktabs}       
\usepackage{amsfonts}       
\usepackage{nicefrac}       
\usepackage{microtype}      
\usepackage{xcolor}         

\usepackage{amsmath}
\usepackage{amssymb}
\usepackage{mathtools}
\usepackage{amsthm}
\usepackage{thmtools}
\usepackage{enumerate}
\usepackage{enumitem}
\usepackage{xargs}
\usepackage{ifthen}
\usepackage{xparse}

\usepackage[authoryear,round]{natbib}
\usepackage{hyperref}
\usepackage[capitalize,noabbrev]{cleveref}

\newcommand{\rmd}{\mathrm{d}}
\newcommand{\rme}{\mathrm{e}}
\newcommand{\eqsp}{\,}

\newcommand{\eqdef}{:=}
\newcommand{\pf}{p}
\newcommand{\pb}{\overleftarrow{p}}
\newcommand{\hell}{\mathrm{H}}

\newcommand{\pidata}{\pi_{\mathrm{data}}}
\newcommand{\rmL}{\mathrm{L}}
\def\maxsw{Max SW}

\newcommand{\Xora}{\overrightarrow{\mathbf{X}}}
\newcommand{\Xola}{\overleftarrow{\mathbf{X}}}
\newcommand{\Xbar}{\bar{\mathbf{X}}}
\newcommand{\X}{\mathbf{X}}
\newcommand{\x}{\mathbf{x}}
\newcommand{\Id}{\mathbf{I}}

\newcommand{\Z}{\mathbf{Z}}
\newcommand{\Q}{\mathrm{Q}}

\newcommand{\Y}{\mathbf{Y}}
\newcommandx{\pdata}[1][1=]{\ifthenelse{\equal{#1}{}}{p_{\operatorname{data}}}{p_{\operatorname{data}}\left(#1\right)}}
\NewDocumentCommand{\gaussiand}{m m o}{\IfValueTF{#3}{\mathcal{N}\left(#3; #1, #2\right)}{\mathcal{N}\left(#1, #2\right)}}
\newcommandx{\gaussMarg}[2][2=]{\ifthenelse{\equal{#2}{}}{\varphi_{#1}}{\varphi_{#1}\left(#2\right)}}
\NewDocumentCommand{\unif}{m o}{\IfValueTF{#2}{\mathcal{U}_{#1}\left(#2\right)}{\mathcal{U}_{#1}}}
\newcommand{\indi}[2]{\mathrm{1}_{#1}\left(#2\right)}
\newcommand{\PE}[2]{\mathbb{E}_{#1}\left[#2\right]}
\newcommand{\CPE}[3]{\mathbb{E}_{#1}\left[#3 \middle | #2\right]}
\newcommand{\CPV}[3]{\mathbb{V}_{#1}\left[#3 \middle | #2\right]}
\newcommand{\set}[1]{\mathrm{#1}}
\newcommand{\borelians}[1]{\mathcal{B}(#1)}
\newcommand{\ball}[2]{\operatorname{B}_{#2}(#1)}
\newcommand{\sphere}[1]{\mathbb{S}_{#1}}
\NewDocumentCommand{\minconst}{m m o}{\varepsilon_{#2|#1}\IfValueTF{#3}{^{(#3)}}{}}
\NewDocumentCommand{\minconstnor}{m m}{\tilde{\varepsilon}_{#2|#1}}
\NewDocumentCommand{\minmeas}{m m o}{\IfValueTF{#3}{\upsilon_{#2| #1}\left(#3\right)}{\upsilon_{#2|#1}}}
\NewDocumentCommand{\minmeasintcte}{m m}{\mathrm{I}_{#2|#1}}
\def\rset{\mathbb{R}}
\def\rsetpos{\mathbb{R}_{+}}

\def\nsetpos{\mathbb{N}^{*}}
\def\xdim{d}
\def\tensprod{\otimes}
\newcommand{\mt}[1]{{#1}^{\operatorname{T}}}
\newcommand{\trace}[1]{\operatorname{Tr}\left(#1\right)}
\NewDocumentCommand{\timechange}{o}{\IfValueTF{#1}{\tau\left(#1\right)}{\tau}}

\newcommand{\ve}[1]{{#1}^{\operatorname{VE}}}
\newcommandx{\fwdmarg}[2][2=]{\ifthenelse{\equal{#2}{}}{\pf_{#1}}{\pf_{#1}\left(#2\right)}}
\newcommandx{\fwdtrans}[4][3=,4=]{\ifthenelse{\equal{#3}{}}{\pf_{#2|#1}}{\pf_{#2|#1}\left(#4|#3\right)}}
\newcommand{\fwdmean}[2]{\operatorname{m}_{#2|#1}}
\newcommand{\fwdstd}[2]{\sigma_{#2|#1}}
\newcommand{\fwdvar}[2]{\sigma^2_{#2|#1}}
\newcommandx{\bwdmarg}[2][2=]{\ifthenelse{\equal{#2}{}}{\pb_{#1}}{\pb_{#1}\left(#2\right)}}
\NewDocumentCommand{\bwdker}{m o o o}{\IfValueTF{#3}{\IfValueTF{#2}{\Q_{#2|#1}(#3; #4)}{\Q_{#1}(#3; #4)}}{\IfValueTF{#2}{\Q_{#2|#1}}{\Q_{#1}}}}
\NewDocumentCommand{\bwdkercomp}{O{} m m o o}{
\mathbb{Q}^{#1}_{#2:#3}\IfValueTF{#4}{(#4; #5)}{}}
\NewDocumentCommand{\pgen}{O{}}{\hat{\pi}^{\theta}}

\NewDocumentCommand{\bwdkerdiscr}{m o o o}{\IfValueTF{#3}{\IfValueTF{#2}{\bar{\Q}_{#2|#1}(#3; #4)}{\bar{\Q}_{#1}(#3; #4)}}{\IfValueTF{#2}{\bar{\Q}_{#2|#1}}{\bar{\Q}_{#1}}}}

\NewDocumentCommand{\approxbwdker}{O{\theta} m o o o o}{\IfValueTF{#4}{\IfValueTF{#3}{\Q^{#1}_{#3|#2}(#4; #5)}{\Q^{#1}_{#2}(#4; #5)}}{\IfValueTF{#3}{\Q^{#1}_{#3| #2}}{\Q^{#1}_{#2}}}}
\NewDocumentCommand{\approxbwdkercomp}{O{\theta} O{} m m o o}{
\mathbb{Q}^{\theta\IfValueTF{#2}{#2}{}}_{#3:#4}\IfValueTF{#5}{(#5; #6)}{}}
\newcommandx{\bridge}[4][3=,4=]{\ifthenelse{\equal{#3}{}}{\pf_{#2 | #1}}{\pf_{#2 | #1}\left(#4 |#3\right)}}
\newcommand{\bridgemean}[3]{\operatorname{m}_{#2|#1}^{#3}}
\newcommand{\bridgevar}[2]{\gamma^2_{#2|#1}}

\NewDocumentCommand{\score}{o o}{\IfValueTF{#2}{\operatorname{S}_{#1}\left(#2\right)}{\operatorname{S}_{#1}}}
\NewDocumentCommand{\tildescore}{o o}{\IfValueTF{#2}{\operatorname{\tilde{S}}_{#1}\left(#2\right)}{\operatorname{\tilde{S}}_{#1}}}
\NewDocumentCommand{\jscore}{o o}{\IfValueTF{#2}{\nabla \operatorname{S}_{#1}\left(#2\right)}{\nabla \operatorname{S}_{#1}}}
\NewDocumentCommand{\scorenet}{o o O{\theta} }{\IfValueTF{#1}{\operatorname{s}_{#3}\left(#2, #1\right)}{\operatorname{s}_{#3}}}
\NewDocumentCommand{\denoiser}{o o}{\IfValueTF{#2}{\operatorname{D}_{#1}(#2)}{\operatorname{D}_{#1}}}
\NewDocumentCommand{\pset}{m O{}}{\IfValueTF{#2}{\mathcal{P}_{#2}(#1)}{\mathcal{P}\left(#1\right)}}

\newcommand{\dotprod}[2]{\left<#1, #2\right>}

\def\eqlaw{\stackrel{\mathcal L}{=}}
\NewDocumentCommand{\csp}{m m o}{\IfValueTF{#3}{\mathcal{C}^{#1}\left(#2; #3\right)}{\mathcal{C}^{#1}\left(#2\right)}}
\newcommand{\bdd}[1]{\mathcal{M}_{\textrm{b}}\left(#1\right)}
\NewDocumentCommand{\lpsp}{m m o}{\IfValueTF{#3}{\mathcal{L}_{#1}\left(#2; #3\right)}{\mathcal{L}_{#1}\left(#2\right)}}
\newcommand{\divergence}[1]{\operatorname{div}(#1)}

\def\refmeas{\pi_{\infty}}
\def\isvp{\alpha}
\newcommand{\noisesch}[1]{\beta_{#1}}
\newcommand{\bwdnoisesch}[1]{\bar{\beta}_{#1}}

\newcommand{\brownvar}[1]{\operatorname{B}_{#1}}
\newcommand{\dbrown}[1]{\rmd\!\operatorname{B}_{#1}}

\newcommand{\scoreloss}[1]{\mathcal{L}(#1)}
\NewDocumentCommand{\pnorm}{O{2} m}{\left\|#2\right\|_{#1}}
\newcommand{\minimum}[2]{\min\left\{#1, #2\right\}}
\NewDocumentCommand{\lyapunov}{m o}{%
  \operatorname{V}_{#1}%
  \IfValueT{#2}{\!\left(#2\right)}%
}

\NewDocumentCommand{\ctescorenorm}{o m}{%
  \IfValueTF{#1}
    {\tilde{\gamma}_{#2}}
    {\gamma_{#2}}
}

\NewDocumentCommand{\ctescoreoffset}{o m}{%
  \IfValueTF{#1}
    {\tilde{\kappa}_{#2}}
    {\kappa_{#2}}
}

\newcommand{\cstediscr}{\mathrm{D}^{(0)}_{s:t}}
\newcommand{\cstediscrx}{\mathrm{D}^{(1)}_{s:t}}
\newcommand{\cstediscrgamma}[1][s:t]{\overline{\Gamma}^{(0)}_{#1}}
\newcommand{\cstediscrgammax}[1][s:t]{\overline{\Gamma}^{(1)}_{#1}}

\newcommand{\ctegrowth}[1]{\operatorname{C}_{g, #1}}
\newcommand{\ctegrowthHess}[1]{\operatorname{C}_{H, #1}}
\newcommand{\maxctegrowthHess}[1]{\overline{\operatorname{C}}_{H, #1}}

\newcommand{\ctejscore}[1]{\operatorname{C}_{#1}}
\def\powerjscore{m}
\def\ctescorepoly{\operatorname{C}_{s}}
\newcommand{\ctestable}[1]{\Gamma^{s}_{#1}}
\def\gridn{N}
\NewDocumentCommand{\grid}{o o}{
    {\IfValueTF{#1}{t_{\IfValueTF{#2}{#1:#2}{#1}}}{\mathcal{T}}}
}
\NewDocumentCommand{\multlyap}{m m o}{\lambda_{#2|#1}^{\IfValueTF{#3}{(#3)}{}}}
\NewDocumentCommand{\maxmultlyap}{m m o}{\overline{\lambda}_{#2 |#1}^{\IfValueTF{#3}{(#3)}{}}}
\NewDocumentCommand{\biaslyap}{m m o}{\mathrm{K}_{#2|#1}^{\IfValueTF{#3}{(#3)}{}}}
\NewDocumentCommand{\maxbiaslyap}{m m o}{\overline{\mathrm{K}}_{#2|#1}^{\IfValueTF{#3}{(#3)}{}}}

\NewDocumentCommand{\pmixtime}{o o}{%
  \IfValueTF{#1}
    {\bar\alpha_{#2|#1}}
    {\bar\alpha}
}

\NewDocumentCommand{\errscore}{m m}{\mathcal{E}_{#1}^{#2}}

\NewDocumentCommand{\bmetric}{s m m}{%
  \IfBooleanTF{#1}
    {\rho_{b^\star}\!\left(#2,#3\right)} 
    {\rho_{b}\!\left(#2,#3\right)}       
}
\newcommand{\normEc}[1]{\left\|#1\right\|}
\newcommand{\normTV}[1]{\left\|#1\right\|_{\mathrm{TV}}}
\newcommand{\normFr}[1]{\left\|#1\right\|_{\operatorname{F}}}
\newcommand{\hellinger}[2]{\hell\left(#1, #2\right)}
\newcommand{\kl}[2]{\mathrm{KL}\left(#1 \middle \| #2\right)}
\NewDocumentCommand{\wasserstein}{O{} O{} m m}{%
  \mathcal{W}_{#1}%
  \IfNoValueF{#2}{^{#2}}%
  \!\left(#3,#4\right)%
}
\newcommand{\norminfty}[1]{\left\|#1\right\|_{\infty}}

\NewDocumentCommand{\Cdiscr}{m o}{%
  C^{\rm discr}_{#1}%
  \IfValueT{#2}{\!\left[#2\right]}%
}

\NewDocumentCommand{\Cmix}{m o}{%
  C^{\rm mix}_{#1}%
  \IfValueT{#2}{\!\left[#2\right]}%
}

\NewDocumentCommand{\Cnet}{m o}{%
  C^{\rm net}_{#1}%
  \IfValueT{#2}{\!\left[#2\right]}%
}

\newcommand{\ie}{\emph{i.e.}}
\newcommand{\eg}{\emph{e.g.}}

\theoremstyle{plain}
\newtheorem{theorem}{Theorem}[section]
\newtheorem{proposition}[theorem]{Proposition}
\newtheorem{lemma}[theorem]{Lemma}
\newtheorem{corollary}[theorem]{Corollary}

\theoremstyle{definition}

\newtheorem{assumption}[theorem]{Assumption}

\theoremstyle{remark}
\newtheorem{remark}[theorem]{Remark}
\newtheorem*{proposition*}{Proposition}

\newlist{assumplist}{enumerate}{1}
\setlist[assumplist]{label={(\Roman*)},ref=\theassumption~(\Roman*)}

\crefalias{assumplisti}{assumption}

\graphicspath{{fig/}}

\title{On Forgetting and Stability of Score-based Generative models}

\author{%
  Stanislas Strasman\textsuperscript{1} \quad
  Gabriel Victorino Cardoso\textsuperscript{2} \quad
  Sylvain Le Corff\textsuperscript{1} \\
  Vincent Lemaire\textsuperscript{1} \quad
  Antonio Ocello\textsuperscript{3}
  \\[0.75em]
  \textsuperscript{1} Sorbonne Université and Université Paris Cité, CNRS, LPSM, F-75005 Paris, France\\
  \textsuperscript{2} Center for Statistics and Images, Mines Paris, PSL University, Fontaineableau, France\\
  \textsuperscript{3} CREST, ENSAE Paris, Institut Polytechnique de Paris, Palaiseau, France\\
}

\date{}

\begin{document}

\maketitle

\begin{abstract}
  Understanding the stability and long-time behavior of generative models is a fundamental problem in modern machine learning.
This paper provides quantitative bounds on the sampling error of  score-based generative models by leveraging stability and forgetting properties of the Markov chain associated with the reverse-time dynamics.
Under weak assumptions, we provide the two structural properties to ensure the propagation of initialization and discretization errors of the backward process: a Lyapunov drift condition and a Doeblin-type minorization condition.
A practical consequence is quantitative stability of the sampling procedure, as the  reverse  diffusion  dynamics  induces a contraction mechanism along the sampling  trajectory.
Our results clarify the role of stochastic dynamics in score-based models and provide a principled framework for analyzing propagation of errors in such approaches. 
\end{abstract}

\section{Introduction}
Score-based generative models  have recently emerged as a unifying framework for high-dimensional generative modeling, with remarkable empirical success across a wide range of practical problems. They are rooted in the framework of denoising diffusion probabilistic models (DDPMs) \citep[see, \eg,][]{ho2020denoising,song2021score,song2022solving}, and of score-matching techniques introduced in \citet{hyvarinen2005estimation,Vincent}. These models can be interpreted as defining stochastic samplers whose associated Markov semigroups are designed to transform a simple reference distribution into a complex target distribution through a sequence of  score-driven updates. This viewpoint naturally raises fundamental questions regarding the stability and convergence of the resulting sampling procedures.

Theoretical analysis of score-based generative models has been explored, in particular in Wasserstein distances, with non-asymptotic bounds highlighting initialization and discretization errors, along with score matching training errors.
Following first guarantees in Wasserstein distances, such as  \citep{chen2022improved}, most recent works use strong assumptions of the target distributions, such as log-concavity \citep{bruno2025wasserstein,strasman2025an,yu2025advancing}, bounded support \citep{beyler2025convergence} or regularity conditions on the score function \citep{tang2024contractive,gao2025wasserstein}. Some recent works have moved beyond the traditional constraints of log-concavity in data distributions and 
regularity assumptions for the score function  and proposed relaxed
convexity assumptions    \citep{silveri2025beyond,strasman2025wasserstein}. In the most recent results, the data distribution is modeled as a perturbed strongly log-concave law and use one-sided Lipschitz conditions which ensure that the taget distribution has sub-Gaussian tails, and consequently,
all its polynomial moments are finite \citep{brigati2025heat,stephanovitch2025regularity}.
\paragraph{Contributions.}
These works do not leverage the forgetting properties of Markov chains to characterize how initialization, score-approximation, and discretization errors propagate along the sampling trajectory. In this paper, we propose an analysis grounded in the Harris stability framework for Markov processes \citep{HairerMattingly2008}, which provides quantitative, exponential decay of the dependence on the initial condition. We show that the backward Markov chain satisfies a Lyapunov drift condition and a localized minorization condition.
Following \citet{HairerMattingly2008},  these properties yield contraction of the associated Markov semigroup in a Wasserstein-type sense. A key in our work is a dissipativity condition on the score function outside a compact set. Notably, this holds \emph{without early stopping}, and yields explicit constants that quantitatively transfer structural assumptions on the data score to the diffusion flow. Such dissipativity conditions are standard in the sampling and optimization literature and have been used extensively to establish stability and convergence of Langevin-type dynamics \citep{meyn2009markov, loecherbach2015ergodicity}. Importantly, this requirement is weaker than the usual assumptions on the target distribution that underpins much of the existing theoretical analysis of diffusion-based samplers \citep{meyn1993stabilityI,meyn2012markov}.

Moreover, our assumptions do not impose global Lipschitz continuity of the score function, nor any form of one-sided Lipschitz regularity.
They allow for polynomial growth of the Jacobian of the score, thereby accommodating target distributions with polynomially super-Gaussian tails. As a result, the proposed framework covers nonconvex and multimodal target distributions. 
\section{Background}
\subsection{Notations}
We denote by $\Id_\xdim$ the $d\times d$ identity matrix, by $\dot f$ the derivative of a function $f$, and by $\pset{\rset^\xdim}[p]$ the set of probability distributions $\mu$ on $\rset^\xdim$ such that $\PE{\mu}{\normEc{\X}^{p}} < \infty$. For a vector $v\in\rset^\xdim$, we write $v^{\otimes 2}\eqdef v \mt{v}$. We write $X\eqlaw Y$ when the random variables $X$ and $Y$ have the same distribution.
We denote by $\csp{k}{\rset^{d_0}}[\rset^{d_1}]$ the space of $k$-times continuously differentiable functions from $\rset^{d_0}$ to $\rset^{d_1}$ and by $\bdd{\rset^{d_0}}$ the space of bounded measurable functions from $\rset^{d_0}$ to $\rset$.
When $d_1=1$, we simplify the notation to $\csp{k}{\rset^{d_0}}$. When there is no ambiguity, $\normEc{\cdot}$ denotes the Euclidean norm on $\rset^\xdim$, and $\normFr{\cdot}$ denotes the Frobenius norm for matrices.

\paragraph{Score-based generative models.} 
Score-based generative models (SGMs) have emerged as a flexible framework for sampling from high-dimensional probability distributions using diffusion processes and a time-reversal argument to bridge the data distribution $\pidata \in \pset{\rset^{\xdim}}$ to a Gaussian distribution $\refmeas \in \pset{\rset^{\xdim}}$. The forward process is known as the noising process and is solution to the following stochastic differential equation (SDE) on a fixed time horizon $t \in [0,T]$: $\Xora_0 \sim \pidata$ and
\begin{align} \label{eq:forward_SDE}
\rmd \Xora_t &= -\isvp \noisesch{t} \Xora_t \rmd t + \sqrt{2\noisesch{t}} \rmd \brownvar{t}\eqsp,
\end{align}
where $(\brownvar{t})_{t \in [0,T]}$ is a $d$-dimensional Brownian motion, $\isvp \geq 0$ and $ [0,T] \to \noisesch{t} \in \rsetpos$ is a non-decreasing noise schedule.  For all $0\leq t\leq T$,  $\fwdmarg{t}$ denotes  the probability density function of the random vector $\Xora_{t}$. Specific choices of $(\isvp,\noisesch{})$ recover either the \emph{Variance Exploding (VE)} formulation with $\isvp=0$ with $\noisesch{t}=\fwdstd{t}{0} \dot{\fwdstd{t}{0}}$ and $\fwdvar{t}{0} = 2 \int_0^t \noisesch{s} \rmd s$ (see, \eg, \citet{song2019generative}) or the Variance Preserving (VP) with $\isvp=1$ \citep{dickstein2015,ho2020denoising}. More broadly, throughout this work we refer to the case $\isvp>0$ as \emph{Variance Preserving (VP)}: the forward diffusion \eqref{eq:forward_SDE} is an Ornstein--Uhlenbeck process with stationary distribution $\pi_{\infty} \sim \gaussiand{0}{\isvp^{-1}\Id_\xdim}$ \citep{strasman2025an}.

The time-reversal of \eqref{eq:forward_SDE} admits a diffusion representation \citep{haussmann1986time}
\begin{align*}
(\Xola_t)_{t \in [0,T]} \eqlaw (\Xora_{T-t})_{t \in [0,T]} \eqsp,
\end{align*}
where $(\Xola_t)_{t\in[0,T]}$ is called the \emph{backward process} (or \emph{reverse-time process}) and is defined as the solution to the SDE: $\Xola_0 \sim  \fwdmarg{T}$\footnote{With abuse of notation we identify probability distribution with their density with respect to the Lebesgue measure.} and
\begin{align} \label{eq:backward_SDE}
\rmd \Xola_t = \left( \isvp \bwdnoisesch{t} \Xola_t + 2  \bwdnoisesch{t}\,\score[T-t][\Xola_t]\right) \rmd t + \sqrt{2 \bwdnoisesch{t}} \rmd \brownvar{t}\eqsp,
\end{align}
with $\score[T-t][\x] \eqdef \nabla \log \fwdmarg{T-t}(\x)$ and $\bwdnoisesch{t} \eqdef \noisesch{T-t}$. We denote by $(\bwdker{s}[t])_{0 \le s \le t \le T}$ the Markov semigroup associated with the backward diffusion \eqref{eq:backward_SDE}, \ie,
\begin{align}
    \label{eq:backward_semigroup}
    \bwdker{s}[t] f(\x) \eqdef \CPE{}{\Xola_s = \x}{f(\Xola_t)}
    \eqsp,
    \quad \text{ for } f\in\bdd{\rset^{\xdim}}
    \eqsp,\eqsp
    \x\in\rset^\xdim
    \eqsp.
\end{align}
\paragraph{Backward process approximations.} Sampling from $\pidata$ amounts to simulating the marginal $\Xola_T$ starting from $\Xola_0\sim \fwdmarg{T}$ as $\pidata = \fwdmarg{T} \bwdker{T}$\footnote{ With abuse of notation we refer, for any $t>0$ to $\bwdker{t}$ as $\bwdker{0}[t]$.}. 
In practice, to turn this identity into a generative model, the following three approximations are required:
%
\begin{enumerate}
    \item \textbf{Mixing time error}. The distribution $\fwdmarg{T}$ is given by a Gaussian convolution of $\pidata$ and is therefore typically intractable (see~\Cref{lem:forward_process_law}). It is usually replaced by a simple reference distribution $\refmeas$. In the VE case, we usually take $\refmeas=\gaussiand{0}{\fwdvar{0}{T} \Id_\xdim}$, while in the VP case we take the stationary distribution of \eqref{eq:forward_SDE}, \ie, $\refmeas=\gaussiand{0}{ \isvp^{-1} \Id_{\xdim}}$. This approximation is independent of $\pidata$ and induces an \emph{initialization bias}.
    \item \textbf{Score approximation}. The drift term of \eqref{eq:backward_SDE} involves the data-dependent score function $\score[t][\x]$, which is, generally intractable. It is approximated using a neural network $\scorenet: (0,T] \times \rset^\xdim \mapsto \rset^\xdim$ parameterized by $\theta \in \Theta$,  and trained to minimize the conditional score matching loss \citep{Vincent}:
\begin{align*} 
\scoreloss{\theta} = \PE{}{ \normEc{\scorenet[\tau][\Xora_{\tau}][\theta] - \nabla \log \fwdtrans{0}{\tau}[\Xora_{0}][\Xora_{\tau}]}^2_2} \eqsp,
\end{align*}
with $\tau$ uniformly distributed over $(0,T]$, independent of $\Xora_{0}$, and $\Xora_{\tau} \sim \fwdtrans{0}{\tau}[\Xora_{0}][\cdot]$.
\item \textbf{Discretization error}. 
The transition kernels $\bwdker{s}[t]$ of the backward diffusion do not admit closed-form expressions in general and must be approximated numerically. Since \eqref{eq:backward_SDE} is time-inhomogeneous through the scalar schedule $u\mapsto\bwdnoisesch{u}$, we discretize it using a \emph{time-changed Euler--Maruyama scheme} with step sizes
$\Delta_k \eqdef \int_{t_k}^{t_{k+1}} \bwdnoisesch{u}\,\rmd u$.
This choice exactly integrates the noise schedule so that the scheme matches exactly the covariance of the Brownian increment over each step.
\end{enumerate}
Combining the three approximations above yields a discrete-time Markov chain
$(\Xbar_{t_k}^{\theta})_{0\le k\le N}$ over a finite discretization  grid $\grid_\gridn = (\grid[0], \cdots, \grid[\gridn])$, for $\gridn \in \nsetpos$, defined by
$\Xbar_{t_0}^{\theta}\sim\refmeas$ and for $k=0,\dots,N-1$, by the recursion
\begin{align}\label{eq:time_changed_euler}
    \Xbar_{t_{k+1}}^{\theta}
    =
    \Xbar_{t_k}^{\theta}
    +
    \Delta_k \Big(
    \isvp \Xbar_{t_k}^{\theta}
    +
    2 \scorenet[T-t_k][\Xbar_{t_k}^{\theta}][\theta]
    \Big) 
    + \sqrt{2\Delta_k} \xi_k
    \eqsp,
\end{align}
with $(\xi_k)_{k\ge0}$ an i.i.d.\ sequence of standard Gaussian random vectors in $\rset^\xdim$.
For $k\in\{0,\dots,N-1\}$, let $\approxbwdker{t_k}[t_{k+1}]$ be the one-step Markov kernel
of the discrete-time scheme \eqref{eq:time_changed_euler}
\begin{align*}
\approxbwdker{t_k}[t_{k+1}] f(\x)
\eqdef \CPE{}{ \Xbar^\theta_{t_k}=\x}{f(\Xbar^\theta_{t_{k+1}})}
\eqsp,
\quad \text{ for }
f\in\bdd{\rset^{\xdim}}\eqsp,\eqsp\x \in \rset^\xdim
\eqsp.
\end{align*}
\paragraph{Generation error analysis.}
Fix a discretization grid $\grid_\gridn=(t_0,\dots,t_N)$ of $[0,T]$ and for integers $0\le k<\ell\le N$, define the composed SGM kernel as
$\approxbwdkercomp[\theta][]{k}{\ell}$ by
\begin{align*}
\approxbwdkercomp[\theta][]{k}{\ell} f(\x)
\eqdef
\int f(\x_\ell)\prod_{r=k}^{\ell-1}\approxbwdker{t_r}[t_{r+1}](\x_r,\rmd \x_{r+1}) \eqsp,
\end{align*}
with $ \x_k\eqdef \x \in \rset^\xdim$ and $f\in\bdd{\rset^\xdim}$.
We denote the resulting generated distribution of the SGM at time $t_k$
\begin{align} \label{eq:generative_model}
\pgen_k \eqdef \refmeas \approxbwdkercomp[\theta][]{0}{k} 
\end{align}
and keep the dependency on $\grid$ implicit. Our goal is then to compare the target distribution $\pidata$ with the distribution $\pgen_N$ produced by the SGM.

\section{Forgetting of the backward process} \label{sec:forgetting}
\paragraph{Harris-type stability and weighted total variation distance.}
Our analysis relies on a classical stability framework for Markov processes known as Harris theory. At a high level, Harris-type results establish quantitative exponential forgetting of the initial condition by combining two ingredients:
(i) a \emph{Lyapunov drift condition} ensuring that the dynamics is pulled back toward a central region of the state space, and
(ii) a \emph{localized minorization (Doeblin-type) condition} ensuring a uniform mixing component when the process visits that region. Together, these two properties yield contraction of the Markov semigroup in a sense that we make precise below.

To quantify the forgetting property, we work with the weighted total variation distance $\bmetric{\cdot}{\cdot}$ defined, for $\mu_1,\mu_2\in\pset{\rset^\xdim}$, by
\begin{align}\label{def:rho_b}
\bmetric{\mu_1}{\mu_2}
\eqdef
\int_{\rset^\xdim} \big(1+b \lyapunov{2}(\x)\big) |\mu_1-\mu_2|(\rmd \x)
\end{align}
where $\lyapunov{2}(\x) \eqdef \|\x\|^2$, $b>0$, and $|\mu_1-\mu_2|$ denotes the total variation measure of the signed measure $\mu_1-\mu_2$.
This metric is standard in Harris-type theorems for unbounded state spaces \citep{HairerMattingly2008}. Moreover, any contraction estimate proved in $\bmetric{\cdot}{\cdot}$  yields quantitative guarantees in both total variation and Wasserstein distance. In particular, for any $\mu_1,\mu_2\in \pset{\rset^\xdim}[2]$, we have (see \cref{lem:ineq:metrics})
\begin{align*}
   \normTV{\mu_1 -\mu_2}  \leq \frac{1}{2} \bmetric{\mu_1}{\mu_2}
    \quad
    \text{ and }
    \quad
\wasserstein[2][2]{\mu_1}{\mu_2}   \leq \frac{2}{b} \bmetric{\mu_1}{\mu_2} \eqsp.
\end{align*}

\paragraph{Assumptions on the data distribution.}
To establish forgetting of the backward Markov chain \eqref{eq:backward_SDE}, we require mild assumptions on the data. The data distribution $\pidata$ admits a density $\fwdmarg{0}\in\csp{2}{\rset^{\xdim}}$ with respect to the Lebesgue measure. Moreover, $p_0$ satisfies the following regularity and Lyapunov-type conditions.
\begin{assumption}\label[assumption]{assump:p0}
There exist constants $\ctescorenorm{0}>\isvp/2$, $~\ctescoreoffset{0}\ge 0$, 
$~\ctejscore{0}>0$, and $p\geq 1$ such that the following assumptions hold:
\begin{assumplist}
    \item \label{assump:p0:score}%
    $\displaystyle
    \dotprod{\score[0][\x]}{\x}
    \le -\ctescorenorm{0} \normEc{\x}^2 + \ctescoreoffset{0}\eqsp,
    $
    for $\x\in\rset^{\xdim}$;
    \item \label{assump:p0:hess}%
    $\displaystyle
    \normFr{
        \jscore[0][\x]
    }
    \le \ctejscore{0}\big(1+\normEc{\x}^{\powerjscore}\big)\eqsp,
    $
    for $\x\in \rset^{\xdim}$.
\end{assumplist}
\end{assumption}
\cref{assump:p0} imposed on the initial distribution play a crucial role in the proof of the main theorem. In particular, \cref{assump:p0:score} ensures that the target distribution $\pidata$ exhibits sub-Gaussian tail behavior (\cref{lem:pdata-sub-gaussian}), which in turn guarantees sufficient integrability and moment bounds. Note that the condition $\ctescorenorm{0}>\isvp/2$ is not restrictive in practice as it simply amounts to calibrating the stationary distribution of \cref{eq:forward_SDE}. Moreover, \cref{assump:p0:hess} provides a control on the growth of the score function $\score[0][x]=\nabla \log \fwdmarg{0}(x)$ as well as on the quantity $\jscore[0][x]+\score[0][x]^{\tensprod 2}$, which plays a central role in the control of the Froebenius norm of the score function for any $t \in [0,T]$ (\cref{cor:bound_hessian}).
    A detailed comparison with recent works
is provided in~\Cref{sec:comparison-literature}.
%
{
}
%
\paragraph{Lyapunov drift for the backward chain.} We prove the first property of Harris-type stability results, the \emph{Lyapunov drift} inequality for the backward Markov semigroup $(\bwdker{s}[t])_{0\le s\le t\le T}$. Such an estimate prevents trajectories from drifting to infinity by ensuring a quadratic pull toward the center of the state space. A key technical step is to show that the dissipativity of the data score at initial time propagates along the forward diffusion. Under \cref{assump:p0:score}, there exist positive continuous functions $t\mapsto \ctescorenorm{t}$ and $t\mapsto \ctescoreoffset{t}$ such that
\begin{align} \label{eq:score_dissipativity_time_t}
    \dotprod{\score[t][\x]}{\x}
    \le - \ctescorenorm{t}\,\normEc{\x}^2 + \ctescoreoffset{t}
    \quad
    \text{ for } t\in[0,T]\eqsp, \eqsp\x\in\rset^\xdim
    \eqsp.
\end{align}
We refer to \cref{prop:lyapunov_stability} for the precise statement and explicit constants. Combining \eqref{eq:score_dissipativity_time_t} with the infinitesimal generator of \eqref{eq:backward_SDE} and Dynkin's formula (followed by Grönwall's inequality) yields the following Lyapunov drift inequality.
\begin{proposition} \label[proposition]{prop:backward_drift_lyapunov}
Suppose that~\cref{assump:p0} holds and let $\lyapunov{\ell}[\x] \eqdef \normEc{\x}^{\ell}$ for $\ell\ge2$.
Then, there exist continuous functions $\ctescorenorm[t]{\cdot,\ell},\ctescoreoffset[t]{\cdot,\ell}:[0,T]\to\rsetpos$ such that, for all $0\le s<t\le T$ and all $\x\in\rset^\xdim$,
\begin{align*}
\bwdker{s}[t]\lyapunov{\ell}[\x]
\le
\multlyap{s}{t}[\ell] \lyapunov{\ell}[\x]
+
\biaslyap{s}{t}[\ell]\eqsp,
\end{align*}
where $\multlyap{s}{t}[\ell]\eqdef \exp \left(-\int_s^t \ctescorenorm[t]{v,\ell}\,\rmd v\right)$
and $\biaslyap{s}{t}[\ell]\eqdef \int_s^t \exp\left(-\int_u^t \ctescorenorm[t]{v,\ell}\,\rmd v\right)\ctescoreoffset[t]{u,\ell}\,\rmd u$.

In particular, for $\ell=2$, we have $\ctescorenorm[t]{t,2} = 2 \bwdnoisesch{t}\big(2\ctescorenorm{T-t}-\isvp\big)$ and $\ctescoreoffset[t]{t,2} = 2\,\bwdnoisesch{t}\big(2\ctescoreoffset{T-t}+\xdim\big)$.
\end{proposition}
\begin{proof}
    The proof is deferred to Appendix \ref{subapp:lyapunov_semigroup}.
\end{proof}
In the sequel, Harris-type contraction is stated with the quadratic Lyapunov function $\lyapunov{2}(\x)=\normEc{x}^2$; we nevertheless prove the more general $\ell$-moment drift bound since it is useful for moment and integrability estimates.

\paragraph{Minorization for the backward chain.} In addition to the Lyapunov drift condition, a second key property in Harris-type arguments is a minorization property, which ensures sufficient mixing of the backward dynamics. In unbounded state-spaces such as $\mathbb{R}^\xdim$, a global Doeblin condition is generally too strong, so one instead localizes it to an appropriate ``small'' or ``petite'' set \citep{meyn2009markov, loecherbach2015ergodicity}. We establish a localized Doeblin condition in~\cref{prop:minorization_main}. This guarantees that, on an appropriate subset, $\bwdker{s}[t]$ dominates a state-independent measure and hence can forget its initial condition. 
\begin{proposition} \label[proposition]{prop:minorization_main}
Let $0\le s<t\le T$ and suppose \cref{assump:p0:score} holds. Fix $r>0$ and define the set
\begin{align*}
\mathcal C_r \eqdef \ball{0}{r}
= \{\x\in\rset^\xdim:\ \lyapunov{2}[\x]\le r^2\}\eqsp.
\end{align*}
Then, there exist a probability measure
$\minmeas{t}{s}$ on $\rset^\xdim$ and a constant $\minconst{s}{t}[r] = \minconstnor{s}{t}\exp(-r^2 / \fwdvar{T-t}{T-s}) \in(0,1)$ such that, for all $\x\in\mathcal C_r$ and all
$\set{A}\in\mathcal B(\rset^\xdim)$,
\begin{align}\label{eq:minorization_main}
\bwdker{s}[t](\x,\set{A}) \ge \minconst{s}{t}[r]\,\minmeas{s}{t}[\set{A}]
\eqsp.
\end{align}
\end{proposition}
\begin{proof}
The proof is deferred to Appendix \ref{app:proofs:minorization}, where we also provide explicit constants.
\end{proof}

\paragraph{Forgetting via Harris contraction.}
The drift and minorization conditions above imply a quantitative forgetting property of the backward dynamics in the weighted total variation metric~\eqref{def:rho_b}. Although Harris-type theorems are typically stated for time-homogeneous Markov kernels, we use them here in a \emph{local-in-time} form: for each interval $[s,t]$, the transition kernel $\bwdker{s}[t]$ satisfies a one-step contraction.

\begin{proposition} \label[proposition]{prop:harris_inhomogeneous_contraction}
Fix $0\le s<t\le T$ and suppose that \cref{assump:p0} holds. Set 
\begin{align*}
    r^2 > r^2_c = \frac{2 \biaslyap{s}{t}[2]}{1-\multlyap{s}{t}[2]}\eqsp,
\end{align*}
where $\biaslyap{s}{t}[2]$ and $\multlyap{s}{t}[2]$ are defined in~\Cref{prop:backward_drift_lyapunov}. Then, there exist $\pmixtime[s][t]\in(0,1)$ and $b^r_{s,t}>0$ such that for
 any probability measures $\mu_1,\mu_2$ on $\rset^\xdim$,
\begin{align*}
\rho_{b^r_{s,t}}(\mu_1\bwdker{s}[t],\mu_2\bwdker{s}[t])
\le
\pmixtime[s][t] \rho_{{b}^r_{s,t}}(\mu_1,\mu_2)\eqsp.
\end{align*}


\end{proposition}
\begin{proof}
Under \cref{assump:p0} $\bwdker{s}[t]$ satisfies the drift and minorization properties of \cref{prop:backward_drift_lyapunov} and \cref{prop:minorization_main} which are Assumption 1 and 2 of \citet{HairerMattingly2008} adapted to time-inhomogenous Markov transition kernels.  The conclusion is then an application of Theorem~1.3 of \citet{HairerMattingly2008}.
\end{proof}
\Cref{prop:harris_inhomogeneous_contraction} establishes the contraction of the backward kernel for a range of metrics with different contraction rates. The choice of ideal metric is directly linked with the choice of $r$. Note that following \citet{HairerMattingly2008}, for any $\alpha_0\in(0,\minconst{s}{t}[r])$ and $\eta_0\in(\multlyap{s}{t}[2] + 2\biaslyap{s}{t}[2]/r^2,1)$, we can choose 
\begin{align} \label{eq:alpha_bar}
    b^r_{s,t} = \alpha_0/\biaslyap{s}{t}[2]
    \quad
    \text{ and }
    \quad
    \pmixtime[s][t]
    =\left[1-\left(\minconst{s}{t}[r]-\alpha_0\right)\right] \vee \frac{2+r^2 b^r_{s,t}\eta_0}{2+r^2 b^r_{s,t}}
    \eqsp,
\end{align}
where $\minconst{s}{t}[r]$ is defined in~\Cref{prop:minorization_main}.
The explicit derivation of $r^2_c$ and $\minconst{s}{t}[r]$ is given in~\Cref{sec:constants} for the variance preserving and variance exploding cases,
    together with explicit lower bounds on the corresponding mixing gap $1-\pmixtime[s][t]$ obtained by optimizing the radius parameter in the Harris contraction estimate (cf \Cref{app:explicit-lower-bound}).

\paragraph{Explicit Gaussian contraction. }
When $\pi_{\rm data}=\gaussiand{\mu}{\Sigma}$ the reverse-time SDE \eqref{eq:backward_SDE} has an explicit linear drift and admits closed-form transition kernels (\cref{lem:backward_flow} and \cref{cor:gaussian_forgetting}). This allows to derive explicit contraction rates in the Euclidean norm, and can be used to get sharp estimates in the $2$-Wasserstein distance (\cref{lem:gaussian_W2_dirac}). Note that up to a multiplicative constant, the $2$-Wasserstein distance is controlled by $\rho_b$ (\cref{lem:ineq:metrics}). In particular, for any $0\le s<t\le T$, for all $\x,\x'\in\rset^\xdim$,
\begin{align*}
    \wasserstein[2][]{\delta_\x \bwdker{s}[t]}{\delta_{\x'} \bwdker{s}[t]}
    \le
    \normEc{\rme^{-\isvp \int_{s}^{t} \bwdnoisesch{u} \rmd u} 
    \Sigma_{T-t}\Sigma_{T-s}^{-1}} \normEc{\x-\x'} \eqsp,
\end{align*}
with $\Sigma_t$ defined in \cref{lem:exactscore}. In \cref{lem:gaussian_contraction_euclidean}, we show that in the VE case ($\isvp=0$) the contracting factor is always smaller than $1$ and in the VP case ($\isvp>0$) the same strict contraction holds whenever
$\lambda_{\max}(\Sigma)^{-1} > \isvp$, a condition reminiscent of $\ctescorenorm{0}>\isvp/2$ in \cref{assump:p0:score}. Note that a similar condition was identified in \citet{strasman2025an}.
\section{Stability of SGMs generation} \label{sec:main_result}
Consider a subdivision $\grid_\gridn = \{t_0,\ldots,t_N\}$. In~\Cref{sec:constants} we provide the explicit computations of $\biaslyap{t_k}{t_{k+1}}[2]$,  $\multlyap{t_k}{t_{k+1}}[2]$ and $\minconst{t_k}{t_{k+1}}[r]$ for all $0\leq k\leq N-1$. Therefore, following \citet{HairerMattingly2008}, we  may derive  $b^r_{t_k,t_{k+1}}$ and $\pmixtime[t_k][t_{k+1}]$, $0\leq k\leq N-1$.
Since these constants exist for all $0\leq k\leq N-1$, there exist $\pmixtime_*\in(0,1)$ and $b_*>0$ such that~\Cref{prop:harris_inhomogeneous_contraction} holds uniformly. This allows to establish~\Cref{thm:global-weak-error} which uses the uniform geometrically decaying forgetting property to upperbound the sampling error of SGMs. We precede this result by some assumptions and notations.




\paragraph{Assumptions on the generative model.} The previous section establishes a quantitative forgetting property for the backward process \eqref{eq:backward_SDE}. We now leverage this contraction to control the generation error, \ie, to bound $\bmetric{\pidata}{\pgen_N}$, with $\pgen_N$ the distribution of the generative model defined in \eqref{eq:generative_model}. Since $\bmetric{\cdot}{\cdot}$ weights total variation by $1+b \lyapunov{2}(\cdot)$ with $\lyapunov{2}(x)=\|x\|^2$, bounding the generation error requires a control of polynomial moments for the approximated discretized chain \eqref{eq:time_changed_euler}. 
\begin{assumption}
\label[assumption]{hyp:schema-numerique_moments}
Let $(\Xbar_{t_k}^{\theta})_{0\le k\le N}$ be defined by \eqref{eq:time_changed_euler}. Assume that
\begin{align} \label{eq:schema_num_moment}
\sup_{0\le k\le N}  \mathbb{E} \left[\lyapunov{4p+4}(\Xbar_{t_k}^{\theta})\right] < \infty \eqsp.
\end{align}
\end{assumption}

\begin{remark}
\Cref{hyp:schema-numerique_moments} is mild on a fixed grid. Since $\Xbar_{t_0}^\theta\sim\refmeas$ is Gaussian, it suffices that the learned score has at most polynomial growth along the grid, \ie, there exist $L_\theta\ge0$ and $r\ge1$ such that for all $k\in\{0,\dots,N-1\}$ and all $\x\in\rset^\xdim$,
\begin{align*}
\normEc{\scorenet[T-t_k][\x][\theta]} \le L_\theta\bigl(1+ \normEc{\x}^r\bigr)\eqsp.
\end{align*}
Under this condition, \eqref{eq:schema_num_moment} holds for any $p$ (see \cref{lem:scheme_finite_poly_moments}). 
In particular, the above growth bound is satisfied if $\scorenet[T-t_k][\cdot][\theta]$ is Lipschitz uniformly in $k$, which is the case for most standard neural architectures.
\end{remark}



\paragraph{Stability of SGMs.} The forgetting property proved in the previous section can be leverage to analyse the stability of SGMs sampling. Our main stability result follows from a simple principle: forgetting converts local errors into a telescoping, geometrically weighted sum, as proved in \cref{thm:global-weak-error}. For all $1\leq k \leq N$, define the score approximation term by 
\begin{align*}
  \left\| \errscore{k}{} \right\|_{\rmL_2(\pgen_{k})}  =  \mathbb{E} \left[ \left\| \score[T-t_{k}][\bar \X_{t_{k}}^\theta] - \scorenet[T-t_{k}][\bar \X_{t_{k}}^\theta] \right\|^2 \right]^{\frac{1}{2}} \eqsp. 
\end{align*}
\begin{theorem}
\label[theorem]{thm:global-weak-error}
Suppose \cref{assump:p0} and \cref{hyp:schema-numerique_moments} hold. Let
$\Delta_k \eqdef \int_{t_k}^{t_{k+1}}\bwdnoisesch{u}\,\rmd u$.
Then, there exist $b>0$ and a contraction factor $\pmixtime \in(0,1)$, depending only on $\xdim$, $\ctescorenorm{0}$, $\ctescoreoffset{0}$, and $\noisesch{\cdot}$, a constant $\Cmix{} >0$,
and nonnegative coefficients $\{\Cdiscr{k}\}_{k=0}^{N-1}$,
$\{\Cnet{k}\}_{k=0}^{N-1}$
such that the SGM output $\pgen$ satisfies
\begin{align*}
  \bmetric{\pidata}{\pgen_N} \leq \pmixtime^N  \Lambda(T) C^{\rm mix} 
  + \sum_{k=1}^N \pmixtime^{N-k} \left(
  \Delta_k C^{\mathrm{disc}}_{k-1}
  +\sqrt{\Delta_k} C^{\mathrm{net}}_{k-1} \big\| \errscore{k-1}{} \big\|_{\rmL_2(\pgen_{k-1})} \right)
  \eqsp,
\end{align*}
where in the VE (resp. VP) case, the regime factor $\Lambda(T)$ is given by
\begin{align*}
  \Lambda(T)\eqdef\frac12 \normEc{\Xora_0}_{\rmL_2}\left(\int_0^T \noisesch{s} \rmd s\right)^{-1/2}
  \quad
  \text{(resp. }
  \Lambda(T)\eqdef
  \rme^{- \isvp \int_0^T \noisesch{s} \rmd s}
  \kl{\pi_{\rm data}}{\refmeas}^{1/2}
  \text{ )}
  \eqsp.
\end{align*}
\end{theorem}

{
\begin{proof}
Throughout the proof, the time grid is fixed and we adopt the shorthand notation $t_k \equiv k$ whenever no ambiguity arises  (\eg, $\bwdker{\ell}[k]$ stands for $\bwdker{\grid[\ell]}[\grid[k]]$).
%
We first decompose the global error into an initialization component and a cumulative approximation component:
{\small
  \begin{align*}
    \rho_b(\pidata, \pgen_N) = \bmetric{\fwdmarg{T} \bwdker{T}}{\refmeas \approxbwdkercomp{0}{\gridn}}
    \leq \rho_b(\fwdmarg{T} \bwdker{T}, \refmeas \bwdker{T}) + \rho_b(\refmeas \bwdker{T}, \refmeas \approxbwdkercomp{0}{\gridn})
    \eqsp.
  \end{align*}
}

\emph{Step 1: Initialization error.}
By applying~\Cref{prop:harris_inhomogeneous_contraction} iteratively across the $N$ transitions of the discretization grid, there exist $b>0$ and 
$\pmixtime \in (0,1)$, such that
\begin{align*}
\bmetric{\fwdmarg{T} \bwdker{T}}{ \refmeas \bwdker{T}}   \leq \pmixtime^N \rho_b(\fwdmarg{T}, \refmeas) \eqsp.   
\end{align*}
The quantity $\rho_b(\fwdmarg{T}, \refmeas)$ measures the degree of mixing achieved by the forward diffusion \eqref{eq:forward_SDE} at time $T$.
It is controlled explicitly using a weighted version of Pinsker's inequality (\cref{prop:mixing_error_rho_b}), combined with an explicit estimate of the KL decay $\kl{p_T}{\pi_\infty}$ established in \cref{lem:mixing_time}.
Together, these results yield
{\small
  \begin{align*}
    \bmetric{\fwdmarg{T} \bwdker{T}}{\refmeas \approxbwdkercomp{0}{\gridn}} \leq \pmixtime^N  \Lambda(T) \Cmix{} \eqsp,
  \end{align*}
}
with $\Cmix{}$ defined in \cref{prop:mixing_error_rho_b}.

\emph{Step 2: Discretization and approximation errors.}
We now control the second error term using a telescoping argument along the grid to make appear one step kernel error, e.g. the discrepency between $\approxbwdker{k-1}[k]$ and $\bwdker{k-1}[k]$. Using the notation $\pgen_k$ as defined in \eqref{eq:generative_model}, for each $k\in\{1,\dots,N\}$ define the intermediate measures:
{\small
  \begin{align*}
    \eta_k \eqdef \pgen_{k-1}\,\bwdker{k-1}[k]\,\bwdker{N}[k]
    \eqsp,
    \quad
    \tilde\eta_k \eqdef \pgen_{k-1}\,\approxbwdker{k-1}[k]\,\bwdker{N}[k] \eqsp,
  \end{align*}
}
so that $\eta_1=\refmeas\bwdker{N}$ and $\tilde\eta_N=\refmeas\approxbwdkercomp{0}{N}$. We have,
{\small
  \begin{align*}
    \rho_b(\refmeas\bwdker{N},\,\refmeas\approxbwdkercomp{0}{N})
    \le \sum_{k=1}^N \rho_b(\eta_k,\tilde\eta_k) \eqsp.
  \end{align*}
}
Using again the forgetting property applied to the tail kernel yields $\bwdker{N}[k]$,
{\small
  \begin{align*}
      \bmetric{\refmeas \bwdker{N}}{\refmeas \approxbwdkercomp{0}{N}}
      \leq \sum_{k=1}^N \pmixtime^{N-k}\bmetric{\pgen_{k-1} \bwdker{k-1}[k]}{\pgen_{k-1} \approxbwdker{k-1}[k]} \eqsp.
  \end{align*}
}
Moreover, by definition of $\rho_b$ any $k \in \{1, \dots, N \}$ and $\x \in \rset^\xdim$ (\cref{lem:rho_b_convexity}),
{\small
  \begin{align*}
    \bmetric{\pgen_{k-1} \bwdker{k-1}[k] }{\pgen_{k-1} \approxbwdker{k-1}[k]}
    \leq \int \bmetric{\delta_\x \bwdker{k-1}[k] }{ \delta_\x \approxbwdker{k-1 }[k]}  \pgen_{k-1} ( \rmd \x ) \eqsp.  
  \end{align*}
}
The remaining task is a Dirac one-step estimate control in $\rho_b$. This control is provided by \cref{prop:one_step_discr_error},  whose proof relies on a local Girsanov argument
(\cref{lem:boundgirsanovproof,cor:boundgirsanovproof}).
Consequently, for each $k\in\{1,\dots,N\}$,
{\small
  \begin{align*}
    \bmetric{\pgen_{k-1}  \bwdker{t_{k-1}}[t_k]}{\pgen_{k-1}  \approxbwdker{t_{k-1}}[t_k]}
    \le
    \Delta_k \Cdiscr{k-1} +  \sqrt{\Delta_k} \Cnet{k-1} \normEc{\errscore{k-1}{} }_{\rmL_2(\pgen_{k-1})} \eqsp,
  \end{align*}
}
with $\Cdiscr{k-1}$ and $\Cnet{k-1}$ defined in \cref{cor:one_step_stability_pgen}. Injecting this estimate into the discounted telescoping bound completes the argument.
\end{proof}
}

\Cref{thm:global-weak-error} establishes a quantitative stability result for SGMs, showing that the reverse diffusion dynamics induces an intrinsic contraction mechanism along the sampling trajectory. The first term, decaying as $\pmixtime^N$ captures the initialization error. The second term accounts for time-discretization error and scales linearly with the local step size $\Delta_k$ while the third term quantifies the propagation of score approximation errors and scales as $\sqrt{\Delta_k} \normEc{ \errscore{k-1}{}}_{L_2(\pgen_{k-1})}$. Crucially, both contributions are geometrically discounted by the factor $\pmixtime^{N-k}$ showing that errors incurred early in the trajectory have a vanishing influence on the final distribution and reveals an intrinsic robustness mechanism of SGMs. In particular, for a uniform discretization grid (for any $k \in \{1, \dots, N \}$ let $\Delta = (t_{k} - t_{k-1}) \noisesch{T}$), the accumulated error admits, in the VP case, the form 
{
\begin{align*}
  \bmetric{\pidata}{\pgen_N} \lesssim \pmixtime^N \rme^{-T} + \frac{\Delta}{1 - \pmixtime} \sup_k \Cdiscr{k}
  + \frac{\sqrt{\Delta}}{1-\pmixtime} \sup_k \normEc{ \errscore{k-1}{}}_{L_2(\pgen_{k-1})} \eqsp.
\end{align*}}

Moreover, the stability bound is stated in the weighted total variation distance
which directly controls other probability distance. Indeed, by \cref{lem:ineq:metrics}, controls the total variation distance and the $2$-Wasserstein distance (up to a
multiplicative factors). Consequently,~\Cref{thm:global-weak-error} immediately yields the same decomposition and order in those metrics.

\section{Numerical illustrations}
\label{sec:exp}

We provide numerical illustrations of the forgetting mechanism established in
\cref{sec:forgetting} and \cref{sec:main_result}. The goal of these experiments is not to benchmark a new sampler, but rather to isolate the stability phenomenon predicted by the theory: perturbations introduced early along the reverse trajectory should have a much weaker impact on the final output than perturbations introduced later.

We first consider a controlled Gaussian mixture model (GMM) in dimension $d=50$, for which \cref{assump:p0} can be verified and the effect of perturbations can be directly visualized. The mixture has $25$ components whose means are arranged on a $5\times 5$ grid in the first two coordinates, allowing two-dimensional projections of the samples. We use the VE scheduler of \citet[Equation~5]{karras2022elucidating} with $N=100$, $\sigma_{\min}=0.002$, $\sigma_{\max}=80$, and $\rho=3$. All results are averaged over $20$ independent seeds; additional Gaussian and GMM experiments, together with implementation details and careful descriptions of the experiments, are deferred to \cref{app:experiments}.

\paragraph{Controlled perturbation experiments.}
We consider two complementary perturbation protocols to test the forgetting mechanism along the reverse trajectory. In the initialization experiment, for a perturbation time $t_{\rm bias}$, we first draw samples from the forward marginal $p_{t_{\rm bias}}$ and shift this initial cloud along a fixed direction $u_{\max}$ chosen to approximately maximize a sliced-Wasserstein discrepancy (see \cref{subapp:sensitivity_to_init}). Starting from these perturbed samples, we run the discretized reverse dynamics down to the data time and compare the final samples to $\pidata$ using the maximum sliced-Wasserstein distance. The results for this initialization perturbation are shown in \cref{fig:gmm:maxsw} (left) and \cref{fig:gmm:pert:init}. We also consider a local score perturbation experiment, where the initialization is left unchanged but the score is perturbed at a single discretization step along the same direction $u_{\max}$, with a time-dependent scaling chosen to make the perturbations comparable across noise levels. The effect of this localized score error is then measured after completing the reverse dynamics (see \cref{subapp:sensitivity_score_error}). The corresponding results are shown in \cref{fig:gmm:maxsw} (right) and \cref{fig:gmm:pert:score}. Together, the two experiments confirm the qualitative prediction of the theory: perturbations introduced early in the reverse trajectory are strongly attenuated, whereas perturbations introduced later remain visible in the final samples.


\begin{figure*}[htbp]
    \centering
    \includegraphics[width=.46\textwidth]{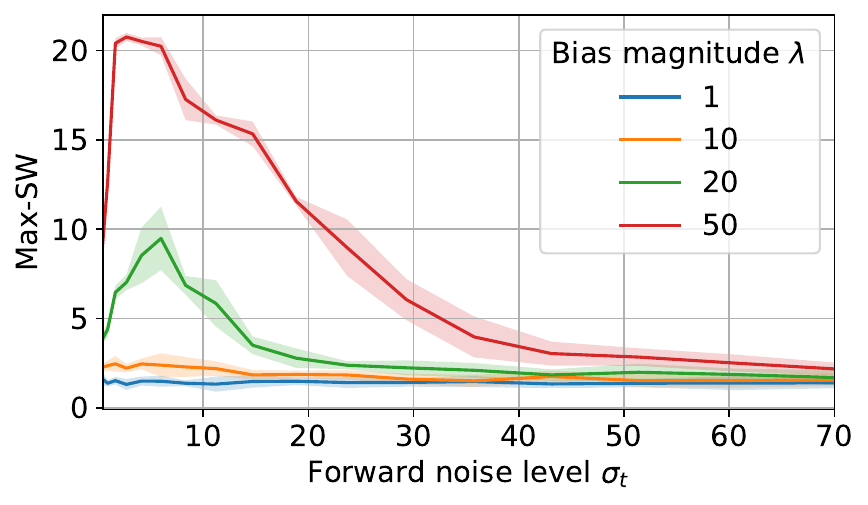}
    \hfill
    \includegraphics[width=.46\textwidth]{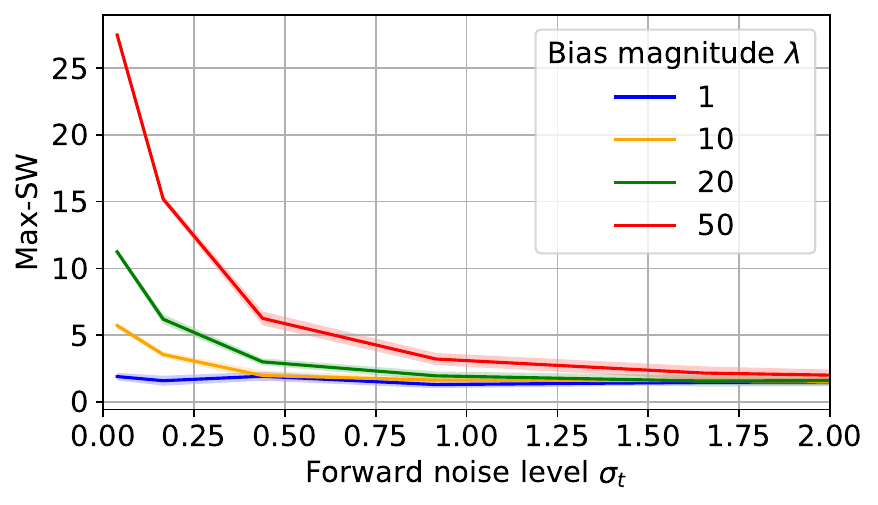}
    \caption{
    {\maxsw} as a function of the noise level and perturbation magnitude $\lambda$.
    \textbf{Left:} initialization perturbation experiment.
    \textbf{Right:} score perturbation experiment.
    We use the forward-time convention ($t=0$ corresponds to the data distribution).
    }
    \label{fig:gmm:maxsw}
    \vspace{-1.5ex}
\end{figure*}

\begin{figure*}[htbp]
    \centering
    \setlength{\tabcolsep}{2pt}
    \renewcommand{\arraystretch}{1}
        \includegraphics[width=0.19\linewidth]{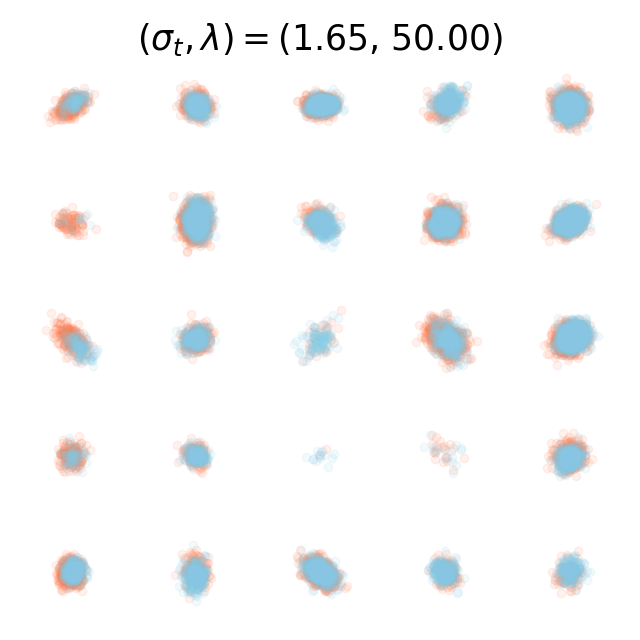}
        \includegraphics[width=0.19\linewidth]{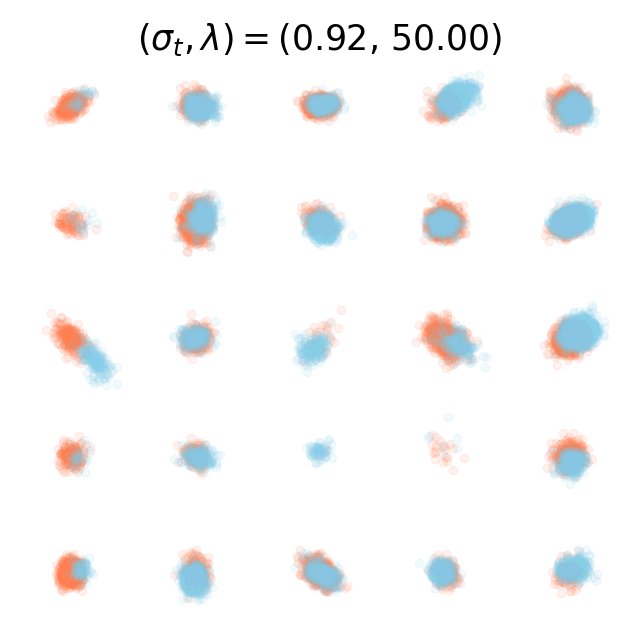}
        \includegraphics[width=0.19\linewidth]{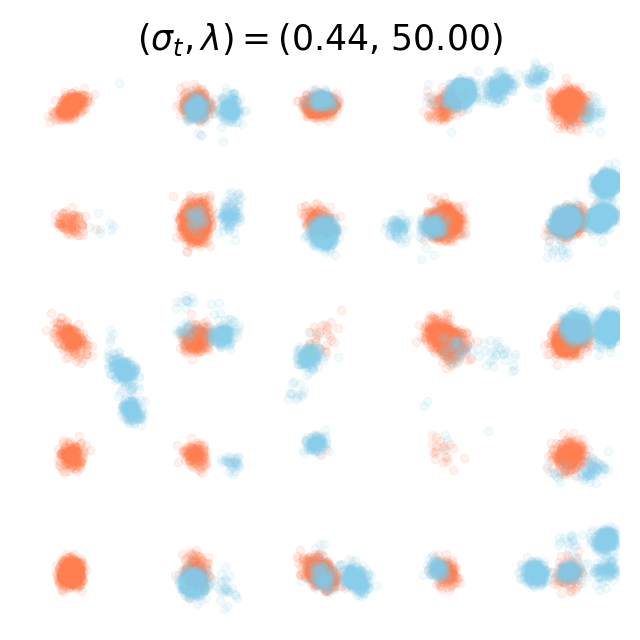}
         \includegraphics[width=0.19\linewidth]{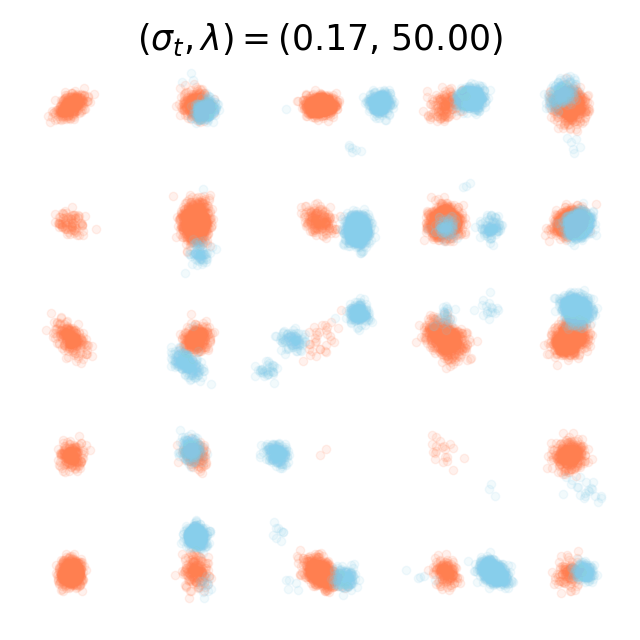}
        \includegraphics[width=0.19\linewidth]{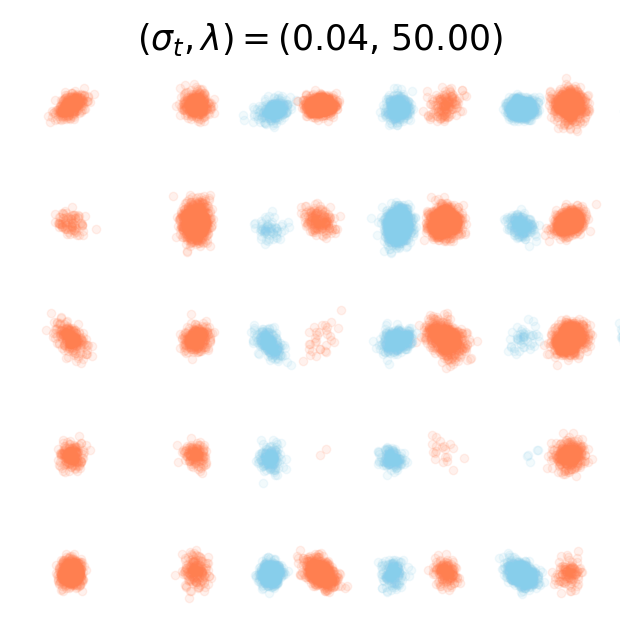}
    \caption{Local perturbation of the score in the GMM case for several noise levels $\sigma_t$ and with $\lambda=50$. Red points are samples from $\pidata$ and blue points the output of the perturbed score experiment.}
    \label{fig:gmm:pert:score}
\end{figure*}

\paragraph{Real-data illustration on CIFAR-10.}
Although the main experiments are intentionally conducted in controlled settings in order to isolate the forgetting effect from optimization and score-learning errors, we also verify that the same qualitative behavior appears on real data. We use the pretrained EDM VP model of \citet{karras2022elucidating} on CIFAR-10 with an Euler--Maruyama sampler. A perturbation direction in image space is chosen by approximately maximizing a sliced-Wasserstein discrepancy between two independently generated batches, and the denoiser is perturbed once at a prescribed sampling step. We then evaluate the resulting generations using FID and maximum sliced-Wasserstein distance over $50{,}000$ samples.

\begin{table}[htbp]
\centering
\caption{CIFAR-10 denoiser perturbation experiment with the pretrained EDM VP model of \citet{karras2022elucidating}. Small step indices correspond to early reverse times, close to Gaussian initialization, while large indices correspond to later times, closer to the data distribution.}
\label{tab:cifar-perturbation}
\begin{tabular}{c|ccccccccc}
\toprule
Step & 0 & 25 & 50 & 70 & 75 & 80 & 85 & 90 & 95 \\
\midrule
FID & 13.3 & 13.0 & 13.1 & 13.6 & 14.4 & 16.4 & 28.3 & 153 & 304 \\
{\maxsw} & 0.011 & 0.014 & 0.016 & 0.033 & 0.044 & 0.060 & 0.094 & 0.266 & 0.779 \\
\bottomrule
\end{tabular}
\end{table}

The same trend is observed: early perturbations have little impact on the final output, whereas late perturbations sharply degrade the generated samples. This is consistent with the geometric discounting predicted by our stability analysis. Additional robustness experiments with $200$ sampling steps and several perturbation magnitudes are reported in \cref{app:experiments}.
{
}
\section{Discussion}
A widely used heuristic in SGMs is that accuracy near the end of the reverse trajectory matters most. This is also reflected in practice: many implementations rely on adaptive time discretizations that allocate finer step sizes close to the data distribution, \ie, near small noise levels \citep{edmkarras}.

Our results provide a theoretical framework to understand this phenomenon. Under a dissipativity condition on the data score at time $0$ and mild polynomial growth control on its Jacobian, we show that the reverse-time diffusion satisfies a Harris-type stability property. In particular, the Markov semigroup of the backward process contracts a weighted total variation distance $\rho_b$. This property is essential to understand the robustness of SGMs. At each step, local perturbations (initialization mismatch, discretization, or score approximation error) propagate with a \emph{geometric discount factor}. As a consequence, errors incurred early along the reverse trajectory have a vanishing influence on the final distribution: the reverse diffusion sampling process \emph{forgets}. This perspective connects SGMs with classical stability tools for Markov processes and provides a principled lens to analyze how noise schedules, numerical integrators, and score errors evolve along the sampling trajectory.
\paragraph{Limitations and perspectives.}

A key limitation is that our minorization argument is localized on a small set. While this yields explicit constants, they can be conservative. In particular, enforcing a single uniform pair $(b,\bar\alpha)$ over an entire refined grid can lead to pessimistic global bounds. Our framework naturally accommodates
time-dependent metrics  but optimizing these
choices along the trajectory to obtain sharper constants is technically challenging and remains an open direction.

Moreover, in this work, we focus on the quadratic Lyapunov function $\lyapunov{2}(x)=\|x\|^2$ but other choices could have been explored. This choice allows for an explicit lower bound on the backward kernel, although alternative choices could be investigated to obtain a tighter minorization properties and therefore sharper sampling error bounds.   The choice of the Lyapunov function has an impact on the constants in the  proposed upper bounds, and investigating the link between these constants and hyperparameter tuning remains an open question for future research.

\clearpage

\section*{Acknowledgements}
The work of GC is supported by the Chaire Geolearning funded by Andra, BNP-Paribas, CCR, and SCOR Foundation.
The work of AO is supported by Hi!\ PARIS and ANR/France 2030 program (ANR-23-IACL-0005). 

\bibliographystyle{plainnat}
\bibliography{bib_ICML}

\appendix
\section*{Appendix}


\paragraph{Table of contents.}

\begin{itemize}
    \item \textbf{\cref{app:discussion_hypotheses}:} Discussion on the assumptions 
    \begin{itemize}
        \item \textbf{\cref{app:assumption_data}:} Regarding the data distribution and its associated score function.
        \item \textbf{\cref{app:regarding_score}:} Regarding the score approximation and the numerical scheme.
    \end{itemize}
    \item \textbf{\cref{app:proba_metrics}:} Probability metrics and useful bounds
    \begin{itemize}
        \item \textbf{\cref{app:proba_def}:} Definitions
        \item \textbf{\cref{app:proba_standard_ineq}:} Standard inequalities
        \item \textbf{\cref{app:proba_rhob}:} Weighted total variation: properties and bounds
    \end{itemize}
    \item \textbf{\cref{app:stability-results}:} Stability properties under Gaussian perturbations
    \begin{itemize}
        \item \textbf{\cref{app:gaussian_rpz}:} Gaussian representation
        \item \textbf{\cref{subsec:gaussian-perturbation}:} Properties of the score under Gaussian perturbation
    \end{itemize}
    \item \textbf{\cref{app:stab+lyapounov}:} Stability of the data assumptions along the diffusion flow and Lyapunov contraction
    \begin{itemize}
        \item \textbf{\cref{subapp:dissipativity}:} Propagation of the dissipativity condition
        \item \textbf{\cref{subapp:lyapunov_semigroup}:} Lyapunov contraction for the backward semigroup (\cref{prop:backward_drift_lyapunov})
        \item \textbf{\cref{subapp:growth}:} Propagation of the growth condition
    \end{itemize}
    \item \textbf{\cref{app:proofs:minorization}:} Localized Doeblin minorization condition for the backward process (Proof of \cref{prop:minorization_main})
    \item \textbf{\cref{sec:constants}:} Quantitative bounds for the Lyapunov and Harris contraction constants
    \begin{itemize}
        \item \textbf{\cref{app:explicit-lower-bound}:} Explicit lower bound on the mixing gap
    \end{itemize}
    \item \textbf{\cref{app:gaussian_framework}:} Gaussian framework: explicit backward kernel and contraction
    \begin{itemize}
    \item \textbf{\cref{subapp:closed_form_gaussian}:} Closed-form score and backward transition
        \item \textbf{\cref{subapp_rate_gaussian}:} Explicit contraction rates for the Euclidean norm
    \end{itemize}
    \item \textbf{\cref{app:general-thm:final-bound}:} Proof of the main stability bound (\cref{thm:global-weak-error})
    \begin{itemize}
        \item \textbf{\cref{subapp:initialization_error}:} Initialization error: mixing properties of the forward process
        \item \textbf{\cref{subapp:discrandapprox}:} One-step discretization and approximation error for the backward kernel
        \item \textbf{\cref{app:technical_lemmas:distances-and-inequalities}:} Technical lemmas for the main proof
    \end{itemize}
    \item \textbf{\cref{app:experiments}:} Numerical illustration of forgetting in controlled settings
    \begin{itemize}
        \item \textbf{\cref{subapp:datasets}:} Synthetic datasets
        \item \textbf{\cref{subapp:sensitivity_to_init}:} Sensitivity to initialization
        \item \textbf{\cref{subapp:sensitivity_score_error}:} Sensitivity to approximation errors
        \item \textbf{\cref{subapp:cifar10}:} Real-data illustration on CIFAR-10
    \end{itemize}
\end{itemize}

\newpage

\section{Discussion on the assumptions}
\label{app:discussion_hypotheses}

This section collects useful results derived from \cref{assump:p0} and \cref{hyp:schema-numerique_moments} that are discussed in the main paper.

\subsection{Comparison with recent literature}
\label{sec:comparison-literature}
{
The assumptions considered in this work are natural in the analysis of sampling and diffusion-based methods, and variants of them have been extensively used in the recent literature.

In particular,~\cref{assump:p0:score} requires a coercive behavior of the score function $\score[0][\x]$ outside a compact set.
This assumption is commonly referred to as a \emph{dissipativity condition} and is standard in the sampling and optimization literature \citep[see, \eg,][]{eberle2016reflection,raginsky2017non,zhang2017hitting,erdogdu2021convergence,erdogdu2022convergence,vacher2025sampling}.
It relaxes global strong convexity of the potential function $x\mapsto -\log \fwdmarg{0} (x)$ \citep{durmus2017nonasymptotic,dalalyan2019user, strasman2025an} therefore cover highly nonconvex and multimodal distributions, which fall outside the scope of works relying on global convexity of the potential. In particular, this includes mixtures of Gaussian distributions, which are shown to satisfy this condition in Section A of \citet{silveri2025beyond} and Section 3.2 of \citet{vacher2025sampling}, covering a broad and practically relevant class of target distributions.

Moreover, our assumptions do not require global Lipschitz continuity of the score, nor any one-sided Lipschitz condition; instead, \cref{assump:p0:hess} allows for polynomial growth of the Jacobian $\jscore[0][\x]$. Consequently, our framework accommodates targets with \emph{polynomially super-Gaussian tails} (\eg, $p(x)\propto \exp(-\|x\|^{\powerjscore})$ with $\powerjscore>2$), for which $\|\jscore[0][\x]\|$ typically diverges as $\|\x\|\to\infty$. These targets that encode strong confinement in some regions of the space are excluded from analyses that impose (one-sided) Lipschitz-type regularity of the score or, equivalently, uniform bounds on $\nabla^2(-\log p)$, as in \citet{chen_kl, chen2023sampling, strasman2025an, silveri2025beyond,gaoW2}. Targets with super-Gaussian tails can be interpreted as enforcing a strong form of regularization: although their support is the whole space, the rapid decay of the density concentrates almost all mass inside a bounded region, effectively mimicking compact support while preserving smoothness.
}

\subsection{Regarding the data distribution and its associated score function.}
\label{app:assumption_data}

\cref{assump:p0:score} is a dissipativity condition on the data score function $\score[0][\cdot]$; in particular, $\pidata$ has sub-Gaussian tails and finite moments of all orders (\cref{lem:pdata-sub-gaussian}). \cref{assump:p0:hess} controls the Jacobian of the score and implies a polynomial growth bound on the score itself (\cref{lem:norm-score-bound}).

\begin{lemma}
    \label[lemma]{lem:pdata-sub-gaussian}

    Suppose that~\cref{assump:p0:score} holds. Then, the data distribution $\pidata$ admits exponential moments of order $\lambda$ for $\lambda<\ctescorenorm{0}/2$, \ie,
    \begin{align}
        \PE{\pidata}{\exp\left(\lambda \normEc{\X}^2\right)} < \infty \eqsp.
    \end{align}
    This implies that $\pidata$ has sub-Gaussian tails.
\end{lemma}
\begin{proof}
    Fix $\lambda < \ctescorenorm{0}/2$. Let \( U(x) = -\log \fwdmarg{0}[x] \).
    Consider the following second-order Taylor expansion:
    \begin{align*}
        U(\x) &= U(0) + \dotprod{\nabla U(0)}{\x} + \frac{1}{2} \dotprod{\x}{ \nabla^2 U(y)\, \x}\eqsp,
    \end{align*}
    for some \( y \in \{ t x : t \in [0,1] \} \). From~\cref{assump:p0:score}, we have that 
    \[
    \dotprod{\score[0][\x]}{\x}
        \le -\ctescorenorm{0} \normEc{\x}^2 + \ctescoreoffset{0} \eqsp.
    \]
    Fix $\bar{\lambda}$ such that $\lambda < \bar{\lambda} < \ctescorenorm{0}/2$.
    The previous inequality implies that there exists $R_{\bar{\lambda}}>0$ such that
    \[
    - \dotprod{\nabla U(\x)}{\x} \le -2\bar{\lambda} \normEc{\x}^2\eqsp,
    \qquad \text{ for }\x \notin \ball{0}{R_{\bar{\lambda}}},\eqsp.
    \]
    By Lemma~2.2 in \citet{bouchut2005uniqueness}, the above inequality implies that
    \[
        -\nabla^2 U(\x) \preceq -2\bar{\lambda}\eqsp \Id_{\xdim}\eqsp,
        \qquad \text{ for }\x \notin \ball{0}{R_{\bar{\lambda}}}\eqsp.
    \]
    Therefore,
    \begin{align*}
        &\int_{\rset^{\xdim}} \rme^{\lambda \normEc{\x}^2} \fwdmarg{0}[\x] \rmd \x
        \\
        &=
        \int_{\ball{0}{R_{\bar{\lambda}}}} \rme^{-\lambda \normEc{\x}^2} \fwdmarg{0}[\x] \rmd \x
        +
        \int_{\rset^{\xdim} \setminus \ball{0}{R_{\bar{\lambda}}}} \rme^{\lambda \normEc{\x}^2}
        \rme^{-U(\x)} \rmd \x 
        \\
        &\le 
        \int_{\ball{0}{R_{\bar{\lambda}}}} \rme^{-\lambda \normEc{\x}^2} \fwdmarg{0}[\x] \rmd \x
        + \int_{\rset^{\xdim} \setminus \ball{0}{R_{\bar{\lambda}}}} \rme^{\lambda \normEc{\x}^2}
            \rme^{-U(0) - \dotprod{\nabla U(0)}{\x} - \bar{\lambda} \normEc{\x}^2} \rmd \x < \infty
        \eqsp,
    \end{align*}
    which concludes the proof.
\end{proof}

\begin{lemma}
    \label[lemma]{lem:norm-score-bound}
    Suppose that~\cref{assump:p0:hess} holds. Then, there exists $\ctescorepoly \in \rsetpos$ such that 
    \begin{align}
        \normEc{\score[0][\x]} \leq \ctescorepoly (1 + \normEc{\x}^{p+1}) \eqsp,
    \end{align}
    where $\ctescorepoly = \max\left\{\sqrt{\xdim}\ctejscore{0}, \normEc{\score[0][0]}\right\}$.
\end{lemma}
\begin{proof}
    We have that
    \begin{align}
        \normEc{\score[0][\x] - \score[0][0]} \leq \int_{0}^{1}\normEc{\jscore[0][t\x] \x} \rmd t \leq \sqrt{\xdim} \ctejscore{0}\int_{0}^{1} (1 + t^{p}\normEc{\x}^{p+1}) \rmd t \eqsp.
    \end{align}
\end{proof}
\begin{corollary} \label[corollary]{cor:hess+score-bound_0}
    Suppose that~\cref{assump:p0:hess} holds. Then, there exists $\ctestable{0} \in \rsetpos$ such that
    \begin{align}
        \normFr{\jscore[0][\x] + \score[0][\x]^{\tensprod 2}} \leq \ctestable{0}\left(1 + \normEc{\x}^{2(p+1)}\right) \eqsp.
    \end{align}
\end{corollary}
\begin{proof}
    It is enough to note that
    \begin{align}
        \normFr{\jscore[0][\x] + \score[0][\x]^{\tensprod 2}} \leq \normFr{\jscore[0][\x]} + \normFr{\score[0][\x]^{\tensprod 2}}
    \end{align}
    and $\normFr{\score[0][\x]^{\tensprod 2}} = \normEc{\score[0][\x]}^2$.
\end{proof}

\subsection{Regarding the score approximation and the numerical scheme} \label{app:regarding_score}

We identify a minimal growth condition on the learned score approximation $\scorenet$ that guarantees the Euler--type scheme defined in \eqref{eq:time_changed_euler} admits finite moments at all discretization times (\cref{lem:scheme_finite_poly_moments}). This ensures all quantities appearing in the latter error decomposition are well-defined.

\begin{lemma}
\label[lemma]{lem:scheme_finite_poly_moments}
Let $(\Xbar_{t_k}^\theta)_{0\le k\le N}$ be defined by \eqref{eq:time_changed_euler} and recall that
$\Xbar_{t_0}^\theta\sim \refmeas$ is Gaussian. Assume that there exist constants $L_\theta\ge 0$ and $r\ge 1$ such that for all $k\in\{0,\dots,N-1\}$ and all $\x\in\rset^\xdim$,
\begin{align}
\label{eq:assump_poly_growth_minimal}
\normEc{\scorenet[T-t_k][\x][\theta]} \le L_\theta \left(1+\normEc{\x}^{r}\right) \eqsp.
\end{align}
Then, for every $q\ge 1$,
\begin{align*}
\sup_{0\le k\le N}\mathbb E \normEc{\Xbar_{t_k}^\theta}^{q}<\infty \eqsp.
\end{align*}
\end{lemma}

\begin{proof}
Let $(\xi_k)_{k\ge0}$ be an i.i.d.\ sequence of standard Gaussian random vectors in $\rset^\xdim$ and write 
\begin{align*}
\Xbar_{t_{k+1}}^{\theta}
=
\Xbar_{t_k}^{\theta}
+
\Delta_k b_{\theta} \left(t_k,\Xbar_{t_k}^{\theta}\right)+ \sqrt{2\Delta\tau_k} \xi_k \eqsp,
\end{align*}
with $b_\theta(t_k,\x)\eqdef \isvp \x + 2\,\scorenet[T-t_k][\x][\theta]$. By \eqref{eq:assump_poly_growth_minimal} there exists $C_{\theta} > 0$ such that
$\normEc{b_{\theta} \left(t_k,\x \right)}\le C_{\theta}(1+\normEc{\x}^r)$ for all $k \in\{0,\dots,N-1\} $ and $\x \in \rset^\xdim$. Fix $p\ge 1$, we get that there exists $C_{\theta}$, which may change from line to line, such that
\begin{align*}
\CPE{}{\Xbar_{t_k}^\theta}{\normEc{\Xbar_{t_{k+1}}^\theta}^p}
&\le
C_{\theta}\left(
\normEc{\Xbar_{t_k}^\theta}^p
+ (\Delta\tau_k)^p \normEc{b_\theta(t_k,\Xbar_{t_k}^\theta)}^p
+ (\Delta\tau_k)^{p/2} \mathbb E \left[ \normEc{\xi_k}^p \right]
\right) \\
&\le
C_{\theta} \left(1+ \normEc{\Xbar_{t_k}^\theta}^{p r}\right)\eqsp,
\end{align*}
Since $\Xbar_{t_0}^\theta$ is Gaussian, it has moments of all orders. The previous inequality shows that for an arbitrary $p$, if $\mathbb E \normEc{\Xbar_{t_k}^\theta}^{pr}<\infty$ then $\mathbb E \normEc{\Xbar_{t_{k+1}}^\theta}^{p}<\infty$. By induction over $k=0,\dots,N-1$, we obtain that $\mathbb{E} \normEc{\Xbar_{t_k}^\theta}^{q}<\infty$. Taking the supremum over the finite set $\{0,\dots,N\}$ yields the claim.
\end{proof}

\begin{remark}
If, for each $k$, the map $\x \mapsto \scorenet[T-t_k][\x][\theta]$ is Lipschitz with constant $L_{\theta}$ and $\sup_k \normEc{\scorenet[T-t_k][0][\theta]}\le M$, then
\begin{align*}
\normEc{\scorenet[T-t_k][x][\theta]}
\le \normEc{\scorenet[T-t_k][0][\theta]} + L_{\theta} \normEc{\x}
\le M + L_{\theta} \normEc{\x},
\end{align*}
so \eqref{eq:assump_poly_growth_minimal} holds with $r=1$.
\end{remark}

\section{Probability metrics and useful bounds} \label{app:proba_metrics}

We collect here the definitions of the probability distances used throughout the paper, together with a few inequalities relating them. Some results are standard and are included for the sake of clarity. The last subsection is dedicated to properties of $\rho_b$, which are used to establish the main contraction results.

\subsection{Definitions} \label{app:proba_def}

\paragraph{Weighted total variation distance.}

For any scale parameter $b > 0$ and $\lyapunov{2}[\x] \eqdef \normEc{\x}^2$, define the weighted supremum norm
\begin{align*}
    \|\varphi\|_b
    \coloneqq \sup_{\x \in \rset^\xdim} 
    \frac{|\varphi(\x)|}{1 + b  \lyapunov{2}[\x] } \eqsp.
\end{align*}
The corresponding dual distance on probability measures \citep{HairerMattingly2008} is given, for $\mu_1,\mu_2\in\mathcal P(\rset^\xdim)$, by
\begin{align} \label{eq:def_rho_b_ftest}
  \bmetric{\mu_1}{\mu_2}
    \coloneqq 
    \sup_{\|\varphi\|_b \le 1}
    \left| \int  \varphi \rmd (\mu_1 - \mu_2) \right| \eqsp.
\end{align}
This coincides with a weighted total variation distance
\begin{align} \label{eq:def_rho_b}
    \bmetric{\mu_1}{\mu_2}
    = 
    \int_{\rset^\xdim} (1 + b \lyapunov{2}[\x] ) |\mu_1 - \mu_2|(\rmd \x) \eqsp.
\end{align}

\paragraph{Total variation distance.} The total variation distance between $\mu_1$ and $\mu_2$ is defined by
\begin{align*}
\normTV{\mu_1-\mu_2}
\eqdef \frac{1}{2}
\sup_{\|f\|_\infty\le 1}
\left|\int f\,\rmd(\mu_1-\mu_2)\right| \eqsp.
\end{align*}
In particular, if $\mu_1$ and $\mu_2$ admit respectively densities $p_1$ and $p_2$ with respect to the Lebesgue measure, then 
\begin{align*}
\normTV{\mu_1-\mu_2}
\eqdef \frac12
\int_{\rset^\xdim} \left| p_1(\x) - p_2(\x) \right| \rmd \x, \quad \x \in \rset^d \eqsp.
\end{align*}

\paragraph{Kullback--Leibler divergence.}

The Kullback--Leibler (KL) divergence of $\mu_1$ with respect to $\mu_2$ is defined as
\begin{align*}
\kl{\mu_1}{\mu_2}
\eqdef
\int_{\rset^\xdim}
\log \left(\frac{\rmd\mu_1}{\rmd\mu_2} \right) \rmd \mu_1 \eqsp,
\end{align*}
if $\mu_1\ll\mu_2$ and $\kl{\mu_1}{\mu_2}=+\infty$ otherwise. 

\paragraph{Hellinger distance.} If $\mu_1$ and $\mu_2$ admit respectively densities $p_1$ and $p_2$ with respect to the Lebesgue measure, the Hellinger distance between $\mu_1$ and $\mu_2$ is defined as
\begin{align*}
\hell(\mu_1,\mu_2)
\eqdef
\left(
\frac12
\int_{\rset^\xdim}
\bigl(\sqrt{p_1(\x)}-\sqrt{p_2(\x)}\bigr)^2 \rmd \x
\right)^{1/2}, \quad \x \in \rset^d \eqsp.
\end{align*}

\subsection{Standard inequalities} \label{app:proba_standard_ineq}
The KL divergence and the total variation distance are related by Pinsker's inequality
\begin{align} \label{eq:pinsker}
\normTV{\mu_1-\mu_2}
\le
\sqrt{\tfrac12 \kl{\mu_1}{\mu_2}} \eqsp.
\end{align}
For a proof of such inequality see for example \citet{bakry2014analysis} or \citet[Chapter~2, Distances between probability measures]{Tsybakov2009}. 
The Hellinger distance is upper bounded
by the KL divergence 
\begin{align} \label{eq:hellinger_kl}
\hell(\mu_1,\mu_2)^2
\le
\tfrac12\,\kl{\mu_1}{\mu_2} \eqsp.
\end{align}
A proof of such result can be found in \citet[Chapter~2, Distances between probability measures]{Tsybakov2009}.

\subsection{Weighted total variation: properties and bounds} \label{app:proba_rhob}

We collect practical inequalities bounding above and below the weighted total variation distance defined in \eqref{eq:def_rho_b} in terms of other probability metrics. These bounds will be used in the subsequent development. 

\begin{lemma} \label[lemma]{lem:weighted_TV_bound_Hellinger}
Let $\mu_1,\mu_2 \in \mathcal{P}(\rset^{\xdim})$ be probability measures admitting Lebesgue densities $p_1,p_2$ respectively such that
\begin{align*}
\int_{\rset^\xdim} \lyapunov{2}^2 \rmd \mu_1<\infty
\quad\text{and}\quad
\int_{\rset^\xdim} \lyapunov{2}^2 \rmd \mu_2 <\infty \eqsp.
\end{align*}
Then, 
\begin{align*}
\bmetric{\mu_1}{\mu_2}
 \le 
\sqrt{2} \hellinger{\mu_1}{\mu_2}
\left(
\sqrt{\int_{\rset^\xdim} \left(1 + b \lyapunov{2}[\x] \right)^2 p_1(\x) \rmd \x}
+
\sqrt{\int_{\rset^\xdim} \left(1 + b \lyapunov{2}[\x] \right)^2 p_2(\x)\rmd \x}
\right) \eqsp.
\end{align*}
\end{lemma}

\begin{proof}
Note that
\begin{align*}
\left|p_1-p_2 \right|
=
\left|\sqrt{p_1}-\sqrt{p_2}\right| \left(\sqrt{p_1}+\sqrt{p_2}\right) \eqsp,
\end{align*}
and with $w (\x) \eqdef 1 + b \lyapunov{2}[\x]$, we have,
\begin{align*}
\bmetric{\mu_1}{\mu_2}
=
\int_{\rset^\xdim} w(\x) \left|\sqrt{p_1(\x)}-\sqrt{p_2(\x)}\right|
\left(\sqrt{p_1(\x)}+\sqrt{p_2(\x)}\right) \rmd \x.
\end{align*}
By the Cauchy--Schwarz inequlity,
\begin{align*}
    &\bmetric{\mu_1}{\mu_2}
    \\
    &\le
    \left(\int_{\rset^\xdim} \left(\sqrt{p_1(\x)}-\sqrt{p_2(\x)}\right)^2 \rmd \x \right)^{1/2}
    \left(\int_{\rset^\xdim} w(\x)^2\left(\sqrt{p_1(\x)}+\sqrt{p_2(\x)}\right)^2 \rmd \x\right)^{1/2} \\
    & = \sqrt2 \hellinger{\mu_1}{\mu_2}
    \left(\int_{\rset^\xdim} w(\x)^2\left(\sqrt{p_1(\x)}+\sqrt{p_2(\x)}\right)^2 \rmd \x\right)^{1/2} \eqsp.
\end{align*}
Then,
\begin{align*}
    \int_{\rset^\xdim} w^2(\sqrt{p_1}+\sqrt{p_2})^2\,\rmd x
    =
    \int w^2 p_1 \rmd \x
    +
    \int w^2 p_2 \rmd \x
    +
    2\int w^2\sqrt{p_1p_2} \rmd \x \eqsp,
\end{align*}
so that applying the Cauchy--Schwarz to the last term yields
\begin{align*}
    &\int_{\rset^\xdim} w^2(\sqrt{p_1}+\sqrt{p_2})^2 \rmd \x
    \\
    &\le
    \int_{\rset^\xdim} w^2 p_1 \rmd \x
    +
    \int_{\rset^\xdim} w^2 p_2 \rmd \x
    +
    2\Big(\int_{\rset^\xdim} w^2 p_1 \rmd x\Big)^{1/2}
    \Big(\int_{\rset^\xdim} w^2 p_2 \rmd \x\Big)^{1/2}\\
    &=
    \Bigg(
    \Big(\int_{\rset^\xdim} w^2 p_1\rmd \x\Big)^{1/2}
    +
    \Big(\int_{\rset^\xdim} w^2 p_2 \rmd \x\Big)^{1/2}
    \Bigg)^2 \eqsp.
\end{align*}
Therefore,
\begin{align*}
\bmetric{\mu_1}{\mu_2}
\le
\sqrt{2} \hellinger{\mu_1}{\mu_2} 
\left(
\sqrt{\int_{\rset^\xdim} w(\x)^2 p_1(\x) \rmd \x}
+
\sqrt{\int_{\rset^\xdim} w(\x)^2 p_2(x) \rmd \x}
\right) \eqsp,
\end{align*}
which concludes the proof.
\end{proof}

\begin{lemma} \label[lemma]{lem:rho_b_convexity}
Let $\mu\in\mathcal P(\rset^\xdim)$ and let
$K,\tilde K$ be two Markov kernels on $\rset^\xdim$ such that
\begin{align*}
\int_{\rset^\xdim} \normEc{\x}^2 (\mu K)(\rmd \x) < \infty
\quad\text{and}\quad
\int_{\rset^\xdim} \normEc{\x}^2 (\mu \tilde K)(\rmd \x) < \infty .
\end{align*}
Then,
\begin{align*}
\bmetric{\mu K}{\mu \tilde K}
\le
\int_{\rset^\xdim} \bmetric{\delta_\x K}{\delta_\x \tilde K} \mu(\rmd \x) \eqsp.
\end{align*}
\end{lemma}

\begin{proof}
Fix $\varphi$ such that $\|\varphi\|_b\le 1$,
\begin{align*}
\left|(\mu K)(\varphi)-(\mu \tilde K)(\varphi)\right| & = \left| \int_{\rset^\xdim} \Big( (\delta_\x K)(\varphi)-(\delta_\x \tilde K)(\varphi)\Big) \mu(\rmd \x) \right| \\
& \le \int_{\rset^\xdim} \big|(\delta_\x K)(\varphi)-(\delta_\x \tilde K)(\varphi)\big| \mu(\rmd \x) \eqsp.
\end{align*}
For each fixed $\x$, since $\|\varphi\|_b\le 1$,
\begin{align*}
\left|(\delta_\x K)(\varphi)-(\delta_x \tilde K)(\varphi)\right|
\le
\sup_{\|\psi\|_b\le 1}\left|(\delta_\x K)(\psi)-(\delta_\x \tilde K)(\psi)\right|
=
\bmetric{\delta_\x K}{\delta_\x \tilde K} \eqsp.
\end{align*}
Therefore,
\begin{align*}
\big|(\mu K)(\varphi)-(\mu\widetilde K)(\varphi)\big|
\le
\int_{\rset^\xdim} \bmetric{\delta_x K}{\delta_x \widetilde K}\,\mu(\rmd x) \eqsp.
\end{align*}
Taking the supremum over all $\varphi$ with $\|\varphi\|_b\le 1$ as in \eqref{eq:def_rho_b_ftest} gives the result.
\end{proof}

\begin{lemma}
\label[lemma]{lem:ineq:metrics}
Let $\mu_1,\mu_2\in \pset{\rset^\xdim}[2]$. Then,
\begin{align*}
   \normTV{\mu_1 -\mu_2}  \leq \frac{1}{2} \bmetric{\mu_1}{\mu_2} \eqsp,
\end{align*}
and
\begin{align*}
\wasserstein[2][2]{\mu_1}{\mu_2}   \leq \frac{2}{b} \bmetric{\mu_1}{\mu_2} \eqsp.
\end{align*}

\end{lemma}

\begin{proof}
The first inequality follows directly from $1 \leq 1 + b \lyapunov{2}$. For the second, \citet[Theorem~6.15]{villani_optimal_2009} yields, 
\begin{align*}
    \wasserstein[2][2]{\mu_1}{\mu_2} \leq 2 \int_{\rset^d} \normEc{\x}^2 | \mu_1 - \mu_2 |  (\rmd \x) \eqsp,
\end{align*}
and
\begin{align*}
    \bmetric{\mu_1}{\mu_2}   &= \int_{\rset^d} (1+ b \left\| \x \right\|^2) | \mu_1 - \mu_2 |\rmd \x 
    \geq \int_{\rset^d} b \left\| \x \right\|^2 | \mu_1 - \mu_2 |\rmd \x \eqsp,
\end{align*}
which concludes the proof.
\end{proof}

\section{Stability properties under Gaussian perturbations}
\label{app:stability-results}

This section collects stability and regularity facts for the time-marginals of the forward diffusion, viewed as Gaussian perturbations of the data distribution.
The key structural point is that $\mathcal{L}(\Xora_t)$ is obtained from $\pidata$ by Gaussian smoothing (in VE) or by a Gaussian smoothing followed by a deterministic down-scaling (in VP). This representation yields explicit identities expressing the score $\score[t]$ and its Jacobian $\jscore[t]$ in terms of conditional expectations of $\score[0]$ and $\jscore[0]$ under the posterior $\mathcal{L}(\Xora_0\mid \Xora_t=\x)$.
These formulas are the main tool to transfer growth, integrability, and dissipativity properties from $\pidata$ to $\mathcal{L}(\Xora_t)$, which will be used later to establish contraction estimates in the metric $\rho_b$.

\subsection{Gaussian representation} \label{app:gaussian_rpz}

Following the forward SDE \eqref{eq:forward_SDE}, the marginal law of the process at any time $t$ admits an explicit Gaussian representation as a convolution of the initial distribution. This Gaussian convolution structure plays a central role in the analysis, as it allows us to transfer properties of the initial distribution to the time-marginal laws of the diffusion. We make this representation precise and derive the corresponding expressions in the following result.


\begin{lemma} \label[lemma]{lem:forward_process_law}
The solution of the forward process \eqref{eq:forward_SDE} writes as
\begin{align} \label{forward_marginal}
    \Xora_t \eqlaw \fwdmean{0}{t} \Xora_0 + \fwdstd{0}{t} G \eqsp,
\end{align}
with $\Xora_0 \sim \pidata$, $G\sim \gaussiand{0}{\Id_\xdim}$ independent of $\Xora_0$, and
\begin{align*}
    \fwdmean{0}{t}
    =
    \exp\left(- \isvp \int_0^t \noisesch{s} \rmd s\right)
    \quad \mathrm{and} \quad 
    \fwdvar{0}{t}
    = 
   \rme^{-2 \isvp \int_0^t \noisesch{s} \rmd s} \int_0^t \frac{2 \noisesch{s}}{\rme^{- 2 \isvp \int_0^s \noisesch{u} \rmd u}} \rmd s 
    \eqsp.
\end{align*}
In particular,
$\Xora_t \stackrel{\mathcal L}{=} \Xora_0 + \fwdstd{0}{t} G$ in the VE case ($\isvp=0$) and
$\Xora_t \stackrel{\mathcal L}{=} \fwdmean{0}{t} \Xora_0 + (\isvp^{-1}(1 - \fwdmean{0}{t}^2))^{1/2} \eqsp G$ in the VP case ($\isvp>0$).
Furthermore, the probability density of the distribution of $\X_s$ given $(\X_t, \X_0)$ writes
\begin{align}
    \bridge{t, 0}{s}[\x_t, \x_0][\x_s] =
    \gaussiand{
        \bridgemean{t}{s}{0}\x_0 + \bridgemean{t}{s}{t}\x_t
    }{
        \bridgevar{t}{s}\Id_d
    }[\x_s] \eqsp,
\end{align}
with $\bridgevar{t, 0}{s} \eqdef \fwdvar{s}{t} \fwdvar{0}{s} / (\fwdvar{s}{t} \fwdmean{s}{t}^2 + \fwdvar{0}{s})$ and   $\bridgemean{t, 0}{s}{t} \eqdef \sqrt{(\bridgevar{t, 0}{s}-\fwdvar{0}{s})/\fwdvar{0}{t}}$ and $\bridgemean{t, 0}{s}{0} \eqdef \fwdmean{0}{s} - \bridgemean{t, 0}{s}{t} \fwdmean{0}{t}$.
\end{lemma}

\begin{proof}
    Define $Y_t = \rme^{\isvp \int_0^t \noisesch{s} \rmd s}  \Xora_t$ such that, by Itô's lemma, 
    \begin{align*}
 \rmd Y_t = \sqrt{2 \noisesch{t}} \rme^{\isvp \int_0^t \noisesch{s} \rmd s} \dbrown{t} \eqsp.
    \end{align*}
Therefore, 
\begin{align*}
    Y_t = Y_0 + \int_0^t \sqrt{2 \noisesch{s}} \rme^{\isvp \int_0^s \noisesch{u} \rmd u} \dbrown{s} \eqsp,
\end{align*}
and
\begin{align} \label{eq:forward_deco}
    \Xora_t = \rme^{- \isvp \int_0^t \noisesch{s} \rmd s} \Xora_0 + \rme^{- \isvp \int_0^t \noisesch{s} \rmd s} \int_0^t \frac{\sqrt{2 \noisesch{s}} }{\rme^{-\isvp \int_0^s \noisesch{u} \rmd u}} \dbrown{s} \eqsp.
\end{align}
Note that when $\isvp = 0$ and $\noisesch{s} = \fwdstd{0}{s} \frac{\rmd \fwdstd{0}{s}}{\rmd s}$, $\int_0^t \sqrt{2 \noisesch{s}}  \dbrown{s}$ is a centered Gaussian vector with covariance matrix
\begin{align*}
    \Sigma_t = \int_0^t 2 \frac{\rmd \fwdstd{0}{s}}{\rmd s}\rmd s \Id_{\xdim} = \int_0^t \frac{\rmd \fwdvar{0}{s}}{\rmd s} \rmd s  \Id_{\xdim}= \fwdvar{0}{t}\Id_{\xdim} \eqsp.
\end{align*}
Hence, for $Z\sim\mathcal{N}(0,\Id_{\xdim})$  independent of $\Xora_0$, $\Xora_t \stackrel{\mathcal L}{=} \Xora_0 +  \fwdstd{0}{t} Z$. When $\isvp >0$, the stochastic part of \eqref{eq:forward_deco} is a centered Gaussian vector with covariance matrix,
\begin{align*}
    \Sigma_t
    = 
    \left(
        \rme^{- 2\isvp\int_0^t \noisesch{s} \rmd s} \int_0^t 2 \noisesch{s} \rme^{ 2 \isvp \int_0^s \noisesch{u} \rmd u} \rmd s
    \right) \eqsp \Id_{\xdim}
    & =
    \frac{1}{\isvp} \left(
        1 -  \rme^{- 2\int_0^t \isvp \noisesch{s} \rmd s}
    \right) \Id_{\xdim}
    =
    \frac{1}{\isvp} \left(
        1 - \fwdmean{0}{t}^2
    \right) \Id_{\xdim}
    \eqsp,
\end{align*}
which concludes the proof.

For the conditional distribution of $\X_s$ given $(\X_t, \X_0)$, note that the conditional distributions of $\X_s$ given $\X_0$, $\X_t$ given $\X_0$ and $\X_t$ given  $\X_s$ admit a Gaussian density with respect to the Lebesgue measure. Thus, by Bayes theorem,
\begin{align}
    \bridge{t, 0}{s}[\x_t, \x_0][\x_s] = \frac{\fwdtrans{s}{t}[\x_s][\x_t]\fwdtrans{0}{s}[\x_0][\x_s]}{\fwdtrans{0}{t}[\x_0][\x_t]} \eqsp.
\end{align}
Therefore, we obtain that
\begin{align}
    \bridge{t, 0}{s}[\x_t, \x_0][\x_s] = \gaussiand{\bridgemean{t, 0}{s}{t}\x_t + \bridgemean{t, 0}{s}{0}\x_0}{\bridgevar{t, 0}{s}\Id_{\xdim}}[\x_s]\eqsp,
\end{align}
with $\bridgevar{t, 0}{s} = \fwdvar{s}{t} \fwdvar{0}{s} / (\fwdvar{s}{t} \fwdmean{s}{t}^2 + \fwdvar{0}{s})$,  $\bridgemean{t, 0}{s}{t} = \sqrt{(\bridgevar{t, 0}{s}-\fwdvar{0}{s})/\fwdvar{0}{t}}$ and $\bridgemean{t, 0}{s}{0} = \fwdmean{0}{s} - \bridgemean{t, 0}{s}{t} \fwdmean{0}{t}$.
\end{proof}

From the Gaussian representation \eqref{forward_marginal} in the case of $\isvp>0$, we may rewrite
\[
\Xora_t \;\stackrel{\mathcal L}{=}\; \fwdmean{0}{t}\Big(\Xora_0 + \varepsilon_t\,G\Big),
\qquad 
\varepsilon_t^2 \coloneqq \frac{\fwdvar{0}{t}}{\fwdmean{0}{t}^2},
\]
where $G\sim\gaussiand{0}{\Id_\xdim}$ is independent of $\Xora_0$. In other words, up to the scaling factor $\fwdmean{0}{t}$, the marginal at time $t$ is a Gaussian perturbation of $\Xora_0$ with mean $0$ and covariance matrix $\varepsilon_t^2 \Id_\xdim = (\fwdvar{0}{t}/\fwdmean{0}{t}^2)\Id_\xdim$. Moreover, there exists a simple relation satisfied by $\Xora_0+\varepsilon G$ and $m\big(\Xora_0+\varepsilon G\big)$ for $m\in(0,1)$, as stated and proved in the following lemma. For this reason, in the following, we will focus on the VE case, and derive bounds for $\isvp>0$ afterwards.

\begin{lemma} \label[lemma]{lem:vetovp}
    Let $m \in (0,1)$ and $q_\epsilon$ (resp. $\bar{q}_\epsilon$) be the density function of the random variable $\X_0 + \epsilon G$ (resp. $m(\X_0 + \epsilon G)$), where $\X_0 \sim \pdata$ and $G \sim \gaussiand{0}{\Id_{\xdim}}$ is independent of $\X_0$. Then, for all $\x \in \rset^\xdim$, the associated score function is given by 
    \begin{align}
    \label{eq:vetovp}
        \nabla\log q_\epsilon(\x) = \frac{1}{m} \nabla\log \bar{q}_\epsilon\left(\frac{\x}{m}\right)\eqsp.
    \end{align}
\end{lemma}
\begin{proof}
By construction, the two random variables are one the scaled version of the other.
Hence, by a linear change of variables, the probability density function of the above writes as $m^{-d}\fwdmarg{t}[\x/m]$, therefore the desired identity is simply obtained applying the logarithmic function and the derivative.
\end{proof}

\begin{remark}
    \label{rmk:sigma-free-param}
    Note that if one uses the same $\noisesch{t}$ for both VE framework ($\isvp=0$) and the general case $\isvp>0$, the relation \eqref{eq:vetovp} does not hold for \emph{the same t}. Indeed, for a given $t$ in the VP framework, one must choose $\fwdmean{0}{t}$ such that $\fwdmean{0}{t} = \exp\left(-\int_0^{t}\isvp\noisesch{s} \rmd s\right)$ and at the same find a time $\timechange[t]$ such that
    \begin{align*}
        \fwdmean{0}{t}^2 \fwdvar{0}{\timechange[t]} 
        =
        \fwdstd{0}{t}^2
        =
        \frac{1}{\isvp}\left(1-\fwdmean{0}{t}^2\right)
        \eqsp,
    \end{align*}
    which implies $\int_0^{\timechange[t]} \noisesch{s} \rmd s =  \int_{0}^{t}\noisesch{s}\exp(2\isvp\int_0^s \noisesch{u} \rmd u) \rmd s$.
    We shall not refrain from using~\cref{lem:vetovp} to establish the link between VE and $\isvp>0$ frameworks, considering the parameter $\fwdvar{0}{t}$ in the VE case as a free parameter.
\end{remark}

\subsection{Properties of the score under Gaussian perturbation}
\label{subsec:gaussian-perturbation}
We study how various properties of the score function are preserved under Gaussian perturbations. These results are crucial for analyzing the forward diffusion process and its associated score functions in both the VE and $\isvp>0$ frameworks.
Due to the previous connection between simply adding Gaussian noise and the VE framework ($\isvp = 0$), we will consider the notations defined for the VE framework for presenting the results. Introduce also the denoiser application $\denoiser[t][\x] = \CPE{}{\Xora_t = \x}{\Xora_0}$.


In this section we present only relations between $\fwdmarg{t}$ and $\fwdmarg{0}$ thus one must interpret $\fwdstd{t}{0}$ as a free parameter.
We will then rely on the following link between both frameworks to translate results in VE to results for VP.
\begin{lemma} \label[lemma]{lem:score-representations}
Assume that
\cref{assump:p0:score} holds.
Then, for all $\x \in \rset^{\xdim}$,
\begin{align}
  \score[t][\x]
  &= - \frac{1}{\fwdvar{0}{t}}
     \left( \x - \denoiser[t][\x] \right)
  \label{eq:score2}
  \\
  &= \CPE{}{\Xora_t = \x}{\score[0][\Xora_0]} 
  \label{eq:score1} \eqsp.
\end{align}
\end{lemma}

\begin{proof}
The density of $\X_t$ is the convolution
\begin{align*}
    \fwdmarg{t}[\x] = \int \fwdmarg{0}[y]\,
    \gaussMarg{t}[\x-y]\,\rmd y,
\end{align*}
where $\gaussMarg{t}$ is the density of $\gaussiand{0}{\fwdvar{0}{t}\Id_{\xdim}}$.
Differentiating under the integral,
\begin{align*}
    \nabla \fwdmarg{t}[\x]
    = \int \fwdmarg{0}[y]\,\nabla_\x \gaussMarg{t}[\x-y]\,\rmd y
    = -\frac{1}{\fwdvar{0}{t}}\int (\x-y)\,\fwdmarg{0}[y]\,\gaussMarg{t}[\x-y]\,\rmd y.
\end{align*}
Hence
\begin{align*}
    \score[t][\x]
    = \frac{\nabla \fwdmarg{t}[\x]}{\fwdmarg{t}[\x]}
    = -\frac{1}{\fwdvar{0}{t}}
      \left(\x - \CPE{}{\X_t=\x}{\X_0}\right),
\end{align*}
which proves \eqref{eq:score2}. For \eqref{eq:score1}, note that, from Bayes' formula for the posterior distribution of $\X_0$ given $\X_t=\x$,
\begin{align*}
\CPE{}{\Xora_t=\x}{\score[0][\Xora_0]}
& = \frac{1}{\fwdmarg{t}[\x]}
  \int \score[0][y]\, \fwdmarg{0}[y]\, \gaussMarg{t}[\x-y]\, \rmd y \\
& = \frac{1}{\fwdmarg{t}[\x]}
  \int \nabla_y \fwdmarg{0}[y]\, \gaussMarg{t}[\x-y]\, \rmd y \eqsp.
\end{align*}
Under~\cref{assump:p0:score}, the sub-Gaussian tail behavior established in~\cref{lem:pdata-sub-gaussian} guarantees the validity of integration by parts with respect to $y$, with vanishing boundary contributions. This gives
\begin{align*}
\int \nabla \fwdmarg{0}[y]\, \gaussMarg{t}[\x-y]\, \rmd y
= -\int \fwdmarg{0}[y]\, \nabla_y \gaussMarg{t}[\x-y]\, \rmd y.
\end{align*}
Noting that $\nabla_y \gaussMarg{t}[\x-y] = -\nabla_\x \gaussMarg{t}[\x-y]$,
we obtain
\begin{align*}
\int \nabla \fwdmarg{0}[y]\, \gaussMarg{t}[\x-y]\, \rmd y
= \int \fwdmarg{0}[y]\, \nabla_\x \gaussMarg{t}[\x-y]\, \rmd y
= \nabla_\x \fwdmarg{t}[x].
\end{align*}
Thus,
\begin{align*}
\CPE{}{\X_t=\x}{\score[0][\X_0]}= \frac{\nabla_\x \fwdmarg{t}[\x]}{\fwdmarg{t}[\x]}
= \score[t][\x],
\end{align*}
which proves \eqref{eq:score1}.
\end{proof}

\begin{lemma} \label[lemma]{lem:rewrite_Hessian}
Assume that
\cref{assump:p0:score} holds.
Then, for all $\x \in \rset^\xdim$,
\begin{align}
    \label{eq:relation-hessian-score-in_time}
    \jscore[t][\x]
    +
    \score[t][\x]\mt{\score[t][\x]}
    = \CPE{}{\Xora_t=\x}{\jscore[0][\Xora_0] + \score[0][\Xora_0]\mt{\score[0][\Xora_0]}}
    \eqsp,
\end{align}
and
\begin{align}
    \jscore[t][\x]
    & = \mathrm{Var}\big(\score[0][\Xora_0] \mid \Xora_t = \x \big)
    + \CPE{}{\Xora_t = \x}{\jscore[0][\Xora_0]} 
    \label{eq:grad-Seps-variance-form}
    \\
    & 
    = \frac{1}{\fwdvar{0}{t}} \left(\CPE{}{\Xora_t=\x}{\score[0][\Xora_0]\mt{\Xora_0}} - \score[t][\x]\mt{\denoiser[t][\x]} \right)
    \label{eq:def_grad_s1} 
    \\
    & = \frac{1}{\fwdstd{0}{t}^4} \mathrm{Var} \left( \Xora_0 \mid \Xora_t = \x \right) - \frac{1}{\fwdvar{0}{t}} \Id_\xdim \eqsp.
    \label{eq:def_grad_s3}
\end{align}
\end{lemma}
\begin{proof}

Under~\cref{assump:p0:score}, from~\cref{lem:pdata-sub-gaussian}, the Gaussian convolutions are smooth and all integrations by parts below are justified (the boundary terms vanish). First,
\begin{align}
\label{eq:log-hessian-identity}
\jscore[t][\x]
= \nabla^2 \log \fwdmarg{t}[\x]
= \frac{\nabla^2 \fwdmarg{t}[\x]}{\fwdmarg{t}[\x]}
  - \score[t][\x]\mt{\score[t][\x]}.
\end{align}
Differentiating twice under the integral and integrating by parts in $y$, we obtain
\begin{align*}
    \nabla^2 \fwdmarg{t}[\x]
    = \int \nabla^2 \fwdmarg{0}[y]\,\gaussMarg{t}[\x-y]\,\rmd y.
\end{align*}
Moreover,
\begin{align*}
    \nabla^2 \fwdmarg{0}[y]
    = \big( \jscore[0][y] + \score[0][y] \mt{\score[0][y]} \big)\,\fwdmarg{0}[y],
\end{align*}
so that we get
\begin{align*}
\frac{\nabla^2 \fwdmarg{t}[\x]}{\fwdmarg{t}[\x]}
= \CPE{}{\Xora_t=\x}{\jscore[0][\Xora_0] + \score[0][\Xora_0]\mt{\score[0][\Xora_0]}}\eqsp.
\end{align*}
Plugging this into \eqref{eq:log-hessian-identity} yields \eqref{eq:relation-hessian-score-in_time}.
Moreover, using \eqref{eq:score1} from~\cref{lem:score-representations}, we get \eqref{eq:grad-Seps-variance-form}. 

To prove \eqref{eq:def_grad_s1}, recall from \eqref{eq:score1} that,
\begin{align*}
\score[t][\x] = \CPE{}{\Xora_t=\x}{\score[0][\Xora_0]} \eqsp,
\end{align*}
and write,
\begin{align*}
 \score[t][\x] =\frac{N(\x)}{\fwdmarg{t}[\x]},
\qquad
N(\x)\eqdef \int \score[0][y] \fwdmarg{0}[y]\,\gaussMarg{t}[\x-y] \rmd y \eqsp.
\end{align*}
It follows that
\begin{align*}
\jscore[t][\x]
& =\frac{\nabla N(\x)}{\fwdmarg{t}[\x]}
-\score[t][\x]\frac{\nabla \mt{\fwdmarg{t}[\x]}}{\fwdmarg{t}[\x]} \\
& = \frac{\nabla N(\x)}{\fwdmarg{t}[\x]}
-\score[t][\x] \mt{\score[t][\x]} \eqsp.
\end{align*}
Since $\nabla_\x\gaussMarg{t}[\x-y]=-(\x-y)\fwdstd{0}{t}^{-2}\gaussMarg{t}[\x-y]$, we have
\begin{align*}
\frac{\nabla N(\x)}{\fwdmarg{t}[\x]}
& =
-\frac{1}{\fwdvar{0}{t}}\CPE{}{\Xora_t=\x}{\score[0][\Xora_0]\mt{(\x - \X_0)}}\\
& =
\frac{1}{\fwdvar{0}{t}} \left( \CPE{}{\Xora_t=\x}{\score[0][\Xora_0] \mt{\Xora_0}} - \CPE{}{\Xora_t=\x}{\score[0][\Xora_0]} \mt{\x} \right) \\
& =
\frac{1}{\fwdvar{0}{t}} \left( \CPE{}{\Xora_t=\x}{\score[0][\Xora_0] \mt{\Xora_0}} - \score[t][\x] \mt{\x} \right) \eqsp.
\end{align*}
Then, using \eqref{eq:score2}, 
\begin{align*}
   \score[t][\x] \mt{\x} & = - \fwdvar{0}{t} \score[t][\x] \mt{\score[t][\x]} + \score[t][\x]\mt{\denoiser[t][\x]} \eqsp.
\end{align*}
Therefore, 
\begin{align*}
\jscore[t][\x]
& = \frac{1}{\fwdvar{0}{t}} \left(\CPE{}{\Xora_t=\x}{\score[0][\Xora_0] \mt{\X_0}} - \score[t][\x] \mt{\denoiser[t][\x]} \right) \eqsp,
\end{align*}
which proves \eqref{eq:def_grad_s1}. Finally, similar computations as for $\jscore[t][\x]$ yields,
\begin{align*}
    \nabla \denoiser[t][\x] & = \frac{1}{\fwdvar{0}{t}} \left( \CPE{}{\Xora_t=\x}{\Xora_0\mt{\Xora_0}} -\denoiser[t][\x]\mt{\denoiser[t][\x]} \right) \\
    & = \frac{1}{\fwdvar{0}{t}} \mathrm{Var} \big(\Xora_0 \mid \Xora_t=\x\big) \eqsp,
\end{align*}
where the last equality follows from the definition of $\denoiser[t]$. Rearranging \eqref{eq:score2} we get, \begin{align*} \denoiser[t][\x] = \fwdvar{0}{t} \score[t][\x] + \x \eqsp,
\end{align*}
which proves \eqref{eq:def_grad_s3} and finishes the proof. 
\end{proof}

\section{Stability of the data assumptions along the diffusion flow and Lyapunov contraction} \label{app:stab+lyapounov}

This section shows that the structural conditions imposed at time $0$ remain stable along the diffusion. First, we prove that the dissipativity inequality in \cref{assump:p0:score} propagates to the forward-time score $\score[t]$ with explicit time-dependent constants. Second, we leverage this propagated dissipativity to establish a Lyapunov drift condition for the backward Markov semigroup $\bwdker{s}[t]$, which is the key ingredient needed to apply Harris-type theorems \citep{HairerMattingly2008}. Finally, we show that polynomial growth controls on the score and its Jacobian (as in~\cref{assump:p0:hess}) are preserved under the forward flow yielding uniform-in-time bounds used throughout the paper.

\subsection{Propagation of the dissipativity condition} \label{subapp:dissipativity}

The properties of the score function established in the previous section allow us to study the stability of the Lyapunov condition stated in~\cref{assump:p0:score} through the diffusion dynamics, \ie, from $\score[0]$ to $\score[t]$. The next proposition makes this transfer explicit by providing time-dependent dissipativity parameters $\ctescorenorm{t}$ and $\ctescoreoffset{t}$.

\begin{proposition}
    \label[proposition]{prop:lyapunov_stability}
    Suppose that~\cref{assump:p0:score} holds. Then, for all $t > 0$ and $\x \in \rset^\xdim$, there exist continuous functions $\ctescorenorm{t}, \ctescoreoffset{t}: [0,T] \to \rset_+$ such that,
    \begin{align}
        \dotprod{\score[t][\x]}{\x} 
        \le - \ctescorenorm{t} \normEc{\x}^2 + \ctescoreoffset{t},
    \end{align}
    where $\ctescorenorm{t}$ and $\ctescoreoffset{t}$ are defined as follows for $t \in(0,T]$
    \begin{itemize}
        \item
        If $\isvp = 0$ (Variance-Exploding case),
        \begin{align*}
            \ctescorenorm{t}
            = \frac{\ctescorenorm{0}}{1 + 2\ctescorenorm{0} \fwdvar{0}{t}},
            \qquad
            \ctescoreoffset{t}
            = \frac{\ctescoreoffset{0} + \xdim}{1 + 2\ctescorenorm{0} \fwdvar{0}{t}}
            \eqsp.
        \end{align*}
        \item
        If $\isvp >0$,
        \begin{align*}
            \ctescorenorm{t} 
            = \frac{
                \ctescorenorm{0}\isvp
            }{
                \isvp \fwdmean{0}{t}^2 + 2\ctescorenorm{0}\left(1 - \fwdmean{0}{t}^2\right)},
            \qquad
            \ctescoreoffset{t} 
            = (\ctescoreoffset{0} + d)\,
            \frac{
                \fwdmean{0}{t}^2 \isvp
            }{
                \isvp\fwdmean{0}{t}^2 + 2\ctescorenorm{0}\left(1 - \fwdmean{0}{t}^2\right)}
            \eqsp,
        \end{align*}
    \end{itemize}
\end{proposition}
\begin{proof}
    \emph{Case $\isvp >0$. }
    Computing the trace of \eqref{eq:def_grad_s1} yields, for all $\x \in \rset^{\xdim}$, $t>0$,
    \begin{align*}
        \divergence{\score[t][\x]} = \frac{1}{\fwdvar{0}{t}} \left( \CPE{}{\Xora_t=\x}{\dotprod{\score[0][\Xora_0]}{\Xora_0}} - \dotprod{\score[t][\x]}{\denoiser[t][\x]} \right)
        \eqsp. 
    \end{align*}
    Moreover, using \eqref{eq:score2},
    \begin{align*}
        \dotprod{\score[t][\x]}{\denoiser[t][\x]} = \dotprod{\score[t][\x]}{\x} + \fwdvar{0}{t} \normEc{\score[t][\x]}^2 \eqsp,
    \end{align*}
    together with the fact that
    \begin{align}
        \label{eq:div_identity}
    \divergence{\score[t][\x]} + \normEc{\score[t][\x]}^2
    = \frac{\Delta \fwdmarg{t}[\x]}{\fwdmarg{t}[\x]} \eqsp, 
    \end{align}
    we get
    \begin{align*}
        \dotprod{\score[t][\x]}{\x} = \CPE{}{\Xora_t=\x}{\dotprod{\score[t][\Xora_0]}{\Xora_0}} - \fwdvar{0}{t} \frac{\Delta \fwdmarg{t}[\x]}{\fwdmarg{t}[\x]}
        \eqsp.
    \end{align*}
    It follows from~\cref{assump:p0:score} that
    \begin{align*}
        \dotprod{\score[t][\x]}{\x} \leq - \ctescorenorm{0} \CPE{}{\Xora_t=\x}{\normEc{\Xora_0}^2} + \ctescoreoffset{0} - \fwdvar{0}{t} \frac{\Delta \fwdmarg{t}[\x]}{\fwdmarg{t}[\x]}
        \eqsp.
    \end{align*}
    Combining $\|a\|^2 = \|b\|^2 + \|a-b\|^2 + 2 \dotprod{a-b}{b}$ and \eqref{eq:score2}, we have
    \begin{align*}
        \CPE{}{\Xora_t=\x}{\normEc{\Xora_0}^2}
        &= \normEc{\x}^2 + \CPE{}{\Xora_t=\x}{\normEc{\Xora_0-\x}^2} 
            + 2 \fwdvar{0}{t} \dotprod{\score[t][\x]}{\x}, \\
        & \ge \normEc{\x}^2 + 2 \fwdvar{0}{t} \dotprod{\score[t][\x]}{\x}
        \eqsp.
    \end{align*}
    It follows
    \begin{align*}
        \dotprod{\score[t][\x]}{\x} 
        \le - \ctescorenorm{0} \normEc{\x}^2  - 2 \ctescorenorm{0} \fwdvar{0}{t} \dotprod{\score[t][\x]}{\x} + \ctescoreoffset{0} - \fwdvar{0}{t} \frac{\Delta \fwdmarg{t}[\x]}{\fwdmarg{t}[\x]} \eqsp,
    \end{align*}
    which yields 
    \begin{align*}
        \dotprod{\score[t][\x]}{\x} 
        \le - \frac{\ctescorenorm{0}}{1 + 2\ctescorenorm{0} \fwdvar{0}{t}}  \normEc{\x}^2 + \frac{1}{1 + 2\ctescorenorm{0} \fwdvar{0}{t}} \left( \ctescoreoffset{0} - \fwdvar{0}{t} \frac{\Delta \fwdmarg{t}[\x]}{\fwdmarg{t}[\x]} \right) \eqsp.
    \end{align*}
    Using identity \eqref{eq:div_identity},
    \begin{align*}
        - \fwdvar{0}{t} \frac{\Delta \fwdmarg{t}[\x]}{\fwdmarg{t}[\x]} =  - \fwdvar{0}{t} \left( \trace{\jscore[t]}(\x) +\normEc{\score[t][\x]}^2\right)
        \eqsp.
    \end{align*}
    This with \eqref{eq:def_grad_s3} yields
    \begin{align*}
        - \fwdvar{0}{t} \frac{\Delta \fwdmarg{t}[\x]}{\fwdmarg{t}[\x]} & = \xdim - \frac{1}{\fwdvar{0}{t}}\trace{ \CPV{}{\Xora_t=\x}{\Xora_0}} - \fwdvar{0}{t} \normEc{\score[t][\x]}^2  \leq \xdim \eqsp,
    \end{align*}
    which concludes the proof.

    \emph{Case $\isvp>0$. } As noted in~\cref{rmk:sigma-free-param}, the general case $\isvp>0$ can be deduced from the VE case by a suitable time change. With abuse of notation, write $\ve{\fwdstd{0}{t}} = \fwdstd{0}{t} / \fwdmean{0}{t}$. By~\cref{lem:vetovp}, we get that
    \begin{align*}
        &\dotprod{\score[t][\x]}{\x}
        \\
        &=
        \dotprod{\score[t][\frac{\x}{{\fwdmean{0}{t}}}]}{\frac{\x}{\fwdmean{0}{t}}}
        \\
        &\leq
        \frac{\ctescorenorm{0}}{
            1 + 2\ctescorenorm{0} (\fwdstd{0}{t} / \fwdmean{0}{t})^2
        }  \normEc{\frac{\x}{\fwdmean{0}{t}}}^2
        + \frac{1}{1 + 2\ctescorenorm{0} (\fwdstd{0}{t} / \fwdmean{0}{t})^2} \left( \ctescoreoffset{0} - (\fwdstd{0}{t} / \fwdmean{0}{t})^2 \frac{\Delta \fwdmarg{t}[\x/{\fwdmean{0}{t}}]}{\fwdmarg{t}[\x/{\fwdmean{0}{t}}]} \right)
        \eqsp,
    \end{align*}
    which concludes the proof.
\end{proof}

\begin{remark}
    \label{rmk:uniform-dissipativity}
    From \cref{lem:forward_process_law} and the explicit expressions of $\fwdmean{0}{t}$ and $\fwdstd{0}{t}$, it is straightforward to check that the condition $\ctescorenorm{0}>\isvp/2$ propagates in time. Indeed, for all $t\ge 0$,
    \begin{align*}
        \ctescorenorm{t}>\frac{\isvp}{2}
        \iff
        \ctescorenorm{0} > \frac{\isvp}{2}\left(
            \fwdmean{0}{t}^2 + \frac{2 \ctescorenorm{0}}{\isvp} \left(1-\fwdmean{0}{t}^2\right)
        \right)
        &\iff
        \ctescorenorm{0} > \frac{\isvp}{2}
        \fwdmean{0}{t}^2 + 
        \ctescorenorm{0}\left(1-\fwdmean{0}{t}^2\right)
        \\
        &\iff
        \left(\ctescorenorm{0}-\frac{\isvp}{2}\right)\fwdmean{0}{t}^2 > 0
        \eqsp,
    \end{align*}
    so that the corresponding time-dependent dissipativity parameter remains larger than the threshold $\isvp/2$ uniformly over $t\ge 0$.
\end{remark}

\subsection{Lyapunov contraction for the backward semigroup (\cref{prop:backward_drift_lyapunov})}
\label{subapp:lyapunov_semigroup}

We now turn to the time-reversed dynamics \eqref{eq:backward_SDE} and its associated semigroup \eqref{eq:backward_semigroup}. Using the propagated dissipativity of $\score[T-t]$, we derive a drift inequality for polynomial Lyapunov functions $\lyapunov{\ell}(x)=\|x\|^\ell$. This is a key ingredient for Harris theorem \citet{HairerMattingly2008} that we use with $\ell = 2$.

\noindent\textbf{\cref{prop:backward_drift_lyapunov}.} (restatement).
\emph{ Suppose that~\cref{assump:p0:score} holds. Let $\lyapunov{\ell}[\x] \eqdef \normEc{\x}^{\ell}$, for $\ell\geq2$. Then, there exist continuous functions
    $\ctescorenorm[t]{\cdot,\ell} ,\ctescoreoffset[t]{\cdot,\ell} :[0,T]\to\rsetpos$ such that, for all $0\le s<t\le T$ and all $\x\in\rset^\xdim$,
    \begin{align*}
        \bwdker{s}[t] \lyapunov{\ell}[\x]
        \le
        \multlyap{s}{t}[\ell] \lyapunov{\ell}[\x]
        + \biaslyap{s}{t}[\ell]
        \eqsp.
    \end{align*}
    where $\multlyap{s}{t}[\ell] \eqdef \exp \left(-\int_s^t \ctescorenorm[t]{v,\ell} \rmd v\right)$ and $\biaslyap{s}{t}[\ell] \eqdef \int_s^t \exp\left( - \int_u^t \ctescorenorm[t]{v,\ell} \rmd v \right) \ctescoreoffset[t]{u,\ell} \rmd u$.
    In particular, when $\ell=2$, 
    \begin{align*}
    \ctescorenorm[t]{t,2} = 2  \bwdnoisesch{t}\left( 2 \ctescorenorm{T-t} -\isvp \right),
    \quad \text{and} \quad
    \ctescoreoffset[t]{t,2} \eqdef 2 \bwdnoisesch{t}\left( 2 \ctescoreoffset{T-t} +\xdim \right) \eqsp.
    \end{align*}
}

\begin{proof}[Proof of \cref{prop:backward_drift_lyapunov}]
Note that for all $\x\in\rset^{\xdim}$ and all $t\in[0,T]$, 
\begin{align*}
    \nabla \lyapunov{\ell}[\x]  = \ell \normEc{\x}^{\ell-2} \x 
    \eqsp,
    \qquad
    \Delta \lyapunov{\ell}[\x] = \ell(\ell -2+\xdim) \normEc{\x}^{\ell-2} \eqsp.
\end{align*}
The infinitesimal generator associated with~\eqref{eq:backward_SDE}  satisfies
\begin{align*}
    \mathcal A_t \lyapunov{\ell}[\x]
    =
    \bwdnoisesch{t}
    \dotprod{\isvp\x + 2 \score[T-t][\x]}{\nabla \lyapunov{\ell}[\x]}
    + \bwdnoisesch{t} \Delta \lyapunov{\ell}[\x] \eqsp.
\end{align*}
From~\cref{prop:lyapunov_stability} together with~\cref{assump:p0:score} there exist continuous functions
$\ctescorenorm,\ctescoreoffset:[0,T]\to\rsetpos$ such that, for all $t\in[0,T]$ and all $\x\in\rset^\xdim$,
\begin{align} \label{eq:scalar_prod_contraction}
    \dotprod{\score[t][\x]}{\x}
    \le -\ctescorenorm{t} \normEc{\x}^2 + \ctescoreoffset{t} \eqsp.
\end{align}
As a consequence, 
\begin{align*}
\mathcal A_t \lyapunov{\ell}[\x]
\le
-a_t \lyapunov{\ell}[\x]+ b_t \normEc{\x}^{\ell-2}
\end{align*}
where
\begin{align*}
a_t \eqdef \ell \bwdnoisesch{t}\left( 2 \ctescorenorm{T-t} -\isvp \right),
\quad
b_t\eqdef \ell \bwdnoisesch{t}\left( 2 \ctescoreoffset{T-t} +\ell - 2 + \xdim \right) \eqsp,
\end{align*}
which finished the proof for $\ell = 2$. Moreover, if $\ell > 2$, for any $\eta_t > 0$, we can apply Young's inequality with conjugated exponents $\bar{p}, \bar{q}$ such that 
$1/\bar{p}=(\ell-2)/\ell$ and $1/\bar{q}=2/\ell$. This gives
\begin{align*}
    \normEc{\x}^{\ell-2}
    =
    \eta_t^{\frac{\ell-2}{\ell}}
    \normEc{\x}^{\ell-2}
    \eta_t^{-\frac{\ell-2}{\ell}}
    \leq
    \frac{\ell-2}{\ell}
    \eta_t
    \normEc{\x}^{\ell} + \frac{2}{\ell}\,\eta_t^{-\frac{\ell-2}{2}} 
    \eqsp.
\end{align*}
Take $\eta_t>0$ such that $a_t - \frac{\ell-2}{\ell} \eta_t b_t >a_t/2$.
Thus,
\begin{align*}
    \mathcal A_t \lyapunov{\ell}[\x]
    \leq 
    - \left( a_t - \frac{\ell-2}{\ell} \eta_t b_t \right)  \lyapunov{\ell}[\x]+ b_t \,\frac{2}{\ell}\,\eta_t^{-\frac{\ell-2}{2}}
    \leq 
    - \frac{a_t}{2} \lyapunov{\ell}[\x]+ b_t \,\frac{2}{\ell}\,\eta_t^{-\frac{\ell-2}{2}}
    \eqsp.
\end{align*}
It follows from Dynkin's formula that, for all $0\le s<t\le T$,
\begin{align*}
    \bwdker{s}[t] \lyapunov{\ell}[\x]
    \le
    \lyapunov{\ell}[\x] + \int_s^t \left( - \ctescorenorm[t]{v, \ell} \bwdker{s}[v] \lyapunov{\ell}[\x] + \ctescoreoffset[t]{v, \ell} \right) \rmd v
    \eqsp.
\end{align*}
where
\begin{align*}
    \ctescorenorm[t]{t, \ell} \eqdef 
    \frac{a_t}{2}
    \eqsp,
    \qquad
    \ctescoreoffset[t]{t, \ell}\eqdef b_t \,\frac{2}{\ell}\,\eta_t^{-\frac{\ell-2}{2}} \eqsp.
\end{align*}
The claimed bound then follows from Grönwall's lemma.
\end{proof}

\subsection{Propagation of the growth condition}
\label{subapp:growth}

We aim at showing the propagation of~\cref{assump:p0:hess} along the flow. 
To this end, we first study the time stability of the quantity $\jscore[t][x] + \score[t][x]\mt{\score[t][x]}$ in the VE setting, following an approach similar to that developed in \cref{subsec:gaussian-perturbation}.
The term plays a central role in the Lyapunov stability analysis. Indeed, as a direct consequence of~\cref{lem:rewrite_Hessian} and~\cref{cor:hess+score-bound_0}, this combination naturally appears when controlling the drift and growth properties of the dynamics, and thus governs the stability estimates. The following lemma further clarifies its interpretation by showing that this quantity is tightly connected to the conditional variance of the denoiser $\Xora_0 \mid \Xora_t$, providing a probabilistic meaning to the key terms involved in the Lyapunov analysis.

\begin{lemma}
\label[lemma]{lem:cond-var-identity}
    Assume that~\cref{assump:p0:score} holds. Then, for all $\x \in \rset^\xdim$,
    \begin{align}
    \label{eq:cond-var-identity}
        \jscore[t][\x]
        + \score[t][\x] \mt{\score[t][\x]}
        = \frac{1}{\fwdstd{0}{t}^4}\,\mathrm{Var}\!\left(\Xora_0 \mid \Xora_t=\x\right)
        -\frac{1}{\fwdstd{0}{t}^2}\Id_{\xdim}
        \eqsp.
    \end{align}
\end{lemma}

\begin{proof}
    Let $\gaussMarg{t}$ be the density of $\gaussiand{0}{\fwdvar{0}{t}\Id_{\xdim}}$.
    Then,
    \begin{align*}
        \nabla \gaussMarg{t}[\x]= -\frac{1}{\fwdstd{0}{t}^2}\x\,\gaussMarg{t}[\x]\eqsp,
        \qquad
        \nabla^2 \gaussMarg{t}[\x]
        =\left(\frac{1}{\fwdstd{0}{t}^4}\x\mt{\x}-\frac{1}{\fwdstd{0}{t}^2}\Id_{\xdim}\right)\gaussMarg{t}[\x]
        \eqsp.
    \end{align*}
    Since
    \[
        \fwdmarg{t}[\x]=\int_{\rset^\xdim} \fwdmarg{0}[y]\,\gaussMarg{t}[\x-y]\,dy
        \eqsp,
    \]
    As in~\cref{lem:score-representations}, from~\cref{assump:p0:score} and~\cref{lem:pdata-sub-gaussian}, we can differentiate under the integral sign to get 
    \[
    \nabla \fwdmarg{t}[\x]=\int \fwdmarg{0}[y]\,\nabla \gaussMarg{t}[\x-y]\,dy,
    \qquad
    \nabla^2 \fwdmarg{t}[\x]=\int \fwdmarg{0}[y]\,\nabla^2 \gaussMarg{t}[\x-y]\,dy.
    \]
    Similarly to the computations in~\cref{lem:score-representations}, applying Bayes’ formula together with the previous integration by parts formulae yields
    \begin{align}
    \label{eq:cond-var-matrix}
        \frac{\nabla^2 \fwdmarg{t}[\x]}{\fwdmarg{t}[\x]} 
        =
        \CPE{}{\Xora_t=\x}{\left(\frac{1}{\fwdstd{0}{t}^4}(\x-\Xora_0)\mt{(\x-\Xora_0)}-\frac{1}{\fwdstd{0}{t}^2}\Id_{\xdim}\right)}
        \eqsp.
    \end{align}
    Using $\score[t][\x]=\nabla \log \fwdmarg{t}[\x]=\nabla \fwdmarg{t}[\x]/\fwdmarg{t}[\x]$ and the formula
    \(
        \jscore[t][\x]=\nabla^2 \log \fwdmarg{t}[\x]=\nabla^2 \fwdmarg{t}[\x]/\fwdmarg{t}[\x]
        -\nabla \log \fwdmarg{t}[\x]\mt{\nabla \log \fwdmarg{t}[\x]},
    \)
    we obtain
    \begin{align*}
        \nabla\score[t][\x]
        + \score[t][\x] \mt{\score[t][\x]}
        =
        \CPE{}{\X_t=\x}{\left(\frac{1}{\fwdstd{0}{t}^4}(\x-\Xora_0)\mt{(\x-\Xora_0)}-\frac{1}{\fwdstd{0}{t}^2}\Id_{\xdim}\right)}
        \eqsp.
    \end{align*}
    Finally, since $\CPE{}{\X_t=\x}{\x-\X_0}
    =\x-\CPE{}{\X_t=\x}{\X_0}$,
    a direct expansion yields
    \[
    \CPE{}{\Xora_t=\x}{\left(\frac{1}{\fwdstd{0}{t}^4}(\x-\Xora_0)\mt{(\x-\Xora_0)}\right)}
    =
    \mathrm{Var}(\Xora_0\mid \Xora_t=x),
    \]
    which plugged into \eqref{eq:cond-var-matrix}, gives \eqref{eq:cond-var-identity}.
\end{proof}

\begin{proposition}
    \label[proposition]{prop:growth_stability}
    Suppose that~\cref{assump:p0:hess} holds and $\isvp=0$. Then, there exists a continuous function $[0,T]\ni t\mapsto\ctegrowth{t}\in\rsetpos$ such that, for all $\x \in \rset^\xdim$,
    \begin{align}
        \label{eq:growth_stability1}
        \normEc{\score[t][\x]}
        &\le \sqrt{\ctegrowth{t}} \big(1 + \normEc{\x}^{p+1}\big)
        \\
        \label{eq:growth_stability2}
        \normFr{\jscore[t][\x] +  \score[t][\x] \mt{\score[t][\x]}}
        &\le \ctegrowth{t} \big(1 + \normEc{\x}^{2p+2}\big)
        \eqsp.
    \end{align}
\end{proposition}
\begin{proof}
    First, we prove \eqref{eq:growth_stability1}. Taking the norm together with Jensen's inequality to \eqref{eq:score1}, we get
    \begin{align*}
        \normEc{\score[t][\x]} \leq  \CPE{}{\Xora_t = \x}{\normEc{\score[0][\Xora_0]}} \eqsp.
    \end{align*}
    Using \cref{lem:norm-score-bound}, it follows
    \begin{align*}
        \normEc{\score[t][\x]} \leq  \ctescorepoly \left( 1 + \CPE{}{\Xora_t = \x}{\normEc{\Xora_0}^{p+1}} \right)
        \eqsp.
    \end{align*}
    Applying \cref{prop:backward_drift_lyapunov} with $\ell = p+1$ yields
    \begin{align*}
        \normEc{\score[t][\x]} \leq 
        \ctestable{0} \left(
            1 + 
            \multlyap{t}{0}[p+1]\normEc{\x}^{p+1}
            +
            \biaslyap{t}{0}[p+1]
        \right)
        \eqsp.
    \end{align*}
    Therefore, we can find $\ctegrowth{t}$ such that \eqref{eq:growth_stability1} holds.

    Similarly, we focus on \eqref{eq:growth_stability2}. Applying the Frobenius norm together with Jensen's inequality to \eqref{eq:relation-hessian-score-in_time}, we obtain
    \begin{align*}
        \normFr{\jscore[t][\x] + \score[t][\x]^{\otimes 2}} \leq  \CPE{}{\Xora_t = \x}{\normFr{\jscore[0][\Xora_0] + \score[0][\Xora_0]^{\otimes 2}}} \eqsp.
    \end{align*}
    Thus, from~\cref{cor:hess+score-bound_0},
    \begin{align*}
        \normFr{\jscore[t][\x] + \score[t][\x]^{\otimes 2}} \leq \ctestable{0} \left( 1 + \CPE{}{\Xora_t = \x}{\normEc{\Xora_0}^{2p+2}} \right) \eqsp.
    \end{align*}
    Applying \cref{prop:backward_drift_lyapunov} with $\ell = 2p+2$, it follows
    \begin{align*}
        \normFr{\jscore[t][\x] + \score[t][\x]^{\otimes 2}}
        \leq
        \ctestable{0} \left(
            1 + 
            \multlyap{t}{0}[2p+2]\normEc{\x}^{2p+2}
            +
            \biaslyap{t}{0}[2p+2]
        \right)
        \eqsp.
    \end{align*}
    This means that, eventually increasing its value, we can find $\ctegrowth{t}$ such that \eqref{eq:growth_stability2} holds. The continuity of the function $[0,T]\ni t\mapsto\ctegrowthHess{t}\in\rsetpos$ follows immediately from the definition of the $t\mapsto(\multlyap{t}{0}[\ell],\biaslyap{t}{0}[\ell])$ from \cref{prop:backward_drift_lyapunov}, for $\ell\geq 2$.

\end{proof}

We are now in a position to leverage the previous result to establish the propagation in time of~\cref{assump:p0:hess}. 
This argument applies to both the VE and VP cases, thereby ensuring that the corresponding regularity and stability properties of the score hold uniformly in time.

\begin{corollary}
    \label[corollary]{cor:bound_hessian}
    Suppose that~\cref{assump:p0} holds. Then, there exists $[0,T]\ni t\mapsto\ctegrowthHess{t}\in\rsetpos$ such that, for all $\x \in \rset^\xdim$,
    it holds that
    \begin{align*}
        \normFr{\jscore[t][\x]} \leq \ctegrowthHess{t} \big(1 + \normEc{\x}^{2p+2}\big)\eqsp,
        \qquad
        \normEc{\score[t][\x]} \leq \ctegrowthHess{t} \big(1 + \normEc{\x}^{2p+3}\big)
        \eqsp.
    \end{align*}
\end{corollary}

\begin{proof}
    Consider first the VE case.
    By combining~\cref{prop:growth_stability,prop:lyapunov_stability} and using the fact that for all $v, w \in \rset^{\xdim}$, $\normFr{v \mt{w}} = |\dotprod{v}{w}|$, we have that
    \begin{align*}
        \normFr{\jscore[t][\x]}
        &\leq
        \normFr{\jscore[t][\x] + \score[t][\x]^{\tensprod 2}} + \normEc{\score[t][\x]}^2
        \leq
        2\ctegrowth{t} \big(1 + \normEc{\x}^{2p+2}\big)
        \eqsp.
    \end{align*}
    Applying \cref{lem:vetovp}, we recover the bound \eqref{cor:bound_hessian} in the variance–preserving case as well.
    
    To establish the polynomial growth of the score function $x \mapsto \score[t][x]$, we proceed as in \cref{lem:norm-score-bound}, using the intermediate value theorem together with the previously derived bound on the Jacobian of the score function.
\end{proof}

\section{Localized Doeblin minorization condition for the backward process (Proof of \cref{prop:minorization_main})}
\label{app:proofs:minorization}

This section proves a localized Doeblin (minorization) condition for the backward Markov kernel $\bwdker{s}[t]$ on bounded sets that is informally stated in \cref{prop:minorization_main}. Concretely, we show that for any radius $r>0$ there exist a constant $\minconst{t}{s}(r)\in(0,1)$ and a probability measure $\minmeas{t}{s}$ such that $\bwdker{s}[t](\x,\cdot)$ dominates $\minconst{t}{s}(r)\minmeas{t}{s}(\cdot)$ uniformly over $\x\in B(0,r)$. Such result is the second ingredient of Harris theorem \citet{HairerMattingly2008}.

\begin{proposition}[Formal statement of \cref{prop:minorization_main}]
\label{prop:minorization-appendix}
    Let $0 \le s < t \le T$. Assume that~\cref{assump:p0} hold. 
    Then, for all $r>0$, there exist a constant $\minconst{s}{t}[r] \in (0,1)$ and a probability measure
    $\minmeas{s}{t}$ on $\rset^\xdim$ such that, for all $\x\in \ball{0}{r}$, $\set{A} \in\mathcal B(\rset^\xdim)$,
    \begin{align*}
    \bwdker{s}[t][\x][\set{A}]\geq \minconst{s}{t}[r] \minmeas{s}{t}[\set{A}]\eqsp,
    \end{align*}
    where
    \begin{align*}
    \minconst{s}{t}[r] = \frac{ (\pi \fwdstd{0}{T-s}^2)^{\xdim/2}}{(\fwdmean{T-t}{T-s}\fwdmean{0}{T-t})^{\xdim}\max\left\{(2\pi \fwdstd{0}{T-s}^2)^{\xdim/2}\fwdmean{0}{T-s}^{-\xdim}\norminfty{\pdata}, 1\right\}}{\exp\left(-\frac{r^2}{\fwdvar{T-t}{T-s}}\right)}\minmeasintcte{s}{t}\eqsp,
    \end{align*}
    with 
    \begin{align*}
    \minmeasintcte{s}{t} \eqdef \PE{}{\gaussiand{0}{\frac{\fwdvar{T-t}{T-s}+2\fwdvar{0}{T-t}\fwdmean{T-t}{T-s}^2}{2\fwdmean{0}{T-t}^2 \fwdmean{T-t}{T-s}^2}\Id_\xdim}[\X_0] } \eqsp.
    \end{align*}
\end{proposition}
\begin{proof}
    Note that for all $\x_t\in \rset^\xdim$, using the notations of \cref{lem:forward_process_law},
\begin{align}
    \fwdmarg{t}[\x_t] &= \int \pdata[\x_0] \gaussiand{\fwdmean{0}{t}\x_0}{\fwdvar{0}{t}\Id_\xdim}[\x_t] \rmd \x_0 \nonumber\\
    &= \frac{1}{\fwdmean{0}{t}^\xdim}\int \pdata[\frac{u}{\fwdmean{0}{t}}] \gaussiand{\x_t}{\fwdvar{0}{t}\Id_\xdim}[u] \rmd u \nonumber\\
    &\leq \max\left\{\fwdmean{0}{t}^{-\xdim}\norminfty{\pdata}, (2\pi \fwdstd{0}{t}^2)^{-\xdim/2}\right\}\nonumber\\
     & \leq  (2\pi \fwdstd{0}{t}^2)^{-\xdim/2}\max\left\{(2\pi \fwdstd{0}{t}^2)^{\xdim/2}\fwdmean{0}{t}^{-\xdim}\norminfty{\pdata}, 1\right\}\eqsp.\label{eq:ub:pt}
\end{align}
Let $r>0$ and assume that $\x_t \in \ball{0}{r}$.  Using that for all $a,b\in\rset$, $(a+b)^2 \leq 2 a^2 + 2 b^2$, we have 
\begin{align}
    \fwdtrans{s}{t}[\x_s][\x_t] &\geq (2\pi \fwdstd{s}{t}^2)^{-\xdim/2}\exp\left(-\frac{\|\fwdmean{s}{t}\x_s\|^2}{\fwdvar{s}{t}}\right) \exp\left(-\frac{\|\x_t\|^2}{\fwdvar{s}{t}}\right) \nonumber\\
    &\geq (\sqrt{2}\fwdmean{s}{t})^{-\xdim}
    \gaussiand{0}{\frac{\fwdvar{s}{t}}{2 \fwdmean{s}{t}^2}\Id_{\xdim}}[\x_s]
    \exp\left(-\frac{r^2}{\fwdvar{s}{t}}\right)\eqsp.\label{eq:lb:fwtrans}
\end{align}
For all $0<s<t<T$, $\set{A} \in \borelians{\rset^\xdim}$, and $\x_t\in \rset^\xdim$,
\begin{align*}
\bwdker{T-t}[T-s][\x_t][\set{A}]&= \CPE{}{\Xora_t=\x_t}{\CPE{}{\Xora_t=\x_t, \Xora_0}{\indi{\set{A}}{\Xora_s}}}\\
&=\int \indi{\set{A}}{\x_s}\pdata[\x_0]\bridge{t, 0}{s}[\x_t, \x_0][\x_s]\frac{\fwdtrans{0}{t}[\x_0][\x_t]}{\fwdmarg{t}[\x_t]} \rmd \x_0\rmd \x_s\eqsp.
\end{align*}
By Markovianity of the forward process, we have, for all $\x_0, \x_s, \x_t\in \rset^\xdim$, 
\begin{align*}
   \bridge{t, 0}{s}[\x_t, \x_0][\x_s] \fwdtrans{0}{t}[\x_0][\x_t] &= \fwdtrans{0}{s}[\x_0][\x_s]\fwdtrans{s}{t}[\x_s][\x_t]\eqsp.
\end{align*}
Therefore, by \eqref{eq:lb:fwtrans}, 
\begin{align*}
    \bwdker{T-t}[T-s][\x_t][\set{A}]
    &\geq \fwdmarg{t}[\x_t]^{-1} (\sqrt{2}\fwdmean{s}{t})^{-\xdim} \exp\left(-\frac{r^2}{\fwdvar{s}{t}}\right)
    \\
    &\qquad \times \int \indi{\set{A}}{\x_s}\fwdtrans{0}{s}[\x_0][\x_s]\gaussiand{0}{\frac{\fwdvar{s}{t}}{2 \fwdmean{s}{t}^2}\Id_{\xdim}}[\x_s]\pdata[\x_0]\rmd \x_0\rmd \x_s\eqsp.
\end{align*}
By Gaussian conjugation, for all $\x_0, \x_s \in \rset^{\xdim}$,  we have
\begin{align*}
    &\fwdtrans{0}{s}[\x_0][\x_s]\gaussiand{0}{\frac{\fwdvar{s}{t}}{2 \fwdmean{s}{t}^2}\Id_{\xdim}}[\x_s] 
    \\
    &= \fwdmean{0}{s}^{-\xdim}\gaussiand{\frac{\fwdvar{s}{t}\fwdmean{0}{s}\x_0}{\fwdvar{0}{t}+\fwdvar{0}{s}\fwdmean{s}{t}^2}}{\frac{\fwdvar{s}{t}\fwdvar{0}{s}}{\fwdvar{0}{t}+\fwdvar{0}{s}\fwdmean{s}{t}^2}\Id_{\xdim}}[\x_s]\gaussiand{0}{\frac{\fwdvar{s}{t}+2\fwdvar{0}{s}\fwdmean{s}{t}^2}{2\fwdmean{0}{s}^2 \fwdmean{s}{t}^2}\Id_\xdim}[\x_0]
    \eqsp.
\end{align*}
Thus, if we define for any $\set{B} \in \borelians{\rset^{\xdim}}$
\begin{align*}
    \minmeasintcte{T-t}{T-s} &\eqdef \PE{}{\gaussiand{0}{\frac{\fwdvar{s}{t}+2\fwdvar{0}{s}\fwdmean{s}{t}^2}{2\fwdmean{0}{s}^2 \fwdmean{s}{t}^2}\Id_\xdim}[\Xora_0] } \eqsp,\\
    \mu_{0,s,t}(\set{B}) &\eqdef \PE{}{\indi{\set{B}}{\Xora_0}\gaussiand{0}{\frac{\fwdvar{s}{t}+2\fwdvar{0}{s}\fwdmean{s}{t}^2}{2\fwdmean{0}{s}^2 \fwdmean{s}{t}^2}\Id_\xdim}[\Xora_0]}/ \minmeasintcte{T-t}{T-s} \eqsp,\\
    \minmeas{T-t}{T-s}[\set{A}] &\eqdef \int \indi{\set{A}}{\x_s}\gaussiand{\frac{\fwdvar{s}{t}\fwdmean{0}{s}\x_0}{\fwdvar{0}{t}+\fwdvar{0}{s}\fwdmean{s}{t}^2}}{\frac{\fwdvar{s}{t}\fwdvar{0}{s}}{\fwdvar{0}{t}+\fwdvar{0}{s}\fwdmean{s}{t}^2}\Id_{\xdim}}[\x_s]\mu_{0,s,t}(\rmd \x_0)\rmd \x_s\eqsp.
\end{align*} 
we can complete the proof by \eqref{eq:ub:pt} with
\begin{align*}
    \minconst{T-t}{T-s} = \frac{ (\pi \fwdstd{0}{t}^2)^{\xdim/2}}{(\fwdmean{s}{t}\fwdmean{0}{s})^{\xdim}\max\left\{(2\pi \fwdstd{0}{t}^2)^{\xdim/2}\fwdmean{0}{t}^{-\xdim}\norminfty{\pdata}, 1\right\}}{\exp\left(-\frac{r^2}{\fwdvar{s}{t}}\right)}\minmeasintcte{T-t}{T-s}\eqsp.
    \end{align*}

\end{proof}

\section{Quantitative bounds for the Lyapunov and Harris contraction constants} 
\label{sec:constants}

By~\Cref{prop:minorization-appendix}, for all $0\leq s < t \leq T$, note that using $\minmeasintcte{T-t}{T-s} \leq \norminfty{\pdata}$ yields
\begin{align*}
    \minconst{T-t}{T-s}[r] \leq \frac{(\pi \fwdstd{0}{t}^2)^{\xdim/2} \norminfty{\pdata}}{(\fwdmean{s}{t}\fwdmean{0}{s})^{\xdim}\max\left\{(2\pi \fwdstd{0}{t}^2)^{\xdim/2}\fwdmean{0}{t}^{-\xdim}\norminfty{\pdata}, 1\right\}}{\exp\left(-\frac{r^2}{\fwdvar{s}{t}}\right)}\eqsp.
\end{align*}
Note that 
\begin{align*}
    &\frac{(\pi \fwdstd{0}{t}^2)^{\xdim/2} \norminfty{\pdata}}{(\fwdmean{s}{t}\fwdmean{0}{s})^{\xdim}\max\left\{(2\pi \fwdstd{0}{t}^2)^{\xdim/2}\fwdmean{0}{t}^{-\xdim}\norminfty{\pdata}, 1\right\}}
    \\
    &= \min\left\{\frac{(\pi \fwdstd{0}{t}^2)^{\xdim/2} \norminfty{\pdata}}{(\fwdmean{s}{t}\fwdmean{0}{s})^{\xdim}},\frac{\fwdmean{0}{t}^{\xdim}}{2^{d/2}(\fwdmean{s}{t}\fwdmean{0}{s})^{\xdim}}\right\}\eqsp,
\end{align*}
which leads to 
\begin{align*}
    \minconst{T-t}{T-s}[r]  \leq \min\left\{\frac{(\pi \fwdstd{0}{t}^2)^{\xdim/2} \norminfty{\pdata}}{(\fwdmean{s}{t}\fwdmean{0}{s})^{\xdim}},\frac{\fwdmean{0}{t}^{\xdim}}{2^{d/2}(\fwdmean{s}{t}\fwdmean{0}{s})^{\xdim}}\right\} \exp\left(-\frac{r^2}{\fwdvar{s}{t}}\right) \eqsp.
\end{align*}

\paragraph{Variance Exploding case.}
Consider now the VE case, \ie, $\isvp=0$. Note that
\begin{align*}
    \begin{split}
    \multlyap{s}{t}[2] = \exp\left(-\int_{s}^{t}2\bwdnoisesch{v}\left(\frac{2\ctescorenorm{0}}{1 + 2\ctescorenorm{0} \fwdvar{0}{T-v}}\right) \rmd v\right) &= \exp\left(\int_{s}^{t}\frac{\rmd \fwdvar{0}{T-v}}{\rmd v}\left(\frac{2\ctescorenorm{0}}{1 + 2\ctescorenorm{0} \fwdvar{0}{T-v}}\right) \rmd v\right) \\
    &=\left(\frac{1 + 2\ctescorenorm{0}\fwdvar{0}{T-t}}{1 + 2\ctescorenorm{0}\fwdvar{0}{T-s}}\right)
    \end{split}
\end{align*}
and
\begin{align*}
    \biaslyap{s}{t}[2] &= \int_{s}^{t}\exp\left(-\int_{u}^{t}2\bwdnoisesch{v}(2\ctescorenorm{T-v})\rmd v\right)2\bwdnoisesch{u}(2\ctescoreoffset{T-u} + \xdim)\rmd u
    \\
    &= -\int_{s}^{t}\left(\frac{1 + 2\ctescorenorm{0}\fwdvar{0}{T-t}}{1 + 2\ctescorenorm{0}\fwdvar{0}{T-u}}\right)\frac{\rmd \fwdvar{0}{T-u}}{\rmd u}\frac{2(\ctescoreoffset{0}+\xdim)}{1 + 2\ctescorenorm{0}\fwdvar{0}{T-u}}\rmd u
    \\
    &\qquad\qquad\qquad\qquad\qquad\qquad
    -\xdim(1+ 2\ctescorenorm{0}\fwdvar{0}{T-t})\int_{s}^{t}\left(\frac{1}{1 + 2\ctescorenorm{0}\fwdvar{0}{T-u}}\right)\frac{\rmd \fwdvar{0}{T-u}}{\rmd u}\rmd u
    \\
    &=-2\left(1 + 2\ctescorenorm{0}\fwdvar{0}{T-t}\right)(\ctescoreoffset{0}+\xdim)\int_{s}^{t}\left(\frac{1}{1 + 2\ctescorenorm{0}\fwdvar{0}{T-u}}\right)^2\frac{\rmd \fwdvar{0}{T-u}}{\rmd u}\rmd u
    \\
    &\qquad\qquad\qquad\qquad\qquad\qquad
    + \xdim \frac{(1 + 2\ctescorenorm{0}\fwdvar{0}{T-t})}{2\ctescorenorm{0}}\log\left(\frac{1 + 2\ctescorenorm{0}\fwdvar{0}{T-s}}{1 + 2\ctescorenorm{0}\fwdvar{0}{T-t}}\right)
    \\
    &=\left(1 + 2\ctescorenorm{0}\fwdvar{0}{T-t}\right)\frac{\ctescoreoffset{0}+\xdim}{\ctescorenorm{0}}\int_{s}^{t}\frac{\rmd}{\rmd u}\left(\frac{1}{1 + 2\ctescorenorm{0}\fwdvar{0}{T-u}}\right)\rmd u
    \\
    &\qquad\qquad\qquad\qquad\qquad\qquad
    + \xdim \frac{(1 + 2\ctescorenorm{0}\fwdvar{0}{T-t})}{2\ctescorenorm{0}}\log\left(\frac{1 + 2\ctescorenorm{0}\fwdvar{0}{T-s}}{1 + 2\ctescorenorm{0}\fwdvar{0}{T-t}}\right)
    \\
    &= \left(1 + 2\ctescorenorm{0}\fwdvar{0}{T-t}\right)\frac{\ctescoreoffset{0}+\xdim}{\ctescorenorm{0}}\left(\frac{1}{1 + 2\ctescorenorm{0}\fwdvar{0}{T-t}}- \frac{1}{1 + 2\ctescorenorm{0}\fwdvar{0}{T-s}}\right)
    \\
    &\qquad\qquad\qquad\qquad\qquad\qquad
    + \xdim \frac{(1 + 2\ctescorenorm{0}\fwdvar{0}{T-t})}{2\ctescorenorm{0}}\log\left(\frac{1 + 2\ctescorenorm{0}\fwdvar{0}{T-s}}{1 + 2\ctescorenorm{0}\fwdvar{0}{T-t}}\right)
    \\
    &= 2(\ctescoreoffset{0} + \xdim)\frac{\fwdvar{T-t}{T-s}}{1 + 2\ctescorenorm{0}\fwdvar{0}{T-s}}+\xdim \frac{(1 + 2\ctescorenorm{0}\fwdvar{0}{T-t})}{2\ctescorenorm{0}}\log\left(\frac{1 + 2\ctescorenorm{0}\fwdvar{0}{T-s}}{1 + 2\ctescorenorm{0}\fwdvar{0}{T-t}}\right)\eqsp,
\end{align*}
where we have used that $\fwdvar{T-t}{T-s} = \fwdvar{0}{T-s} - \fwdvar{0}{T-t}$. Therefore, we have that
\begin{align}
\label{eq:r-carre}
    \frac{2 \biaslyap{s}{t}[2]}{1 - \multlyap{s}{t}[2]} = 2\frac{\ctescoreoffset{0} + \xdim}{\ctescorenorm{0}} + \xdim (1 + 2\ctescorenorm{0}\fwdvar{0}{T-t})\log\left(\frac{1 + 2\ctescorenorm{0}\fwdvar{0}{T-s}}{1 + 2\ctescorenorm{0}\fwdvar{0}{T-t}}\right)\frac{1 + 2\ctescorenorm{0}\fwdvar{0}{T-s}}{2\ctescorenorm{0}\fwdvar{T-s}{T-t}}\eqsp.
\end{align}
\paragraph{Variance Preserving case.}
For the VP case ($\isvp>0$), note that, since $\fwdvar{0}{u}=\isvp^{-1}\left(1-\fwdmean{0}{u}^2\right)$ for any $u\geq0$, we have
\begin{align*}
    \begin{split}
        \multlyap{s}{t}[\ell]
        &=
        \exp \left(-\int_s^t \ctescorenorm[t]{v,\ell}\,\rmd v\right)
        =
        \exp \left(-\int_s^t 
            2 \bwdnoisesch{t}\big(2\ctescorenorm{T-v}-\isvp\big)
        \,\rmd v\right)
        \\
        &= \exp\left(
            \int_{s}^{t}
            2\pmixtime
            \bwdnoisesch{v}\fwdmean{0}{T-v}^2
            \left(
                \frac{
                    1-\frac{2\ctescorenorm{0}}{\pmixtime}
                }{
                    \left(1-\frac{2\ctescorenorm{0}}{\pmixtime}\right)\fwdmean{0}{T-v}^2 + \frac{2\ctescorenorm{0}}{\pmixtime}
                }
            \right) \rmd v \right)
        \\
        &=
        \exp\left(
            \int_{s}^{t}
            \frac{\rmd \fwdmean{0}{T-v}^2}{\rmd v}\left(
                \frac{
                    1-\frac{2\ctescorenorm{0}}{\pmixtime}
                }{
                    \left(1-\frac{2\ctescorenorm{0}}{\pmixtime}\right)\fwdmean{0}{T-v}^2 + \frac{2\ctescorenorm{0}}{\pmixtime}
                }
            \right) \rmd v\right)
        \\
        &=
        \frac{
            \left(1-\frac{2\ctescorenorm{0}}{\pmixtime}\right)\fwdmean{0}{T-t}^2 + \frac{2\ctescorenorm{0}}{\pmixtime}
        }{
            \left(1-\frac{2\ctescorenorm{0}}{\pmixtime}\right)\fwdmean{0}{T-s}^2 + \frac{2\ctescorenorm{0}}{\pmixtime}
        }
        =
        \frac{
            \fwdmean{0}{T-t}^2 + 2\ctescorenorm{0}\fwdvar{0}{T-t}
        }{
            \fwdmean{0}{T-s}^2 + 2\ctescorenorm{0}\fwdvar{0}{T-s}
        }
    \end{split}
\end{align*}
and
\begin{align*}
    \biaslyap{s}{t}[2] &= \int_{s}^{t}\exp\left(-\int_{u}^{t}2\bwdnoisesch{v}(2\ctescorenorm{T-v}-\isvp)\rmd v\right)2\bwdnoisesch{u}(2\ctescoreoffset{T-u} + \xdim)\rmd u \\
    &=
    \int_{s}^{t}2\bwdnoisesch{v}\left(
        \frac{
            \fwdmean{0}{T-t}^2 + 2\ctescorenorm{0}\fwdvar{0}{T-t}
        }{
            \fwdmean{0}{T-v}^2 + 2\ctescorenorm{0}\fwdvar{0}{T-v}
        }
    \right)
    \left(
        \frac{
            2(\ctescoreoffset{0}+\xdim)\fwdmean{0}{T-v}^2
        }{
            \fwdmean{0}{T-v}^2 + 2\ctescorenorm{0}\fwdvar{0}{T-v}
        }+\xdim
    \right)
    \rmd u\\
    &=
    \int_{s}^{t}2\bwdnoisesch{v}\left(
        \frac{
            \left(
                1-\frac{2\ctescorenorm{0}}{\isvp}
            \right)\fwdmean{0}{T-t}^2 + \frac{2\ctescorenorm{0}}{\isvp}
        }{
            \left(
                1-\frac{2\ctescorenorm{0}}{\isvp}
            \right)\fwdmean{0}{T-v}^2 + \frac{2\ctescorenorm{0}}{\isvp}
        }
    \right)
    \left(
        \frac{
            2(\ctescoreoffset{0}+\xdim)\fwdmean{0}{T-v}^2
        }{
            \left(
                1-\frac{2\ctescorenorm{0}}{\isvp}
            \right)\fwdmean{0}{T-v}^2 + \frac{2\ctescorenorm{0}}{\isvp}
        }+\xdim
    \right)
    \rmd u
    \\
    &=
    \int_{s}^{t}
    \frac{1}{\isvp \fwdmean{0}{T-u}^2}\frac{\rmd \fwdmean{0}{T-u}^2}{\rmd u}\left(
        \frac{
            \left(
                1-\frac{2\ctescorenorm{0}}{\isvp}
            \right)\fwdmean{0}{T-t}^2 + \frac{2\ctescorenorm{0}}{\isvp}
        }{
            \left(
                1-\frac{2\ctescorenorm{0}}{\isvp}
            \right)\fwdmean{0}{T-v}^2 + \frac{2\ctescorenorm{0}}{\isvp}
        }
    \right)
    \left(
        \frac{
            2(\ctescoreoffset{0}+\xdim)\fwdmean{0}{T-v}^2
        }{
            \left(
                1-\frac{2\ctescorenorm{0}}{\isvp}
            \right)\fwdmean{0}{T-v}^2 + \frac{2\ctescorenorm{0}}{\isvp}
        }+\xdim
    \right)
    \rmd u
    \eqsp,
\end{align*}
where we used the definition of $\fwdmean{0}{t}^2=\exp\left(-2\isvp\int_{0}^{t}\noisesch{v}\rmd v\right)$ together with $\fwdvar{0}{t}= \frac{1}{\isvp}\left(1 - \fwdmean{0}{t}^2\right)$, for $t\in[0,T]$, when $\isvp>0$. Using the change of variable $y(u) = \fwdmean{0}{T-u}^2$, we can rewrite the last integral as
\begin{align*}
    &\biaslyap{s}{t}[2]
    \\
    &= 
    \left(
        \left(
            1-\frac{2\ctescorenorm{0}}{\isvp}
        \right)\fwdmean{0}{T-t}^2 + \frac{2\ctescorenorm{0}}{\isvp}
    \right)
    \int_{\fwdmean{0}{T-s}^2}^{\fwdmean{0}{T-t}^2}
    \left(
        \frac{
            2(\ctescoreoffset{0}+\xdim)/\isvp
        }{
        \left(
            \left(
                1-\frac{2\ctescorenorm{0}}{\isvp}
            \right)y + \frac{2\ctescorenorm{0}}{\isvp}
        \right)^2
        }+
        \frac{
            \xdim/\isvp
        }{
        y\left(
            \left(
                1-\frac{2\ctescorenorm{0}}{\isvp}
            \right)y + \frac{2\ctescorenorm{0}}{\isvp}
        \right)
        }
    \right)
    \rmd y
    \\
    &=
    \frac{2(\ctescoreoffset{0}+\xdim)}{\isvp-2\ctescorenorm{0}}
    \left(
        1- \frac{
            \left(
                1-\frac{2\ctescorenorm{0}}{\isvp}
            \right)\fwdmean{0}{T-t}^2 + \frac{2\ctescorenorm{0}}{\isvp}
        }{
            \left(
                1-\frac{2\ctescorenorm{0}}{\isvp}
            \right)\fwdmean{0}{T-s}^2 + \frac{2\ctescorenorm{0}}{\isvp}
        }
        \right)
    \\
    &\qquad\qquad\qquad
    +
    \frac{\xdim}{2\ctescorenorm{0}}
    \left(\left(
                1-\frac{2\ctescorenorm{0}}{\isvp}
            \right)\fwdmean{0}{T-t}^2 + \frac{2\ctescorenorm{0}}{\isvp}\right)
    \log\left(
        \frac{
            \fwdmean{0}{T-t}^2\left(\left(
                1-\frac{2\ctescorenorm{0}}{\isvp}
            \right)\fwdmean{0}{T-s}^2 + \frac{2\ctescorenorm{0}}{\isvp}\right)
        }{
            \fwdmean{0}{T-s}^2  \left(\left(
                1-\frac{2\ctescorenorm{0}}{\isvp}
            \right)\fwdmean{0}{T-t}^2 + \frac{2\ctescorenorm{0}}{\isvp}\right)
        }
    \right)
    \\
    &=
    2\isvp(\ctescoreoffset{0}+\xdim)
    \left(
        \frac{
            \fwdmean{0}{T-s}^2-\fwdmean{0}{T-t}^2
        }{
            \fwdmean{0}{T-s}^2 + 2\ctescorenorm{0}\fwdvar{0}{T-s}
        }
        \right)
    \\
    &\qquad\qquad\qquad
    +
    \frac{\xdim}{2\ctescorenorm{0}}
    \left(\fwdmean{0}{T-t}^2 + 2\ctescorenorm{0}\fwdvar{0}{T-t}\right)
    \log\left(
        \frac{
            \fwdmean{0}{T-t}^2\left(\fwdmean{0}{T-s}^2 + 2\ctescorenorm{0}\fwdvar{0}{T-s}\right)
        }{
            \fwdmean{0}{T-s}^2  \left(\fwdmean{0}{T-t}^2 + 2\ctescorenorm{0}\fwdvar{0}{T-t}\right)
        }
    \right)
    \eqsp.
\end{align*}

\subsection{Explicit lower bound on the mixing gap}
\label{app:explicit-lower-bound}

In this subsection, we derive an explicit lower bound on the mixing gap $1-\pmixtime[s][t]$ by optimizing the choice of the radius parameter entering the Harris-type contraction estimate.

Fix $s < t$.
We reparametrize the radius $r$ by $r^2 = (1+\xi)r_c^2$, for $\xi >0$.
Note that in this case, we have
\begin{align}
    \multlyap{s}{t} + \frac{2 \biaslyap{s}{t}}{r^2} = \multlyap{s}{t} + \frac{1 -\multlyap{s}{t}}{1+\xi} \eqsp. 
\end{align}
To reduce some of the parameters, we will make two arbitrary choices:
\begin{align}
    \alpha_0 = \minconst{s}{t}[\xi] / 2\eqsp, \qquad 
    1 - \eta_0 = \frac{(1-\multlyap{s}{t})\xi}{2(1 + \xi)}\eqsp.
\end{align}
Those choices amount to choosing $\alpha_0$ and $\eta_0$ in the middle of the available interval, which is determined once $\xi$ is fixed.
From \eqref{eq:alpha_bar} and using that $b^r_{s,t} \eqdef \alpha_0/\biaslyap{s}{t}[2] = \minconst{s}{t}[\xi] / (2 \biaslyap{s}{t}[2])$ and $\minconst{s}{t}[\xi] = \minconstnor{s}{t}\exp(-(1+\xi) r_c^2 / \fwdvar{T-t}{T-s})$
\begin{align*}
1 - \pmixtime[s][t] & =     \left[\left(\minconst{s}{t}[r]-\alpha_0\right)\right] \wedge \frac{r^2 b^r_{s,t}(1- \eta_0)}{2+r^2 b^r_{s,t}} \\
& = \frac{\minconst{s}{t}[\xi]}{2}  \wedge \frac{r^2 \minconstnor{s}{t} \frac{(1-\multlyap{s}{t})\xi}{2(1 + \xi)} }{4 \biaslyap{s}{t}[2] \exp\left\{ (1+\xi) r_c^2 / \fwdvar{T-t}{T-s} \right\} +r^2 \minconstnor{s}{t} }
\end{align*}
It follows that,
\begin{align} \label{eq=1_alphabarre}
    1 - \pmixtime[s][t](\xi) =  \minconstnor{s}{t}\minimum{\frac{1}{2} \exp\left(-\frac{(1+\xi) r_c^2}{\fwdvar{T-t}{T-s}}\right)}{\frac{(1-\multlyap{s}{t})\xi}{2(1 + \xi)\minconstnor{s}{t} + 4 (1 - \multlyap{s}{t})\exp\left(\frac{(1+\xi) r_c^2}{\fwdvar{T-t}{T-s}}\right)}}\eqsp.
\end{align}
We have that $\lim_{\xi \rightarrow 0^{+}} 1 - \pmixtime[s][t](\xi) = 0$ and $\lim_{\xi \rightarrow \infty} 1 - \pmixtime[s][t](\xi) = 0$, with positive derivative at $0$, which imply the existence of an optimal $\xi_\star \in (0, \infty)$.

To produce a more precise bound, we focus on $\xi$ such that 
\begin{align*}
    \frac{1}{2} \exp\left(-\frac{(1+\xi) r_c^2}{\fwdvar{T-t}{T-s}}\right) = \frac{(1-\multlyap{s}{t})\xi}{2(1 + \xi)\minconstnor{s}{t} + 4 (1 - \multlyap{s}{t})\exp\left(\frac{(1+\xi) r_c^2}{\fwdvar{T-t}{T-s}}\right)} \eqsp,
\end{align*}
which is equivalent to 
\begin{align}
    \label{eq:mixtime:equality2terms}
    \frac{1}{2} = \frac{(1-\multlyap{s}{t})\xi}{2(1 + \xi)\minconstnor{s}{t}\exp\left(-\frac{(1+\xi) r_c^2}{\fwdvar{T-t}{T-s}}\right) + 4 (1 - \multlyap{s}{t})} \eqsp.
\end{align}
While it is hard to find an analytical solution to this equation, note that there is always at least one $\xi \in \rsetpos$ that solves such equation. 
Note furthermore that
\begin{align*}
    \frac{\xi}{4} \geq \frac{(1-\multlyap{s}{t})\xi}{2(1 + \xi)\minconstnor{s}{t}\exp\left(-\frac{(1+\xi) r_c^2}{\fwdvar{T-t}{T-s}}\right) + 4 (1 - \multlyap{s}{t})} \eqsp,
\end{align*}
which implies that if $\xi_{c}$ is the solution to \eqref{eq:mixtime:equality2terms}, then one has $\xi_{c} \leq 2$.
Therefore, by choosing $\xi_{c}$ we have that
\begin{align}
    1 - \pmixtime[s][t](\xi_c)\geq \frac{1}{2}\minconstnor{s}{t} \exp\left(-\frac{3r_c^2}{\fwdvar{T-t}{T-s}}\right) \eqsp.
\end{align}
In order to obtain an explicit bound in terms of the variance of the forward process, we upper bound
\begin{align*}
    r_c^2 &= \frac{2 \biaslyap{s}{t}[2]}{1 - \multlyap{s}{t}[2]} = 2\frac{\ctescoreoffset{0} + \xdim}{\ctescorenorm{0}} + \xdim (1 + 2\ctescorenorm{0}\fwdvar{0}{T-t})\log\left(\frac{1 + 2\ctescorenorm{0}\fwdvar{0}{T-s}}{1 + 2\ctescorenorm{0}\fwdvar{0}{T-t}}\right)\frac{1 + 2\ctescorenorm{0}\fwdvar{0}{T-s}}{2\ctescorenorm{0}\fwdvar{T-s}{T-t}}\\
    &\leq 2\frac{\ctescoreoffset{0} + \xdim}{\ctescorenorm{0}} + \xdim (1 + 2\ctescorenorm{0}\fwdvar{0}{T-t})\frac{1 + 2\ctescorenorm{0}\fwdvar{0}{T-s}}{2\ctescorenorm{0}\fwdvar{T-s}{T-t}}\frac{2\ctescorenorm{0}(\fwdvar{0}{T-s} - \fwdvar{0}{T-t})}{\minimum{1 + 2\ctescorenorm{0}\fwdvar{0}{T-s}}{1 + 2\ctescorenorm{0}\fwdvar{0}{T-t}}}\\
    &=2\frac{\ctescoreoffset{0} + \xdim}{\ctescorenorm{0}} + \xdim (1 + 2\ctescorenorm{0}\fwdvar{0}{T-t})\eqsp.
\end{align*}
Therefore, we have that
\begin{align}
    1 - \pmixtime[s][t](\xi_c)\geq \frac{1}{2}\minconstnor{s}{t} \exp\left(-\frac{6\ctescoreoffset{0} + 9\xdim}{\ctescorenorm{0}\fwdvar{T-t}{T-s}}\right) \exp\left(-6\xdim\frac{\fwdvar{0}{T-t}}{\fwdvar{T-t}{T-s}}\right)\eqsp.
\end{align}


Note that, if $\xi \in (0,2]$, then
\begin{align*}
    \frac{(1-\multlyap{s}{t}[2])\xi}{2(1 + \xi)\minconstnor{s}{t} + 4 (1 - \multlyap{s}{t}[2])\exp\left(\frac{(1+\xi) r_c^2}{\fwdvar{T-t}{T-s}}\right)}
    &\leq
    \frac{(1-\multlyap{s}{t}[2])\xi}{ 4 (1 - \multlyap{s}{t})\exp\left(\frac{(1+\xi) r_c^2}{\fwdvar{T-t}{T-s}}\right)}
    \\
    & \leq
    \frac{\xi}{ 4} \exp\left(- \frac{(1+\xi) r_c^2}{\fwdvar{T-t}{T-s}}\right)
    \leq \frac{1}{ 2} \exp\left(- \frac{(1+\xi) r_c^2}{\fwdvar{T-t}{T-s}}\right) \eqsp.
\end{align*}
Hence, for $\xi \in (0,2]$  the second term of \eqref{eq=1_alphabarre} is the minimum. Hence, \begin{align}\label{eq:lower-gap-xi}
1-\pmixtime[s][t](\xi)
=
\minconstnor{s}{t}\,
\frac{(1-\multlyap{s}{t}[2])\xi}
{2(1+\xi)\minconstnor{s}{t}+4(1-\multlyap{s}{t}[2])\exp \left\{ \frac{(1+\xi) r_c^2}{\fwdvar{T-t}{T-s}}\right\}} \eqsp.
\end{align}
We have
\begin{align*}
    &2(1+\xi)\minconstnor{s}{t}+4(1-\multlyap{s}{t}[2]) \exp \left\{ \frac{(1+\xi) r_c^2}{\fwdvar{T-t}{T-s}}\right\}
    \\
    &\leq
    \left(2(1+\xi)\minconstnor{s}{t}+4(1-\multlyap{s}{t}[2])\right)\exp \left\{ \frac{(1+\xi) r_c^2}{\fwdvar{T-t}{T-s}}\right\}
    \eqsp,
\end{align*}
and therefore
\begin{align*} 
    1-\pmixtime[s][t](\xi)
    \ge
    \minconstnor{s}{t}
    \frac{(1-\multlyap{s}{t}[2])\xi}{2(1+\xi)\minconstnor{s}{t}+4(1-\multlyap{s}{t}[2])}
    \exp \left\{-\frac{(1+\xi) r_c^2}{\fwdvar{T-t}{T-s}}\right\} \eqsp,
    \qquad \text{ for }\xi\in(0,2]\eqsp.
\end{align*}
In particular, choosing $\xi=2$ yields 
\begin{align*}
    1-\pmixtime
    \ge
    \frac{(1-\multlyap{s}{t}[2]) \minconstnor{s}{t}}{3\minconstnor{s}{t}+2(1-\multlyap{s}{t}[2])} 
    \exp \left\{-\frac{3 r_c^2}{\fwdvar{T-t}{T-s}}\right\} \eqsp.
\end{align*}

\section{Gaussian framework: explicit backward kernel and contraction}
\label{app:gaussian_framework}

This section provides a closed-form analysis of the backward dynamics in the Gaussian setting. When the data distribution is Gaussian, the forward marginals remain Gaussian and the score admits an explicit linear form. This makes the backward SDE linear with time-dependent coefficients, yielding an explicit representation of both the backward flow and its transition probabilities.

\subsection{Closed-form score and backward transition}
\label{subapp:closed_form_gaussian}

\begin{lemma}\label[lemma]{lem:exactscore}
Consider the forward diffusion \eqref{eq:forward_SDE} and assume
$\pidata = \gaussiand{\mu}{\Sigma}$, with $\mu\in\rset^\xdim$ and
$\Sigma\in\rset^{\xdim\times\xdim}$ symmetric positive definite. Then, for every $t\in[0,T]$, using the notations of \cref{lem:forward_process_law}, the score function is given by
\begin{align} \label{eq:score_function_gaussian}
\score[t][\x]
= -\Sigma_t^{-1}\big(\x-\fwdmean{0}{t}\mu\big)\eqsp,
\end{align}
with
\begin{align*} \Sigma_t \eqdef \fwdmean{0}{t}^2 \Sigma + \fwdvar{0}{t} \Id_d \eqsp. \end{align*}
Moreover, 
\begin{align}\label{eq:Sigma_derivative}
\partial_t \Sigma_{T-t}
=
2\noisesch{T-t}\left(
\isvp \fwdmean{0}{T-t}^{2}\Sigma
-
\left(1-\isvp \fwdvar{0}{T-t}\right)\Id_\xdim
\right)\eqsp.
\end{align}
\end{lemma}

\begin{proof}
Using \cref{lem:forward_process_law}
\begin{align*} 
    \Xora_t \eqlaw \fwdmean{0}{t} \Xora_0 + \fwdstd{0}{t} G \eqsp,
\end{align*}
with $\Xora_0 \sim \gaussiand{\mu}{\Sigma}$, $\Xora_t$ is Gaussian with $\fwdmean{0}{t}\mu$ and covariance is
$\fwdmean{0}{t}^{2}\Sigma+\fwdvar{0}{t}\Id_\xdim$. The score of a Gaussian $\gaussiand{m}{C}$ is $- C^{-1}(\x-m)$, giving the
claimed expression for \eqref{eq:score_function_gaussian}. It remains to compute the time derivative. We have
\begin{align*}
\partial_t \fwdmean{0}{t}=-\isvp\noisesch{t}\fwdmean{0}{t}
\quad\Longrightarrow\quad
\partial_t \fwdmean{0}{t}^{2}=-2\isvp\noisesch{t}\fwdmean{0}{t}^{2} \eqsp,
\end{align*}
and
\begin{align*}
\partial_t \fwdvar{0}{t}=2\noisesch{t}-2\isvp\noisesch{t}\fwdvar{0}{t}
=2\noisesch{t}(1-\isvp\fwdvar{0}{t}) \eqsp.
\end{align*}
 Differentiating
$\Sigma_t=\fwdmean{0}{t}^{\,2}\Sigma+\fwdvar{0}{t}\Id_\xdim$ and using
$\partial_t  \Sigma_{T-t}=-\left.\frac{\rmd}{\rmd s}\Sigma_s\right|_{s=T-t}$
gives \eqref{eq:Sigma_derivative}.
\end{proof}
\begin{remark} \label[remark]{req:commute}
    For any $s,t \in [0,T]$ the matrices $\Sigma_s$ and $\Sigma_t$ commute.
\end{remark}

\begin{lemma}\label[lemma]{lem:backward_flow}
Assume that $\pidata=\gaussiand{\mu}{\Sigma}$. Then, for all $t\in[0,T]$,
\begin{align*}
\Xola_t
&=
\rme^{-\isvp \int_0^t \bwdnoisesch{s}\,\rmd s}
\Sigma_{T-t}\Sigma_T^{-1}\,\Xola_0
\\
&\quad
+\rme^{-\isvp \int_0^t \bwdnoisesch{s} \rmd s}
\Sigma_{T-t}\int_0^t
\rme^{\isvp \int_0^s \bwdnoisesch{u} \rmd u}
\sqrt{2\bwdnoisesch{s}} 
\Sigma_{T-s}^{-1} \dbrown{s}
\nonumber\\
&\quad
+2 \left( \rme^{-\isvp \int_0^t \bwdnoisesch{s} \rmd s}
\Sigma_{T-t}\int_0^t
\bwdnoisesch{s}
\rme^{\isvp \int_0^s \bwdnoisesch{u} \rmd u}
\Sigma_{T-s}^{-2}
\fwdmean{0}{T-s} \rmd s \right)\mu
\eqsp,
\end{align*}
where $\Sigma_t=\fwdmean{0}{t}^2\Sigma+\fwdvar{0}{t}\Id_\xdim$ (as in~\Cref{lem:exactscore}).
\end{lemma}

\begin{proof}
By~\Cref{lem:exactscore}, for $u\in[0,T]$,
$\score[u][\x]=-\Sigma_u^{-1}(\x-\fwdmean{0}{u}\mu)$.
Plugging $u=T-t$ into \eqref{eq:backward_SDE} gives
\begin{align*}
\rmd \Xola_t
=
\left(
\isvp\bwdnoisesch{t}\Xola_t
-2\bwdnoisesch{t}\Sigma_{T-t}^{-1}\left(\Xola_t-\fwdmean{0}{T-t}\mu\right)
\right)\rmd t
+\sqrt{2\bwdnoisesch{t}} \dbrown{t}\eqsp.
\end{align*}
Define
\begin{align*}
Z_t \eqdef \Sigma_{T-t}^{-1} 
\rme^{\isvp\int_0^t \bwdnoisesch{s} \rmd s} \Xola_t\eqsp.
\end{align*}
Using It\^o's formula and \eqref{eq:Sigma_derivative} we obtain
\begin{align*}
\rmd Z_t
=
2\bwdnoisesch{t} 
\rme^{\isvp\int_0^t \bwdnoisesch{s} \rmd s}
\Sigma_{T-t}^{-2} \fwdmean{0}{T-t} \mu \rmd t
+
\rme^{\isvp\int_0^t \bwdnoisesch{s} \rmd s}
\sqrt{2\bwdnoisesch{t}}
\Sigma_{T-t}^{-1} \dbrown{t}\eqsp.
\end{align*}
Integrating from $0$ to $t$ yields
\begin{align*}
Z_t
=
\Sigma_T^{-1}\Xola_0
+
2\int_0^t
\bwdnoisesch{s}
\rme^{\isvp\int_0^s \bwdnoisesch{u}\,\rmd u}
\Sigma_{T-s}^{-2}
\fwdmean{0}{T-s} \mu\rmd s
+
\int_0^t
\rme^{\isvp\int_0^s \bwdnoisesch{u} \rmd u} 
\sqrt{2\bwdnoisesch{s}}
\Sigma_{T-s}^{-1} \dbrown{s}\eqsp,
\end{align*}
and multiplying by $\rme^{-\isvp\int_0^t \bwdnoisesch{s}\,\rmd s}\Sigma_{T-t}$
gives the desired expression for $\Xola_t$.
\end{proof}

\begin{corollary}
\label[corollary]{cor:gaussian_forgetting}
Assume that $\pidata=\gaussiand{\mu}{\Sigma}$ and let $0\le s<t\le T$.
Then, the conditional distribution of $\Xola_t$ given $\Xola_s=\x$ is Gaussian with
\begin{align*}
\CPE{}{\Xola_s=\x}{\Xola_t}
&=
\rme^{-\isvp \int_s^t \bwdnoisesch{u}\,\rmd u} 
\Sigma_{T-t}\Sigma_{T-s}^{-1} x \\
&\quad
+2\left(
\rme^{-\isvp \int_s^t \bwdnoisesch{u} \rmd u}
\Sigma_{T-t}\int_s^t
\bwdnoisesch{r}
\rme^{\isvp \int_s^r \bwdnoisesch{u} \rmd u}
\Sigma_{T-r}^{-2}
\fwdmean{0}{T-r} \rmd r
\right)\mu
\eqsp,
\\[2mm]
\CPV{}{\Xola_s=\x}{\Xola_t}
&=
\Sigma_{T-t}
-
\rme^{-2\isvp \int_s^t \bwdnoisesch{u} \rmd u}
\Sigma_{T-t}\Sigma_{T-s}^{-1}\Sigma_{T-t}
\eqsp,
\end{align*}
in particular the conditional covariance does not depend on $\x$.
\end{corollary}

\begin{proof}
Starting from~\Cref{lem:backward_flow} (applied at times $t$ and $s$) and using
the identity
\begin{align*}
    &\Sigma_T^{-1}\Xola_0
    +
    2\int_0^{s}
    \bwdnoisesch{r}
    \rme^{\isvp \int_0^{r} \bwdnoisesch{u}\,\rmd u}
    \Sigma_{T-r}^{-2}
    \fwdmean{0}{T-r} \mu\rmd r
    +
    \int_0^{s}
    \rme^{\isvp \int_0^{r} \bwdnoisesch{u}\,\rmd u}
    \sqrt{2\bwdnoisesch{r}}
    \Sigma_{T-r}^{-1}\,\dbrown{r}
    \\
    &=
    \rme^{\isvp \int_0^{s} \bwdnoisesch{u}\rmd u}
    \Sigma_{T-s}^{-1}\Xola_s
    \eqsp,
\end{align*}
one obtains a decomposition of $\Xola_t$ given $\Xola_s$ into: a linear term in $\Xola_s$,
a stochastic integral over $[s,t]$, and a deterministic $\mu$-shift over $[s,t]$. Taking the expectation and using that the stochastic part has mean $0$ yields the conditional mean. The conditional covariance is given by It\^o's isometry and the same calculation as in \cref{lem:backward_flow} (now on
$[s,t]$ instead of $[0,t]$).
\end{proof}

\subsection{Explicit contraction rates for the Euclidean norm}
\label{subapp_rate_gaussian}

We exploit this structure of the transitions derived previously to obtain explicit contraction bounds: first in Euclidean operator norm (through $\|A_{s:t}\|$), and, as a consequence, in $\mathcal W_2$ distance.

\begin{lemma} \label[lemma]{lem:gaussian_contraction_euclidean}
Following \cref{cor:gaussian_forgetting}, in the Gaussian case, the conditional mean satisfies for any $s < t$
\begin{align*}
\CPE{}{\Xola_s =\x}{\Xola_t}
=
A_{s:t} \x + b_{s:t} \eqsp,
\qquad
A_{s:t}
\eqdef
\rme^{-\isvp \int_{s}^{t} \bwdnoisesch{u} \rmd u} 
\Sigma_{T-t} \Sigma_{T-s}^{-1} \eqsp,
\end{align*}
for some deterministic offset $b_{s:t}$ (explicit in \cref{cor:gaussian_forgetting}). In particular,
the dependence on the initial condition over one step is entirely carried by the matrix $A_{s:t}$, which is a strict contraction $\normEc{A_{s:t}} < 1$ for the Euclidean (operator) norm in the VE case and whenever $\isvp^{-1} \ge \lambda_{\max}(\Sigma)$ in the VP case.
\end{lemma}

\begin{proof}
Recall from \cref{lem:exactscore} that
\begin{align*}
\Sigma_t=\fwdmean{0}{t}^2\Sigma+\fwdvar{0}{t}\Id_\xdim.
\end{align*}
Let $\lambda_1,\dots,\lambda_\xdim$ be the eigenvalues of $\Sigma$. Then, 
\begin{align*}
\normEc{A_{s:t}} 
&=
\rme^{-\isvp \int_{s}^{t} \bwdnoisesch{u}\,\rmd u}
\max_{1\le i\le \xdim}
\frac{\fwdmean{0}{T-t}^{2}\lambda_i+\fwdvar{0}{T-t}}
{\fwdmean{0}{T-s}^{2}\lambda_i+\fwdvar{0}{T-s}}
\eqsp.
\end{align*}
\begin{itemize}
\item \textbf{VE case} ($\isvp=0$). By definition, $\fwdmean{0}{t}\equiv 1$ and $\fwdvar{0}{t}$ is strictly increasing in $t$. Since $T-t<T-s$, we have $\fwdvar{0}{T-t}<\fwdvar{0}{T-s}$ and therefore for every $i$,
\begin{align*}
\frac{\lambda_i+\fwdvar{0}{T-t}}{\lambda_i+\fwdvar{0}{T-s}}<1,
\end{align*}
which implies $\normEc{A_{s:t}}<1$. Hence the one--step conditional mean is a strict contraction in Euclidean norm.
\item \textbf{VP case} ($\isvp>0$). The prefactor
$\rme^{-\isvp \int_{t_k}^{t_{k+1}} \bwdnoisesch{u}\,\rmd u}<1$ is always contractive. Moreover, in the VP case, for any $s \in [0,T]$ $\fwdvar{0}{s} = \isvp^{-1} \left( 1 - \fwdmean{0}{s}^{2} \right)$, hence, for any $1\le i\le \xdim$,
\begin{align*}
\frac{\fwdmean{0}{T-t}^{2}\lambda_i+\fwdvar{0}{T-t}}
{\fwdmean{0}{T-s}^{2}\lambda_i+\fwdvar{0}{T-s}} = \frac{\isvp^{-1} + \fwdmean{0}{T-t}^{2} \left( \lambda_i - \isvp^{-1} \right) }{\isvp^{-1} + \fwdmean{0}{T-s}^{2} \left( \lambda_i - \isvp^{-1} \right)}
\eqsp.
\end{align*}
In particular, using that $\fwdmean{0}{t}$ is strictly decreasing in $t$,
\begin{itemize}
    \item if $\lambda_i \le \isvp^{-1}$, then
    \begin{align*}
    \frac{\fwdmean{0}{T-t}^{2}\lambda_i+\fwdvar{0}{T-t}}
    {\fwdmean{0}{T-s}^{2}\lambda_i+\fwdvar{0}{T-s}}
    < 1,
    \end{align*}
    so the covariance ratio is contractive along this direction. Moreover, if $\lambda_{\max} \le \isvp^{-1}$, then for all $i$, $\lambda_i \le \isvp^{-1}$.
    \item if $\lambda_i \ge \isvp^{-1}$, then the ratio is maybe be larger and and this eigendirection may expand if not compensated by the prefactor term.
\end{itemize}
It follows that the one--step conditional mean in the VP case is also a strict contraction in Euclidean norm when $\lambda_{\max} <\isvp^{-1}$.
\end{itemize}
\end{proof}

\begin{lemma}\label[lemma]{lem:gaussian_W2_dirac}
Assume that $\pidata=\gaussiand{\mu}{\Sigma}$ and let $0\le s<t\le T$. Then, for any $\x,\x'\in\rset^\xdim$,
\begin{align*}
\wasserstein[2][]{\delta_\x \bwdker{s}[t]}{\delta_{\x'} \bwdker{s}[t]}
\le
\normEc{\rme^{-\isvp \int_{s}^{t} \bwdnoisesch{u} \rmd u} 
\Sigma_{T-t}\Sigma_{T-s}^{-1}} \normEc{\x-\x'}\eqsp.
\end{align*}
\end{lemma}

\begin{proof}
In the Gaussian case, for any $0\le s<t\le T$ the conditional law of $\Xola_t$ given $\Xola_s=\x$ is Gaussian with covariance independent of $\x$ (see \cref{cor:gaussian_forgetting}) and conditional mean
\begin{align*}
\CPE{}{\Xola_s =\x}{\Xola_t}
=
A_{s:t} \x + b_{s:t} \eqsp,
\qquad
A_{s:t}
\eqdef
\rme^{-\isvp \int_{s}^{t} \bwdnoisesch{u} \rmd u} 
\Sigma_{T-t} \Sigma_{T-s}^{-1} \eqsp.
\end{align*}
As a consequence, using the closed-form formula for Gaussian random variables (Bures metric)
\begin{align*}
\wasserstein[2][]{\delta_\x \bwdker{s}[t]}{\delta_{\x'} \bwdker{s}[t]}
& =
\normEc{
\rme^{-\isvp \int_s^t \bwdnoisesch{u}\,\rmd u}
\Sigma_{T-t}\Sigma_{T-s}^{-1}(\x-\x')
} \\
&\le
\normEc{
\rme^{-\isvp \int_s^t \bwdnoisesch{u}\,\rmd u}
\Sigma_{T-t}\Sigma_{T-s}^{-1}
} \normEc{\x-\x'}\eqsp.
\end{align*}
\end{proof}

\section{Proof of the main stability bound (\cref{thm:global-weak-error})}
\label{app:general-thm:final-bound}

A proof sketch is provided in the main text. This appendix completes the proof. We collect the intermediate estimates used in the proof sketch from the main text.



\subsection{Initialization error: mixing properties of the forward process} \label{subapp:initialization_error}

We bound the discrepancy between the terminal forward law $p_T$ and the reference Gaussian measure $\pi_\infty$ used to initialize the backward chain. 
In the VE case, the KL divergence to the Gaussian reference is controlled via a convexity argument (\cref{lem:mixing_time}). In the VP case, KL contracts exponentially fast along the Ornstein--Uhlenbeck flow (\cref{lem:mixing_time}). These KL bounds are then converted into weighted total variation bounds through Pinsker/Hellinger estimates and moment controls.

\begin{lemma} \label[lemma]{lem:mixing_time}
    Let $p_T$ denote the law of $\Xora_T$ and suppose that \cref{assump:p0:score} holds. 
    \begin{itemize}
        \item If $\isvp = 0$ (Variance Exploding) and $ \refmeas = 
        \gaussiand{0}{(2 \int_0^T \noisesch{s} \, \rmd s)\Id_\xdim}$, then
        \begin{align} \label{eq_mixing_VE}
                \kl{p_T}{\refmeas}
        \leq \frac{\mathbb E \left[ \normEc{\Xora_0}^2 \right] }{4 \int_0^T \noisesch{s}  \rmd s}
        \eqsp.
        \end{align}
        \item If $\isvp > 0$ (Variance Preserving) and $\refmeas = 
        \gaussiand{0}{\isvp^{-1} \Id_\xdim}$, then
        \begin{align} \label{eq_mixing_VP}
        \kl{p_T}{\refmeas}
        \leq \kl{\pi_{\rm data} }{\refmeas} \rme^{-2 \isvp \int_0^T \noisesch{s}  \rmd s}
        \eqsp.
        \end{align}
    \end{itemize}
\end{lemma}

\begin{proof}
    In the VE case, recall that, for $\x \in \rset^d$,
    \begin{align*} p_T(\x) = \int_{\rset^d} \fwdmarg{0}(y) (2 \pi \fwdvar{0}{T} )^{-d/2} \exp \left\{- \frac{\normEc{\x-y}^2}{2 \fwdvar{0}{T}} \right\} \rmd y \eqsp.
    \end{align*}
    Using convexity of the Kullback-Leibler in its first argument yields
    \begin{align*}
     \kl{p_T}{\refmeas}
        \leq \int_{\rset^d} \fwdmarg{0} (y) \kl{\gaussiand{y}{(2 \int_0^T \noisesch{s} \, \rmd s)\Id_\xdim}}{\gaussiand{0}{(2 \int_0^T \noisesch{s}  \rmd s)\Id_\xdim}} \rmd y \eqsp.
    \end{align*}
    Moreover,
    \begin{align*}
    \kl{\gaussiand{y}{(2 \int_0^T \noisesch{s} \, \rmd s)\Id_\xdim}}{\gaussiand{0}{(2 \int_0^T \noisesch{s} \, \rmd s)\Id_\xdim}} = \frac{\left\| y \right\|^2}{4 \int_0^T \noisesch{s}  \rmd s} \eqsp,
    \end{align*}
    which proves \eqref{eq_mixing_VE}. In the VP case, note that it follows from \cref{assump:p0:score} together with \cref{lem:pdata-sub-gaussian} that
    \begin{align*} 
    \kl{\pi_{\rm data}}{\pi_{\infty}} \leq \log \left\| \fwdmarg{0} \right\|_{\infty} + \frac{1}{2 \isvp^{-1}} \mathbb E \left[ \normEc{\Xora_0}^2 \right] + \frac{d}{2} \log \left( 2 \pi \isvp^{-1} \right) < \infty.
    \end{align*}
    It follows from Lemma B.1 of \citet{strasman2025an}, after a suitable time change and an adjustment of the invariant measure, that \eqref{eq_mixing_VP} holds.
\end{proof}

\begin{proposition} \label[proposition]{prop:mixing_error_rho_b}
Suppose that \cref{assump:p0:score} holds and let $\Xora_{\infty} \sim \pi_\infty$. Then, the mixing time for SDE \eqref{eq:forward_SDE} in the weighted total variation distance is upper bounded by
\begin{align*} \rho_b(p_T, \refmeas) \leq \left(\frac{1}{\sqrt 2} + b\sqrt{2 \left(\mathbb{E}\left[V^2(\Xora_T) \right]+\mathbb{E} \left[V^2(\Xora_{\infty}) \right]\right) }\right) \frac{ \normEc{\Xora_0}_{L_2}  }{ 2 \sqrt{\int_0^T \noisesch{s}  \rmd s}} \eqsp,\end{align*}
in the VE case ($\alpha = 0$), and by
\begin{align*} \rho_b(p_T, \refmeas) \leq \left(\frac{1}{\sqrt 2} + b\sqrt{2 \left(\mathbb{E}\left[V^2(\Xora_T) \right]+\mathbb{E} \left[V^2(\Xora_{\infty}) \right]\right) }\right)\sqrt{\kl{\pi_{\rm data} }{\refmeas}} \rme^{- \isvp \int_0^T \noisesch{s}  \rmd s} \eqsp,\end{align*}
in the VP case ($ \alpha > 0$).
\end{proposition}

\begin{proof}
Using that the both distributions are absolutely continuous with respect to the Lebesgue measure:
\begin{align*}
    \rho_b (p_T, \refmeas ) & = \int (1 + b V(\x)) | p_T(\x) -\fwdmarg{\infty}(\x) | \rmd \x \\
    & = \normTV{p_T -\fwdmarg{\infty}} + b \int V(\x) | p_T(\x) -\fwdmarg{\infty}(\x) | \rmd \x \eqsp.
\end{align*}
The left hand-side is controlled using Pinsker's inequality \eqref{eq:pinsker},
\begin{align*} \normTV{p_T -\fwdmarg{\infty}} \leq \sqrt{ \frac12 \kl{\fwdmarg{T}}{\refmeas}} \end{align*}
It remains to control the right-hand side. Write
\begin{align*}
    |p_T-p_\infty|
    =
    |\sqrt{p_T}-\sqrt{p_\infty}|\,(\sqrt{p_T}+\sqrt{p_\infty}).
\end{align*}
By Cauchy--Schwarz,
\begin{align*}
    &\int V(\x)|p_T(\x)-p_\infty(\x)| \rmd \x
    \\
    &\le
    \left(\int V^2(\x) \left( \sqrt{p_T(\x)}+\sqrt{p_\infty(\x)} \right)^2 \rmd \x \right)^{1/2}
    \left(\int \left( \sqrt{p_T(\x)}-\sqrt{p_\infty(\x)} \right)^2 \rmd \x\right)^{1/2}.
\end{align*}
On one hand, using $(\sqrt{a}+\sqrt{b})^2\le 2(a+b)$,
\begin{align*}
\int V^2(\x)(\sqrt{p_T(\x)}+\sqrt{p_\infty(\x)})^2 \rmd \x
& \leq
2\int V^2(\x)(p_T(\x)+p_\infty(\x)) \rmd \x \\
& =
2\left(\mathbb{E}\left[V^2(\Xora_T) \right]+\mathbb{E} \left[V^2(\Xora_{\infty}) \right]\right) \eqsp,
\end{align*}
with $\Xora_{\infty} \sim \pi_\infty$. On the other hand, using inequality \eqref{eq:hellinger_kl}
\begin{align*}
\int \left( \sqrt{p_T(\x)}-\sqrt{p_\infty(\x)} \right)^2 \rmd x = 2 \hell(\mu_1,\mu_2)^2 \leq \kl{p_T}{\refmeas} \eqsp.
\end{align*}
It follows that, 
\begin{align*}
\rho_b(p_T,\pi_\infty)
& \le \left(\frac{1}{\sqrt 2} + b\sqrt{2 \left(\mathbb{E}\left[V^2(\Xora_T) \right]+\mathbb{E} \left[V^2(\Xora_{\infty}) \right]\right) }\right)\sqrt{\kl{\fwdmarg{T}}{\refmeas}} \eqsp,
\end{align*}
since $\mathbb{E} \left[V^2(\Xora_T)\right] < \infty$ thanks to \cref{lem:pdata-sub-gaussian}
and recall that $\Xora_\infty \sim \pi_{\infty}$ is a Gaussian measure so that $\mathbb{E} \left[V^2(\Xora_\infty)\right]<\infty$. Finally, conclusion follows from~\cref{lem:mixing_time}.
\end{proof}

\subsection{One-step discretization and approximation error for the backward kernel} \label{subapp:discrandapprox}

We quantify the one-step discrepancy between the exact backward transition $\bwdker{s}[t]$ and its learned/discretized counterpart $\approxbwdker{s}[t]$. 
The bound is expressed in a weighted total variation distance and splits into two components: 
a discretization/freezing error (due to keeping the drift frozen over $[s,t]$) and a score approximation error.  The proof proceeds by controlling a local KL divergence via a Girsanov argument and converting it to a weighted bound through Hellinger/Pinsker inequalities and Lyapunov moment estimates. In this section we use the shorthand notations
\begin{align*}
\maxmultlyap{s}{t}[\ell]
\eqdef \sup_{u\in[s,t]} \multlyap{s}{u}[\ell],
\qquad
\maxbiaslyap{s}{t}[\ell]
\eqdef
\sup_{u\in[s,t]} \biaslyap{s}{u}[\ell],
\qquad
\maxctegrowthHess{t,s}
\eqdef
\sup_{u\in[s,t]} \ctegrowthHess{T-u}.
\end{align*}
for the constants of \cref{prop:backward_drift_lyapunov} and \cref{cor:bound_hessian},

\begin{proposition} \label[proposition]{prop:one_step_discr_error}
Let $0 \leq s < t \leq T$ such that $s < t$ and $\x \in \rset^\xdim$. Suppose that \cref{assump:p0}  hold. Set $\Delta \eqdef \int_{s}^{t}\bwdnoisesch{u} \rmd u > 0$ and
\begin{align*}
\errscore{s}{\x} \eqdef 
    \score[T-s][\x] - \scorenet[T-s][\x]
 \eqsp.
\end{align*}
Then,
\begin{align*}
\bmetric{\delta_{\x}\bwdker{s}[t]}{\delta_{\x}\approxbwdker{s}[t]}
&\le
\sqrt{\Delta^2\big(\cstediscrgamma + \cstediscrgammax \normEc{\x}^{4p+4}\big)
    + 3\Delta \normEc{\errscore{s}{\x}}^2 } H_s(\x)
\\
&\le
\Delta \sqrt{\cstediscrgamma + \cstediscrgammax \normEc{\x}^{4p+4}} H_{s:t}(\x)
+ \sqrt{3\Delta} \normEc{\errscore{s}{\x}} H_{s:t}(\x)
\eqsp,
\end{align*}
where
\begin{align*}
    \cstediscrgamma \eqdef
    \tfrac34 \Delta \isvp^4 \left(\maxmultlyap{s}{t}[2] + \maxbiaslyap{s}{t}[2] \right)
    +3\left(\isvp^2\xdim + 8 \maxctegrowthHess{t,s}^2\left(1+\maxbiaslyap{s}{t}[4p+4]\right)\right) \eqsp,
\end{align*}

\begin{align*}
    \cstediscrgammax \eqdef
    24\maxctegrowthHess{t,s}^2 \maxmultlyap{s}{t}[4p+4]
    + \tfrac34 \Delta \isvp^4 \maxmultlyap{s}{t}[2] \eqsp,
\end{align*}
and 
\begin{align*}
H_{s:t}(\x) \eqdef \left( \sqrt{\bwdker{s}[t]( 1 + b \lyapunov{2}(\x))^2} + \sqrt{\approxbwdker{s}[t]( 1 + b \lyapunov{2}(\x))^2} \right) \eqsp.
\end{align*}
\end{proposition}

\begin{proof}
Using~\Cref{lem:weighted_TV_bound_Hellinger} and inequality \eqref{eq:hellinger_kl} 
relating the Hellinger distance and KL divergence, we obtain for $\x\in\rset^\xdim$,
\begin{align*}
    &\bmetric{\delta_\x \bwdker{s}[t]}{ \delta_\x \approxbwdker{s}[t]}
    \\
    &\le
    \sqrt{2} \hellinger{\delta_\x \bwdker{s}[t]}{ \delta_\x \approxbwdker{s}[t]}
    \\
    &\qquad\times
    \left(
    \CPE{}{\Xola_{s}=\x}{
        \left(1+b\lyapunov{2}[\Xola_{t}]\right)^2
    }^{1/2}
    +
    \CPE{}{\bar \X^\theta_{s}=\x}{
        \left(1+b\lyapunov{2}[\Xbar_{t}^{\theta}]\right)^2
    }^{1/2}
    \right)
    \\
    &\le
    \sqrt{ \kl{\delta_\x \bwdker{s}[t]}{ \delta_\x \approxbwdker{s}[t]}}
    \\
    &\qquad\times
    \left(
    \CPE{}{\Xola_{s}=\x}{
        \left(1+b\lyapunov{2}[\Xola_{t}]\right)^2
    }^{1/2}
    +
    \CPE{}{\bar \X^\theta_{s}=\x}{
        \left(1+b\lyapunov{2}[\Xbar_{t}^{\theta}]\right)^2
    }^{1/2}
    \right) \eqsp.
\end{align*}

Since $(1+b\lyapunov{2})^2\leq 2(1+b^2\lyapunov{2}^2)$, applying \cref{prop:backward_drift_lyapunov} for $\lyapunov{4}$, we have that
$\CPE{}{\Xola_{s}=\x}{\lyapunov{2}[\Xola_{t}]^2}<\infty$. Moreover, since $\normEc{\scorenet[s][\x]}$ is deterministic then $\CPE{}{\bar\X^{\theta}_{s}=\x}{\lyapunov{2}[\bar\X^{\theta}_{t}]^2} < \infty$.

\smallskip
\noindent\textbf{Step 1: KL bound via $L_2$ drift bound.}
Consider the continuous-time interpolation $(\bar\X_u^\theta)_{u\in[s,t]}$ of the Markov chain
\eqref{eq:time_changed_euler}, defined for $u \in [s, t]$, as 
\begin{align} \label{eq:euler_approx_score_continuous}
\bar \X_u^\theta
= \x + 
\int_{s}^{u} \bwdnoisesch{\ell}\rmd \ell \left( \isvp  \x
+ 2\, \scorenet[T-s][\x][\theta] \right)
+ \int_{s}^{u}\sqrt{2 \bwdnoisesch{\ell}} \dbrown{\ell} \eqsp,
\end{align}
Applying \cref{lem:boundgirsanovproof} and \cref{cor:boundgirsanovproof}, together with the data processing inequality \citep[Lemma~1.6]{Nutz2021EntropicOT}, the Kullback-Leibler divergence between the one--step kernels can be bounded in terms of the $L^2$ drift mismatch between \eqref{eq:backward_SDE} and \eqref{eq:euler_approx_score_continuous}:
\begin{align} \label{eq:girsanov_local}
    \begin{split}
        &\kl{
        \delta_\x \bwdker{s}[t]}{
        \delta_\x \approxbwdker{s}[t]
        }
        \\
        &\le
        \frac{1}{4}\int_{s}^{t}
        \bwdnoisesch{u}
        \CPE{}{\Xola_{s}= \x}{
            \normEc{\isvp \left( \Xola_u - \x \right) + 2 \left( \score[T-u][\Xola_u] - \scorenet[T-s][\x] \right) }^2
        }
        \rmd u \eqsp.
    \end{split}
\end{align}
\smallskip
\noindent\textbf{Step 2: error decomposition.}
The score mismatch decomposes as the sum of a freezing error and a score approximation error that writes as
\begin{align*}
    & \score[T-u][\Xola_u] - \scorenet[T-s][\x]
    \\
    & = 
    \bigg(
        \score[T-u][\Xola_u] -  \score[T-s][\x]
    \bigg)
    +
    \bigg(
        \score[T-s][\x]- \scorenet[T-s][\x]
    \bigg)
    \\
    & = \Y_u^{\x} - \Y_{s}^{\x}  + \errscore{s}{\x} \eqsp,
\end{align*}
where we used the notation of~\cref{score differential form} \ie, $\Y_u^\x \eqdef \score[T-u][\Xola_{u}]$ for $u \in [s, t]$.
From~\cref{score differential form}, with $\Z_u^\x=\jscore[T-u][\Xola_u^\x]$ for $u \in [s, t]$, we have that 
\begin{align*}
    \Y_t^\x - \Y_{s}^\x = - \isvp \int_{s}^{t} \bwdnoisesch{u} \Y_u^{\x} \rmd u + \int_{s}^{t}\sqrt{2 \bwdnoisesch{u} } \Z_u^\x \dbrown{u} \eqsp.
\end{align*}
By definition of the backward process we have
\begin{align*}
    \Xola_t^\x - \x = \int_{s}^{t} \left( \isvp  \bwdnoisesch{u}   \Xola_u^\x  + 2 \bwdnoisesch{u} \Y_u^{\x} \right) \rmd u +  \int_{s}^{t} \sqrt{2 \bwdnoisesch{u}} \dbrown{u}  \eqsp.
\end{align*}
Combining these cancels the $\Y$-integrals and yields the following identity
\begin{align*}
\isvp(\Xola_t^\x-\x) + 2(\Y_t^\x-\Y_{s}^\x)
=
\isvp^2\int_{s}^t \bwdnoisesch{u}\,\Xola_u^\x\,\rmd u
+\int_{s}^t \sqrt{2\bwdnoisesch{u}}\big(\isvp \Id_\xdim + 2\Z_u^\x\big) \dbrown{u} \eqsp.
\end{align*}

\smallskip
\noindent\textbf{Step 3: bounding the conditional error terms.} Plugging the above decomposition in \eqref{eq:girsanov_local} and using the triangle inequality yields
\begin{align*}
    &\CPE{}{\Xola_{s}= \x}{
        \normEc{\isvp \left( \Xola_t - \x \right) + 2 \left( \score[T-t][\Xola_t] - \scorenet[T-s][\x] \right)}^2
        }
    \\
    &\leq 3
    \CPE{}{\Xola_{s}  = \x }{
        \normEc{\isvp^2 \int_{s}^{t} \bwdnoisesch{u} \Xola_u \rmd u}^2
    }
    \\
    &\qquad
    + 3 \CPE{}{\Xola_s  = \x}{\normEc{\int_{s}^{t}\sqrt{2 \bwdnoisesch{u}} \left( \isvp \Id_\xdim + 2 \Z_{u}^{\x} \right) \dbrown{u}}^2}+ 12 \normEc{\errscore{s}{\x}}^2 \eqsp.
\end{align*}
Let $\nu(\rmd u) \eqdef (\bwdnoisesch{u}/\Delta) \rmd u$ be a probability measure over $[s,t]$. Then, using Jensen's inequality
\begin{align*}
    \CPE{}{\Xola_{s} =\x }{
        \normEc{\int_{s}^{t} \bwdnoisesch{u} \Xola_u \rmd u}^2}
        &  = \Delta^2
        \CPE{}{\Xola_{s}= \x }{
        \normEc{ \int_{s}^{t}  \Xola_u \nu (\rmd u) }^2} \\
        & \le \Delta^2
         \CPE{}{\Xola_{s} = \x }{
         \int_{s}^{t}   \normEc{\Xola_u}^2 \nu (\rmd u) }
         \\
         &
          = \Delta
         \CPE{}{\Xola_{s} = \x }{
         \int_{s}^{t} \bwdnoisesch{u}  \normEc{\Xola_u}^2 \rmd u } \eqsp.
\end{align*}
By \cref{prop:backward_drift_lyapunov},
\begin{align*}
    \CPE{}{\Xola_s =\x}{\normEc{\Xola_u}^2}
    \le \multlyap{s}{u}[2]\|\x\|^2 + \biaslyap{s}{u}[2]
    \eqsp.
\end{align*}
Hence, 
\begin{align} \label{eq:control_linear_drift}
    \CPE{}{\Xola_{s} =\x }{
        \normEc{\isvp^2 \int_{s}^{t} \bwdnoisesch{u} \Xola_u \rmd u}^2}
        &\leq \isvp^{4}\Delta
         \CPE{}{\Xola_{s}=\x }{
         \int_{s}^{t} \bwdnoisesch{u}  \normEc{\Xola_u}^2 \rmd u } \notag \\ 
         & \leq \isvp^{4}\Delta \left( \normEc{\x}^2 \int_s^t \bwdnoisesch{u} \multlyap{s}{u}[2] \rmd u + \int_s^t \bwdnoisesch{u} \biaslyap{s}{u}[2] \rmd u \right) \notag \\
        &\le \isvp^4 \Delta^2 \left( \maxmultlyap{s}{t}[2] \normEc{\x}^2 + \maxbiaslyap{s}{t}[2] \right) \eqsp.
\end{align}
By Itô's isometry,
\begin{align*}
   &\CPE{}{\Xola_{s}=\x}{ \normEc{\int_{s}^{t}\sqrt{2 \bwdnoisesch{u}} \left( \isvp \Id_\xdim + 2 \Z_u^\x \right) \dbrown{u}}^2}
    = 2 \int_{s}^t \bwdnoisesch{
    u} \CPE{}{\Xola_{s}=\x}{\normFr{ \isvp \Id_\xdim + 2 \Z_u^\x }^2}\rmd u\\
    &\qquad \le 4 \isvp^2 \xdim \Delta + 16 \int_{s}^t \bwdnoisesch{u} \CPE{}{\Xola_s=\x}{\normFr{\Z_u^\x  }^2} \rmd u \eqsp.
\end{align*}
Since $\Z_u^\x \eqdef \jscore[T-u][\Xola_u]$, \cref{cor:bound_hessian} applied at time $T-u$ and point $\Xola_u$ yields
\begin{align*}
\normFr{\Z_u^\x}
=
\normFr{\jscore[T-u][\Xola_u]}
\le
\ctegrowthHess{T-u}\big(1+ \normEc{\Xola_u}^{2p+2}\big)\eqsp.
\end{align*}
Squaring and using $(1+a)^2\le 2(1+a^2)$, we obtain
\begin{align*}
\CPE{}{\Xola_s=\x}{\normFr{\Z_u^\x}^2}
&\le
2\,\ctegrowthHess{T-u}^2\Big(1+\CPE{}{\Xola_s=\x}{\normEc{\Xola_u}^{4p+4}}\Big)\eqsp.
\end{align*}
Therefore,
\begin{align*}
\int_s^t \bwdnoisesch{u} 
\CPE{}{\Xola_s=\x}{\normFr{\Z_u^\x}^2} \rmd u
\le
2\int_s^t \bwdnoisesch{u} \ctegrowthHess{T-u}^2
\Big(1+\CPE{}{\Xola_s=\x}{\normEc{\Xola_u}^{4p+4}}\Big) \rmd u \eqsp.
\end{align*}

Finally, applying \cref{prop:backward_drift_lyapunov} with $\ell=4p+4$ gives, for all $u\in[s,t]$,
\begin{align*}
\CPE{}{\Xola_s=\x}{\normEc{\Xola_u}^{4p+4}}
\le
\multlyap{s}{u}[4p+4]\normEc{\x}^{4p+4}+\biaslyap{s}{u}[4p+4]\eqsp,
\end{align*}
and thus

\begin{align*}
\int_s^t \bwdnoisesch{u}
\CPE{}{\Xola_s=\x}{\normFr{\Z_u^\x}^2} \rmd u
& \le
2\int_s^t \bwdnoisesch{u} \ctegrowthHess{T-u}^2
\Big(
1+\multlyap{s}{u}[(4p+4)] \normEc{\x}^{4p+4}+\biaslyap{s}{u}[(4p+4)]
\Big) \rmd u \\
&\le
2 \Delta \maxctegrowthHess{t,s}^2\Big(1+\maxmultlyap{s}{t}[(4p+4)] \normEc{\x}^{4p+4}+\maxbiaslyap{s}{t}[(4p+4)]\Big) \eqsp.
\end{align*}
Hence,
\begin{align} \label{eq:control_hessian}
    \begin{split}
        &\CPE{}{\Xola_{s}=\x}{ \normFr{\int_{s}^{t}\sqrt{2 \bwdnoisesch{u}} \left( \isvp \Id_\xdim + 2 \Z_u \right) \dbrown{u}}^2}
        \\
        &\leq 4\Delta \left( \isvp^2 \xdim + 8 \maxctegrowthHess{t,s}^2\Big(1+\maxmultlyap{s}{t}[(4p+4)] \normEc{\x}^{4p+4}+\maxbiaslyap{s}{t}[(4p+4)]\Big) \right)
        \eqsp.
    \end{split}
\end{align}
Combining \eqref{eq:control_linear_drift} with \eqref{eq:control_hessian} yields, 
\begin{align*}
    &3 \CPE{}{\Xola_{s}  = \x }{
        \normEc{\isvp^2 \int_{s}^{t} \bwdnoisesch{u} \Xola_u \rmd u}^2
    }
    + 3 \CPE{}{\Xola_s  = \x}{\normEc{\int_{s}^{t}\sqrt{2 \bwdnoisesch{u}} \left( \isvp \Id_\xdim + 2 \Z_{u}^{\x} \right) \dbrown{u}}^2} \\
    &\le
    3 \isvp^4 \Delta^2\Big(
    \maxmultlyap{s}{t}[2]\normEc{\x}^2 + \maxbiaslyap{s}{t}[2]
    \Big)
    + 12 \Delta\Bigg(
    \isvp^2 \xdim
    + 8 \maxctegrowthHess{t,s}^2
    \Big(
    1
    + \maxmultlyap{s}{t}[4p+4]\,\normEc{\x}^{4p+4}
    + \maxbiaslyap{s}{t}[4p+4]
    \Big)
    \Bigg)\eqsp.
\end{align*}
Hence, using that $\normEc{\x}^2 \leq 1 + \normEc{\x}^{4p+4}$,
\begin{align*}
    &3 \CPE{}{\Xola_{s}  = \x }{
        \normEc{\isvp^2 \int_{s}^{t} \bwdnoisesch{u} \Xola_u \rmd u}^2
    }
    + 3 \CPE{}{\Xola_s  = \x}{\normEc{\int_{s}^{t}\sqrt{2 \bwdnoisesch{u}} \left( \isvp \Id_\xdim + 2 \Z_{u}^{\x} \right) \dbrown{u}}^2}
    \\
    &\le
    \Delta \left(\cstediscr + \cstediscrx \normEc{\x}^{4p+4}\right),
\end{align*}
with
\begin{align*}
\cstediscr \eqdef
3 \Delta \isvp^4 \left(\maxmultlyap{s}{t}[2] + \maxbiaslyap{s}{t}[2] \right)
+12\left(\isvp^2\xdim + 8 \maxctegrowthHess{t,s}^2\left(1+\maxbiaslyap{s}{t}[4p+4]\right)\right) \eqsp,
\end{align*}
and
\begin{align*}
\cstediscrx \eqdef
96 \maxctegrowthHess{t,s}^2\,\maxmultlyap{s}{t}[4p+4]
+3 \Delta \isvp^4 \maxmultlyap{s}{t}[2].
\end{align*}
\smallskip
\noindent\textbf{Step 4: combining all the bounds.}
It follows that the one-step discretization error is upper bounded by,
\begin{align*}
    &\kl{
    \delta_\x \bwdker{s}[t]}{
    \delta_\x \approxbwdker{s}[t]
    }
    \\
    & \le
    \frac{1}{4}\int_{s}^{t}
    \bwdnoisesch{u}
    \CPE{}{\Xola_{s}= \x}{
        \normEc{\isvp \left( \Xola_u - \x \right) + 2 \left( \score[T-u][\Xola_u] - \scorenet[T-s][\x] \right) }^2
    }
    \rmd u \\
    &\le \frac{1}{4}\int_s^t \bwdnoisesch{u}\,
    \Big(
    \Delta\big(\cstediscr + \cstediscrx \normEc{\x}^{4p+4}\big)
    + 12\normEc{\errscore{s}{\x}}^2
    \Big)\rmd u \\
    &= \frac{\Delta^2}{4}\big(\cstediscr + \cstediscrx \normEc{\x}^{4p+4}\big)
    + 3\Delta \normEc{\errscore{s}{\x}}^2 \eqsp, 
\end{align*}
absorbing $\tfrac14$ in $\cstediscr$ and $\cstediscrx$ yields the exact constants in the statement of the Lemma.
\end{proof}

\begin{corollary}
\label[corollary]{cor:one_step_stability_pgen}
Let $k\in\{1,\dots,N\}$ and set $\Delta_k \eqdef \int_{t_{k-1}}^{t_k}\bwdnoisesch{u}\,\rmd u$. Assume the hypotheses of \cref{prop:one_step_discr_error} hold, and recall that
$\pgen_k \eqdef \refmeas \approxbwdkercomp[\theta][]{0}{k}$.
Then
\begin{align*}
\bmetric{\pgen_{k-1}  \bwdker{t_{k-1}}[t_k]}{\pgen_{k-1}  \approxbwdker{t_{k-1}}[t_k]} 
&\le
\Delta_k \Cdiscr{k-1} +  \sqrt{\Delta_k} \Cnet{k-1} \left\| \errscore{k-1}{} \right\|_{L_2(\pgen_{k-1})} \eqsp,
\end{align*}
with
\begin{align*}
 \left\| \errscore{k-1}{} \right\|_{L_2(\pgen_{k-1})}  = \mathbb{E} \left[ \left\| \score[T-t_{k-1}][\bar \X_{t_{k-1}}^\theta] - \scorenet[T-t_{k-1}][\bar \X_{t_{k-1}}^\theta] \right\|^2 \right]^{1/2} \eqsp,  
\end{align*}
\begin{align*}
\Cdiscr{k-1} \eqdef \sqrt{\cstediscrgamma[t_{k-1}:t_{k}]+\cstediscrgammax[t_{k-1}:t_{k}] \mathbb{E} \left[ \normEc{\Xbar^{\theta}_{t_{k-1}}}^{4p+4}\right]}
 \sqrt{\mathbb{E} \left[H_k(\Xbar^{\theta}_{t_{k-1}})^2\right]}
\end{align*}
and 
\begin{align*}
\Cnet{k-1} = \sqrt{3 \mathbb{E} \left[H_k(\Xbar^{\theta}_{t_{k-1}})^2\right]} \eqsp
\end{align*}
and where $H_k$ (with $s=t_{k-1}$ and $t=t_k$), $\cstediscrgamma[t_{k-1}:t_{k}]$, and $\cstediscrgammax[t_{k-1}:t_{k}]$ are defined in \cref{prop:one_step_discr_error}.  
\end{corollary}

\begin{proof}
Using \cref{lem:rho_b_convexity} we have that
\begin{align*}
\bmetric{\pgen_{k-1}\bwdker{t_{k-1}}[t_k]}{\pgen_{k-1}\approxbwdker{t_{k-1}}[t_k]}
&\le
\int
\bmetric{\delta_{\x}\bwdker{t_{k-1}}[t_k]}{\delta_{\x}\approxbwdker{t_{k-1}}[t_k]}
 \pgen_{k-1}(\rmd \x) \eqsp.
\end{align*}
Let $\Xbar^{\theta}_{t_{k-1}} \sim \pgen_{k-1}$. Then, by Cauchy--Schwarz
\begin{align*}
\int \normEc{\errscore{t_{k-1}}{\x}} H_k(\x) \pgen_{k-1}(\rmd \x)
&=
\mathbb{E} \left[
\normEc{\errscore{t_{k-1}}{\Xbar^{\theta}_{t_{k-1}}}} H_k(\Xbar^{\theta}_{t_{k-1}})
\right] \\
&\le
\sqrt{\mathbb{E}\left[\normEc{\errscore{t_{k-1}}{\Xbar^{\theta}_{t_{k-1}}}}^2\right]}
\sqrt{\mathbb{E} \left[H_k(\Xbar^{\theta}_{t_{k-1}})^2\right]} \eqsp,
\end{align*}
and similarly,
\begin{align*}
\int \sqrt{\cstediscrgamma+\cstediscrgammax \normEc{\x}^{4p+4}} H_k(\x) \pgen_{k-1}(\rmd \x)
&\le
\sqrt{\cstediscrgamma+\cstediscrgammax \mathbb{E} \left[ \normEc{\Xbar^{\theta}_{t_{k-1}}}^{4p+4}\right]}
 \sqrt{\mathbb{E} \left[H_k(\Xbar^{\theta}_{t_{k-1}})^2\right]} \eqsp.
\end{align*}
Multiplying by $\Delta_k$ and $\sqrt{3\Delta_k}$ yields the claimed bound. It remains to check the finiteness of the expectations on the right-hand side. In particular, let $G(\x)\eqdef (1+b \lyapunov{2}(\x))^{2}$ and
\begin{align*}
H_k(\x)\eqdef
\sqrt{\bwdker{\grid[k-1]}[\grid[k]]\,G(\x)}
+
\sqrt{\approxbwdker{\grid[k-1]}[\grid[k]]\,G(\x)} \eqsp.
\end{align*}
Then, using $(a+b)^2\le 2(a^2+b^2)$,
\begin{align*}
\mathbb{E} \left[ H_k \big(\Xbar^\theta_{t_{k-1}}\big)^2 \right] & \le
2(\pgen_{k-1}\bwdker{t_{k-1}}[t_k])[G]
+ 2(\pgen_{k-1}\approxbwdker{t_{k-1}}[t_k])[G] \\
&\le 2 (\refmeas \approxbwdkercomp{0}{ k-1} \bwdker{\grid[k-1]}[\grid[k]]) \left[G\right]
+
2  \refmeas \approxbwdkercomp{1}{ k} \left[G\right] \eqsp.
\end{align*}
In particular, since $G(\x)\lesssim 1+\lyapunov{4}(\x)$, \cref{hyp:schema-numerique_moments} yields $ \refmeas \approxbwdkercomp{1}{ k} \left[G\right] <\infty$.
Moreover, applying~\cref{prop:backward_drift_lyapunov} with $\ell=4$ and using again~\cref{hyp:schema-numerique_moments},
\begin{align*}
(\refmeas  \approxbwdkercomp{1}{ k-1}  \bwdker{\grid[k-1]}[\grid[k]]) \left[\lyapunov{4}(\x)\right]
\le
\multlyap{\grid[k-1]}{\grid[k]}[4] \refmeas  \approxbwdkercomp{0}{ k-1}  \left[\lyapunov{4}(\x)\right] 
+\biaslyap{\grid[k-1]}{\grid[k]}[4]
<\infty \eqsp,
\end{align*}
and therefore 
\begin{align*}\mathbb{E} \left[ H_k \big(\Xbar^\theta_{t_{k-1}}\big)^2 \right]<\infty \eqsp.\end{align*} 
\end{proof}





\subsection{Technical lemmas for the main proof}
\label{app:technical_lemmas:distances-and-inequalities}

First, this section provides the pathwise change-of-measure identity used to control the KL divergence between the exact backward SDE over one step and the continuous-time interpolation of the (discrete) SGM dynamics (\cref{lem:boundgirsanovproof} and \cref{cor:boundgirsanovproof}). Second, under the uniform Hessian control from \cref{cor:bound_hessian} we write the SDE representation of the score function that is used in \cref{prop:one_step_discr_error}.

\paragraph{One-step Girsanov type bound}
\begin{lemma} \label[lemma]{lem:boundgirsanovproof}
Fix a discretization grid $0=\grid[1] \leq \cdots \leq \grid[\gridn]=T$, an index $k \in \{1, \cdots, \gridn-1\}$ and a point $\x\in \rset^d$. Let $(\Xola_t)_{t\in[t_k,t_{k+1}]}$ be solution to the backward SDE
\begin{align} \label{eq:backward_sde_girsanov}
\rmd \Xola_t = \bwdnoisesch{t}  b(t, \Xola_t) \rmd t + \sqrt{2 \bwdnoisesch{t}} \dbrown{t} \eqsp, \qquad \Xola_{t_k} = \x \eqsp,
\end{align} 
with $b(t,y) \eqdef \isvp y + 2 \score[T-t][y]$. Denote by $\mathbb P_{t_k}^{ \x}$ the distribution of \eqref{eq:backward_sde_girsanov} on $(\Xola_t)_{t \in [t_k,t_{k+1}]}$. Similarly, let $(\bar\X_t^{\theta})_{t\in[t_k,t_{k+1}]}$ be the continuous-time interpolation of the generative model defined as solution to
\begin{align} \label{eq:euler_approx_girsanov}
\rmd \bar \X_t^\theta
=  \bwdnoisesch{t} b_{\theta}(t_k,\x)  \rmd t
+ \sqrt{2 \bwdnoisesch{t}} \dbrown{t} \eqsp , \qquad \bar \X_{t_k}^\theta = \x \eqsp,
\end{align}
with $b_{\theta}(t_k,\x) \eqdef \isvp \x + 2 \scorenet[T-t_k][\x]$. Denote by $\mathbb{P}^{\x, \theta}_{t_k}$ the distribution of \eqref{eq:euler_approx_girsanov} on $(\bar \X_t^\theta)_{[t_k,t_{k+1}]}$. Suppose that \cref{assump:p0:score} and \cref{assump:p0:hess} hold. Then, it follows that
\begin{align*}
    &\mathbb{E} \left[ \log \left( \frac{\rmd \mathbb P_{t_k}^{\x}}{\rmd \mathbb P_{t_k}^{\x,\theta}} \left((\Xola_{t})_{t\in [t_k,t_{k+1}]} \right) \right) \right]
    \\
    &= \mathbb{E}\Bigg[
    \int_{t_k}^{t_{k+1}}
    \sqrt{\frac{\bwdnoisesch{t}}{2}}
    \big(b(t, \Xola_t )-b_\theta(t_k,\x)\big)^\top \rmd B_t
    +\frac14 \int_{t_k}^{t_{k+1}}
    \bwdnoisesch{t}
    \normEc{b(t, \Xola_t ) - b_\theta(t_k,\x)}^2 \rmd t
    \Bigg] \eqsp,
\end{align*}
where $(B_t)_{t\in[t_k,t_{k+1}]}$ is a Brownian motion under $\mathbb P_{t_k}^{\x}$.
\end{lemma}

\begin{proof}
Let $\Omega \eqdef C([t_k,t_{k+1}],\rset^d)$ and $\X=(\X_t)_{t\in[t_k,t_{k+1}]}$ be the canonical process on $\Omega$. Introduce the following reference process
\begin{align} \label{eq:scaled_brownian}
\rmd \mathbf{Z}_t =  \sqrt{2 \bwdnoisesch{t}} \dbrown{t} \eqsp, \qquad \mathbf{Z}_{t_k} = \x \eqsp,
\end{align} 
and denote by $\mathbb{W}^{\x}_{t_k}$ the distribution of \eqref{eq:scaled_brownian} on $\Omega$.

\smallskip
\noindent\emph{Step 1: Time-changed processes.}

Our goal is to apply Theorem~7.7 of \citet{LipsterShiryaev}, which is stated for diffusion processes with unit diffusion coefficient. We therefore introduce the deterministic time-change
\begin{align*}
\tau(t)\coloneqq \int_{t_k}^t 2\bwdnoisesch{s} \rmd s \eqsp , \text{ for } t\in[t_k,t_{k+1}] \eqsp,
\end{align*}
and denote by $\tau^{-1}:[0,\tau(t_{k+1})]\to[t_k,t_{k+1}]$ its inverse. By the inverse function theorem,
\begin{align*}
(\tau^{-1})'(u)=\frac{1}{2\bwdnoisesch{\tau^{-1}(u)}} \eqsp.
\end{align*}

Define the time-changed path space $\widehat{\Omega}\coloneqq C([0,\tau(t_{k+1})],\rset^d) $ and introduce the bijective time-change map $\widehat{T}:\Omega\to\widehat{\Omega}$ by
\begin{align*}
(\widehat{T} \omega)(u)\eqdef \omega(\tau^{-1}(u)),\qquad u\in[0,\tau(t_{k+1})] \eqsp.
\end{align*}
Its inverse $\widehat{T}^{-1}:\widehat{\Omega}\to\Omega$ is given by $(\widehat{T}^{-1}\hat\omega)(t)=\hat\omega(\tau(t))$ for $t\in[t_k,t_{k+1}]$. For $\widehat{\mathbf Z}_u\eqdef \mathbf Z_{\tau^{-1}(u)}$, we get that $(\widehat{\mathbf Z}_u)_{u\in[0,\tau(t_{k+1})]}$ is a standard $d$-dimensional Brownian motion starting from $\x$, and we denote by $\widehat{\mathbb W}^{\x}_{t_k}\eqdef \mathbb W_{t_k}^{\x}\circ \widehat{T}^{-1}$ its law on $\widehat{\Omega}$. Similarly, define the time-changed processes
\begin{align*}
\widehat \Xola_u\eqdef \Xola_{\tau^{-1}(u)}
\qquad
( \text{resp. } \widehat \Xbar_u^\theta\eqdef \Xbar_{\tau^{-1}(u)}^\theta) \eqsp,
\end{align*}
satisfying
\begin{align*}
\rmd \widehat \Xola_u
=\tfrac12 b(\tau^{-1}(u),\widehat \Xola_u)\,\rmd u+\rmd \widehat{\mathbf Z}_u \eqsp,
\qquad
\rmd \widehat \Xbar_u^\theta
=\tfrac12 b_\theta(t_k,\x)\,\rmd u+\rmd \widehat{\mathbf Z}_u \eqsp.
\end{align*}
We denote by $\widehat{\mathbb P}_{t_k}^{\x} =\mathbb P_{t_k}^{\x} \circ \widehat{T}^{-1} $ and $\widehat{\mathbb P}_{t_k}^{\x,\theta} = \mathbb P_{t_k}^{\x,\theta} \circ \widehat{T}^{-1} $ their respective laws on $\widehat{\Omega}$.

\smallskip
\noindent\emph{Step 2: density processes.} 

Let $\widehat{\X}=(\widehat{\X}_u)_{u\in[0,\tau(t_{k+1})]}$
be the canonical process on $\widehat\Omega$.
Under $\widehat{\mathbb W}_{t_k}^\x$, $\widehat{\X}$ is a standard $d$-dimensional Brownian motion with $\widehat{\X}_0=\x$. Moreover, under $\widehat{\mathbb P}_{t_k}^{\x}$ (resp. $\widehat{\mathbb P}_{t_k}^{\x,\theta}$) it has the same distribution as $(\widehat \Xola_u)_{u \in [0, \tau(t_{k+1})]}$ (resp. $(\widehat \Xbar_u^\theta)_{u \in [0, \tau(t_{k+1})]}$). 

It follows from \cref{prop:growth_stability}  and \cref{prop:backward_drift_lyapunov} that
\begin{align*}
    &\mathbb{E} \left[ \int_{0}^{\tau(t_{k+1})} \frac14  \normEc{b(\tau^{-1}(u),\widehat \Xola_u)}^2 \rmd u \right]
    \\
    & = \mathbb{E} \left[ \int_{0}^{\tau(t_{k+1})} \frac14 \normEc{\isvp \widehat \Xola_u + 2 \score[T- \tau^{-1}(u) ][\widehat \Xola_u]}^2 \rmd u \right] < \infty
    \eqsp.
\end{align*}
From \cref{prop:growth_stability}, we also get that
\begin{align*}
\mathbb{E} \left[ \int_{0}^{\tau(t_{k+1})} \frac14  \normEc{b(\tau^{-1}(u), \widehat{\mathbf Z}_u)}^2 \rmd u \right] < \infty \eqsp.
\end{align*}
Since the integrands are nonnegative, these imply the corresponding a.s.\ finiteness conditions required in Theorem 7.7 in
\citet{LipsterShiryaev}. Then, by Theorem 7.7 in  \citet{LipsterShiryaev}, the path measures are equivalent $\widehat{\mathbb P}_{t_k}^{\x} \sim \widehat{\mathbb W}^{\x}_{t_k}$ and
\begin{align*}
    \begin{split}
        \frac{\rmd \widehat{\mathbb P}_{t_k}^{\x}}{\rmd \widehat{\mathbb W}^{\x}_{t_k}} ( \widehat{\X} )
        =
        \exp \left\{
        \int_0^{\tau(t_{k+1})}
        \tfrac12\, b(\tau^{-1}(u),\widehat{\X}_u )^\top \rmd \widehat{\X}_u
        -\frac18 \int_0^{\tau(t_{k+1})} \big\|b(\tau^{-1}(u), \widehat{\X}_u )\big\|^2  \rmd u
        \right\}
        \\
        \widehat{\mathbb W}^{\x}_{t_k} \text{ a.s.}
    \end{split}
\end{align*}
Moreover, since $b_{\theta}(t_k,\x)$ is deterministic, the same theorem yields
$\widehat{\mathbb P}_{t_k}^{\x,\theta} \sim \widehat{\mathbb W}^{\x}_{t_k}$ and
\begin{align*}
\frac{\rmd \widehat{\mathbb P}_{t_k}^{\x,\theta}}{\rmd \widehat{\mathbb W}^{\x}_{t_k}}(\widehat{\X})
=
\exp \left\{
\int_0^{\tau(t_{k+1})}
\tfrac12\, b_{\theta}(t_k, \x )^\top \rmd \widehat{\X}_u
-\frac18 \int_0^{\tau(t_{k+1})} \big\|b_{\theta}(t_k, \x )\big\|^2 \,\rmd u
\right\} \quad \widehat{\mathbb W}^{\x}_{t_k} \text{ a.s.}
\end{align*}
Therefore, the ratio of the above densities yields $\widehat{\mathbb P}_{t_k}^{\x,\theta}$ almost surely,
\begin{align} \label{eq:density_process}
\frac{\rmd \widehat{\mathbb P}_{t_k}^{\x}}{\rmd \widehat{\mathbb P}_{t_k}^{\x,\theta}} (\widehat{\X})
=
\exp\Bigg\{
\int_0^{\tau(t_{k+1})}
\tfrac12\Big(b(\tau^{-1}(u), \widehat{\X}_u )-b_\theta(t_k,\x)\Big)^\top \rmd \widehat{\X}_u
\\
-\frac18\int_0^{\tau(t_{k+1})}
\Big(\normEc{b(\tau^{-1}(u),\widehat{\X}_u )}^2- \normEc{b_\theta(t_k,\x)}^2\Big)\rmd u
\Bigg\} \eqsp.
\end{align}
By equivalence of the measures, equation \eqref{eq:density_process} also holds $\widehat{\mathbb P}_{t_k}^{\x}$ almost surely. 

Under $\widehat{\mathbb P}_{t_k}^{\x}$ the process
\begin{align*}
\widehat{\mathrm B}_u
\coloneqq
 \widehat{\X}_u
-\x
-\int_0^u \tfrac12 b(\tau^{-1}(s),\widehat{\X}_s) \rmd s,
\qquad \text{for } u\in[0,\tau(t_{k+1})] \eqsp,
\end{align*}
is a $d$-dimensional standard Brownian motion. Consequently, under $\widehat{\mathbb P}_{t_k}^{\x}$, $\rmd \widehat{\X}_u =   \tfrac12 b(\tau^{-1}(u),\widehat{\X}_u)\rmd u + \rmd \widehat{\mathrm B}_u$. As a consequence, $\widehat{\mathbb P}_{t_k}^{\x}$ almost surely,
\begin{align*}
\frac{\rmd \widehat{\mathbb P}_{t_k}^{\x}}{\rmd \widehat{\mathbb P}_{t_k}^{\x,\theta}} (\widehat{\X})
=
\exp\Bigg\{
\int_0^{\tau(t_{k+1})}
\tfrac12\Big(b(\tau^{-1}(u), \widehat{\X}_u )-b_\theta(t_k,\x)\Big)^\top \rmd \widehat{\mathrm B}_u
\\
+\frac18\int_0^{\tau(t_{k+1})}
\Big(\normEc{b(\tau^{-1}(u),\widehat{\X}_u ) - b_\theta(t_k,\x)}^2\Big)\rmd u
\Bigg\} \eqsp.
\end{align*}

\smallskip
\noindent\emph{Step 3: density process for the original-time laws.}

Let $\widehat{T}:\Omega\to\widehat\Omega$ be the deterministic time-change map defined in Step~1. By construction of the time-changed processes, recall that
\begin{align*}
\widehat{\mathbb P}_{t_k}^{\x}=\mathbb P_{t_k}^{\x}\circ \widehat{T}^{-1},
\qquad
\widehat{\mathbb P}_{t_k}^{\x,\theta}=\mathbb P_{t_k}^{\x,\theta}\circ \widehat{T}^{-1} \eqsp.
\end{align*}
In particular, we have equivalence of the measures, \ie, $\mathbb P_{t_k}^{\x}\sim \mathbb P_{t_k}^{\x,\theta}$ and for any test function $\psi: \Omega \to \mathbb{R}$, 
\begin{align*}
    \int_{\Omega} \psi(\X) \rmd \mathbb P_{t_k}^{\x} = \int_{\hat \Omega} \psi(\widehat{T}^{-1} \widehat{\X}) \widehat{\mathbb P}_{t_k}^{\x}
    = \int_{\widehat{\Omega}} \psi(\widehat{T}^{-1} \widehat{\X}) \frac{\rmd \widehat{\mathbb P}_{t_k}^{\x} }{ \rmd \widehat{\mathbb P}_{t_k}^{\x,\theta}} (\widehat{\X}) \rmd \widehat{\mathbb P}_{t_k}^{\x,\theta} = \int_{\Omega} \psi( \X) \frac{\rmd \widehat{\mathbb P}_{t_k}^{\x} }{ \rmd \widehat{\mathbb P}_{t_k}^{\x,\theta}} (\widehat{T} \X) \rmd \mathbb P_{t_k}^{\x,\theta}
    \eqsp,
\end{align*}
so that
\begin{align*}
\frac{\rmd \mathbb P_{t_k}^{\x}}{\rmd \mathbb P_{t_k}^{\x,\theta}}(\X)
=
\frac{\rmd \widehat{\mathbb P}_{t_k}^{\x}}{\rmd \widehat{\mathbb P}_{t_k}^{\x,\theta}}(\widehat{T} \X)\eqsp,
\qquad
\mathbb P_{t_k}^{\x,\theta}\text{-a.s.}
\end{align*}
In particular, using that $\mathbb P_{t_k}^{\x,\theta} \sim \mathbb P_{t_k}^{\x}$,
\begin{align*}
    & \mathbb{E}_{\mathbb P_{t_k}^{\x}} \left[ \log \left( \frac{\rmd \mathbb P_{t_k}^{\x}}{\rmd \mathbb P_{t_k}^{\x,\theta}} (\X) \right) \right] \\
    & =\mathbb{E}_{\mathbb P_{t_k}^{\x}} \left[ \int_{t_k}^{t_{k+1}}
    \sqrt{\tfrac{\bwdnoisesch{t}}{2}}\Big(b(t, \X_t )-b_\theta(t_k,\x)\Big) \rmd {\mathrm B}_t +\frac14 \int_{t_k}^{t_{k+1}}
    \bwdnoisesch{t} \Big(\normEc{b(t, \X_t ) - b_\theta(t_k,\x)}^2\Big)\rmd t \right]
    \eqsp,
\end{align*}
which concludes the proof.
\end{proof}

\begin{corollary} \label[corollary]{cor:boundgirsanovproof}
Under \cref{assump:p0}, 
\begin{align*}
    \mathbb{E}\left[
    \int_{t_k}^{t_{k+1}}
    \frac{\bwdnoisesch{t}}{2}
    \normEc{b(t, \Xola_t )-b_\theta(t_k,\x) }^2   \rmd t \right] < \infty \eqsp.
\end{align*}
Thus, we get
\begin{align*}
       \kl{\mathbb P_{t_k}^{ \x}}{\mathbb P_{t_k}^{ \x,\theta}} = \frac14 \mathbb{E}\left[ \int_{t_k}^{t_{k+1}}
\bwdnoisesch{t}\,
\|b(t, \Xola_t ) - b_\theta(t_k,\x)\|^2\,\rmd t
\right] \eqsp.
    \end{align*}
\end{corollary}

\begin{proof}
Note that,
\begin{align*}
    &\mathbb{E} \left[
    \int_{t_k}^{t_{k+1}}
    \frac{\bwdnoisesch{t}}{2}
    \normEc{b(t,\Xola_t)-b_\theta(t_k,\x)}^2 \rmd t\right]
    \\
    &\le
    \mathbb{E} \left[\int_{t_k}^{t_{k+1}}
    \bwdnoisesch{t} \normEc{b(t,\Xola_t)}^2 dt\right]
    +\normEc{b_\theta(t_k,\x)}^2 \int_{t_k}^{t_{k+1}}\bwdnoisesch{t}\rmd t 
    \eqsp.
\end{align*}
Moreover, using \cref{prop:growth_stability}, there exists a universal constant $C_k$ (depending on $[t_k,t_{k+1}]$), such that
\begin{align*}
\mathbb{E} \left[\int_{t_k}^{t_{k+1}}
\bwdnoisesch{t} \normEc{b(t,\Xola_t)}^2  \rmd t\right]
&\le
C_k\int_{t_k}^{t_{k+1}}\bwdnoisesch{t}
\mathbb{E} \left[1+ \normEc{\Xola_t}^{2p+2}\right] \rmd t \eqsp.
\end{align*}
Using \cref{prop:backward_drift_lyapunov} with $\ell=2p+2$, we have
$\sup_{t\in[t_k,t_{k+1}]}\mathbb E[\|\Xola_t\|^{2p+2}]<\infty$. Therefore, 
\begin{align*}
\mathbb{E} \left[\int_{t_k}^{t_{k+1}}
\bwdnoisesch{t} \normEc{b(t,\Xola_t)}^2  \rmd t\right] < \infty \eqsp,
\end{align*}
which implies
\begin{align*}
\mathbb{E} \left[ \log \left( \frac{\rmd \mathbb P_{t_k}^{\x}}{\rmd \mathbb P_{t_k}^{\x,\theta}} \left((\Xola_{t})_{t\in [t_k,t_{k+1}]} \right) \right) \right] = \frac14 \mathbb{E}\Bigg[ \int_{t_k}^{t_{k+1}}
\bwdnoisesch{t}\,
\normEc{b(t, \Xola_t ) - b_\theta(t_k,\x)}^2 \rmd t
\Bigg] \eqsp.
\end{align*}
\end{proof}

\paragraph{Backward evolution of the score.}
Consider the process obtained by evaluating the score function along the backward trajectory, namely
\[
\Y_t \eqdef \score[T-t][\Xola_t],
\]
where \((\Xola_t)_{t\ge 0}\) is a weak solution to \eqref{eq:backward_SDE}. This process can be shown to satisfy a SDE on the entire interval [0,T], as detailed in the following result. The well-posedness of this SDE, as well as the validity of the representation \eqref{eq:SDE-score}, rely critically on
\cref{cor:bound_hessian}.





\begin{lemma}
    \label[lemma]{score differential form}
    Suppose that \cref{assump:p0} hold. Then, for any $\x \in \rset^{\xdim} $, we have
    \begin{align}
        \label{eq:SDE-score}
        \rmd \Y_t = - \isvp \bwdnoisesch{t}  \Y_t  \rmd t + \sqrt{2 \bwdnoisesch{t}} \nabla \score[T-t][\Xola_t]\dbrown{t}
        \eqsp.
    \end{align}
\end{lemma}
\begin{proof}
    Consider the Fokker--Planck equation associated with the forward process \eqref{eq:forward_SDE}, \ie,
    \begin{align}
    \label{eq:Fokker-Planck}
        \partial_{t} \fwdmarg{t}(\x) = \isvp \noisesch{t} \divergence{\x \fwdmarg{t} (\x)} + \noisesch{t} \Delta \fwdmarg{t} (\x)\eqsp,
    \end{align}
    for $\x \in \rset^\xdim, t\in (0,T]$.
    We manipulate the preceding equation to derive the partial differential equation satisfied by the function \((t,x)\mapsto \log \fwdmarg{t}(x)\). Thus,
    \begin{align*}
        \partial_{t} \log \fwdmarg{t}(\x) = \isvp \noisesch{t} \frac{\divergence{\x \fwdmarg{t} (\x)}}{\fwdmarg{t}(\x)} +  \noisesch{t} \frac{\Delta \fwdmarg{t} (\x)}{\fwdmarg{t}(\x)}
        \eqsp.
    \end{align*}
    Introducing $\Delta \fwdmarg{t}/ \fwdmarg{t} = \Delta \log \fwdmarg{t} + \normEc{\nabla \log \fwdmarg{t}}^2$ and expanding the previous equation yields
    \begin{align}
        \partial_{t} \log \fwdmarg{t}(\x) = \isvp \noisesch{t} ( d + \dotprod{\x}{\nabla \log \fwdmarg{t}})+ \noisesch{t} \left(  \Delta \log \fwdmarg{t} + \normEc{\nabla \log \fwdmarg{t}}^2  \right)\eqsp. \label{eq:fokkerplank:logpt}
    \end{align}
    Using Itô's Lemma, $Y_t = \nabla \phi_t ( \Xola_t)$, with $\phi_t(x) = \log \fwdmarg{T-t}(x)$ we get
    \begin{align*}
        &\rmd Y_t
        \\
        &= \Big[
        \partial_t \nabla \phi_t\left(\Xola_t\right) + \nabla^2 \phi_t\left(\Xola_t\right)\left(   
            \isvp \bwdnoisesch{t} \Xola_t + 2 \bwdnoisesch{t} \nabla\phi_t\left(\Xola_t\right)
        \right)
         \\
        &  \qquad\qquad\qquad\qquad
        \qquad\qquad 
        +
        \bwdnoisesch{t} \Delta \nabla \phi_t\left(\Xola_t\right)
        \Big]\rmd t + \sqrt{2 \bwdnoisesch{t}} \nabla^2 \phi_t\left(\Xola_t\right) \dbrown{t}
        \\
        &=
        \Big[\nabla\left(
            \partial_t \phi_t\left(\Xola_t\right)+
            \bwdnoisesch{t} \left(\Delta \phi_t\left(\Xola_t\right)+
            \normEc{\nabla\phi_t\left(\Xola_t\right)}^2\right)
        \right)
        \\
            & \qquad\qquad\qquad\qquad
        \qquad\qquad 
        +
        \isvp \bwdnoisesch{t}
        \nabla^2 \phi_t\left(\Xola_t\right) \Xola_t
        \Big]\rmd t + \sqrt{2 \bwdnoisesch{t}} \nabla^2 \phi_t\left(\Xola_t\right) \dbrown{t} \eqsp,
    \end{align*}
    where we used that, for $\x \in \rset^\xdim$ $ 2 \nabla^2 \phi_t(\x)  \nabla \phi_t(\x) = \nabla \normEc{\nabla \phi_t(\x)}^2$. By \eqref{eq:fokkerplank:logpt},  
    \begin{align*}
        \partial_{t} \phi_t (\x) = -\isvp \noisesch{T-t} ( d + \dotprod{\x}{\nabla \log \fwdmarg{t}}) - \bwdnoisesch{t} \left(  \Delta \log \fwdmarg{t} + \normEc{\nabla \log \fwdmarg{t}}^2  \right)\eqsp, 
    \end{align*}
    which implies
    \begin{align*}
        \rmd Y_t = \left[ - \isvp \bwdnoisesch{t} \left( \nabla \phi_t(\Xola_t) + \nabla^2 \phi_t (\Xola_t) \Xola_t  \right) + \isvp \bwdnoisesch{t} \nabla^2 \phi_t (\Xola_t) \Xola_t \right] \rmd t 
        + \sqrt{2 \bwdnoisesch{t}} \nabla^2 \phi_t\left(\Xola_t\right) \dbrown{t} \eqsp.
    \end{align*}
    which implies \eqref{eq:SDE-score}. 
    Note that the Fokker–Planck equation \eqref{eq:Fokker-Planck} is only known to hold for $t \in (0,T]$. Therefore, the validity of this SDE representation cannot be extended directly up to the terminal time $T$. To overcome this limitation, it is necessary to establish additional regularity properties of the stochastic terms involved. In particular, we must show that the quadratic variation of the martingale component in \eqref{eq:SDE-score} remains uniformly bounded on $[0,T]$, which allows the SDE to be extended to the full time horizon and ensures the well-posedness of the backward dynamics. 
    
    Applying \cref{cor:bound_hessian}, there exists a constant $C>0$, independent of $t\in[0,T)$, that bounds the quadratic variation of \eqref{eq:SDE-score} as follows
    \begin{align*}
        \PE{}{
            \int_{0}^{t}2 \bwdnoisesch{s}\normFr{\jscore[T-s][\Xola_s]}^2 \rmd s
        \leq
        }
        \PE{}{
            \int_{0}^{t}4 \bwdnoisesch{s}
            \ctegrowthHess{T-s}^2 \left(1 + \normEc{\Xola_s}^{4p+4}\right)
            \rmd s
        }
        \leq C
        \eqsp,
    \end{align*}
    where in the last inequality we used the continuity of $t\mapsto \ctegrowthHess{t}$ from \cref{cor:bound_hessian}, together with the Gaussian representation \eqref{forward_marginal}. Indeed, \cref{lem:forward_process_law} shows that the law of $\Xola_t$, equal to the law of $\Xora_{T-t}$, is the product of a continuous factor times a Gaussian convolution. As both the normal distribution and $p_0$ admit any order moment, from \cref{lem:pdata-sub-gaussian}, we have that the previous bound is uniform in time and does not explode for $t\to T$.
\end{proof}

\section{Numerical protocols and additional illustrations} \label{app:experiments}


    This appendix provides the experimental details and additional numerical results supporting the forgetting/stability mechanism illustrated in the main text. We first describe the controlled synthetic settings used to isolate the effect predicted by the theory, and then report an additional real-data experiment on CIFAR-10 showing that the same qualitative trend remains visible for a pretrained diffusion model.
    All experiments were run on a server equipped with a single NVIDIA A100-SXM4 GPU with 40GB of memory. The synthetic experiments were designed to be lightweight and reproducible on CPU.

\subsection{Synthetic datasets} \label{subapp:datasets}

We study the following probability distributions that satisfy \cref{assump:p0}.

\paragraph{Gaussian random vectors.} We set dimension to $d=50$ and consider Gaussian data distributions of the form
\begin{align*}
\pi_{\rm data}=\mathcal N(\mu,\Sigma),
\qquad 
\mu \coloneqq \mathbf 1_d \in \mathbb R^d,
\end{align*}
where $\mathbf 1_d$ denotes the vector of ones. We study three covariance structures:

\begin{enumerate}[leftmargin=1.2em]
\item \textbf{Isotropic.}
We take an isotropic covariance
\begin{align*}
\Sigma^{(\mathrm{iso})} \eqdef \sigma^2 \Id_\xdim,
\qquad \sigma^2 = 0.1,
\end{align*}
so that all coordinates have the same variance and are uncorrelated.

\item \textbf{Anisotropic (Heteroscedastic).}
We consider a diagonal covariance with two variance levels:
\begin{align*}
\Sigma^{(\mathrm{heterosc})} \eqdef \mathrm{diag}(v_1,\dots,v_d),
\qquad
v_j \eqdef 
\begin{cases}
1, & 1\le j\le 5,\\
10^{-3}, & 6\le j\le d,
\end{cases}
\end{align*}
so that the first five coordinates have much larger variance than the remaining ones (no cross-correlation).

\item \textbf{Correlated.}
We use a dense covariance matrix with unit marginal variances and slowly decaying off-diagonal correlations:
\begin{align*}
\Sigma^{(\mathrm{corr})}_{jj} \eqdef 1,
\qquad
\Sigma^{(\mathrm{corr})}_{jj'} \eqdef \frac{1}{\sqrt{|j-j'|+1}}
\quad\text{for } j\neq j',
\end{align*}
for $1\le j,j'\le d$.
\end{enumerate}

\paragraph{Gaussian Mixture Model.} We consider a Gaussian mixture model in $\rset^{50}$ defined by
\begin{align*}
    \pdata[\x] = \sum_{i=1}^{25} \omega_i \gaussiand{\mu_i}{\Sigma_i} \eqsp,
\end{align*}
where 
\begin{itemize}
    \item $\{\omega_i\}_{i=1}^{25}$ are sampled i.i.d. from a $\xi^2$ distributions with 3 degrees of freedom and re-normalized to sum to $1$,
    \item $\{\mu_i\}_{i=1}^{25}$ have all but the first $2$ coordinates at $0$. The first two coordinates are evenly spaced at a square with lower corner at $(-10, -10)$ and upper corner at $(10, 10)$,
    \item Each $\Sigma_i$ is of the type $U_i^T D U_i$ where $U_i$ is obtained as one of the orthonormal matrices from the SVD of a random Gaussian matrix $\tilde{U}_i$ and $D$ is a diagonal matrices with entries $(1, 1/2, \cdots, 1/25)$.
\end{itemize}
We used a single draw of the aforementioned random numbers for all the experiments.

\subsection{Sensitivity to initialization} \label{subapp:sensitivity_to_init}

\paragraph{Gaussian targets.} This experiment isolates the effect of a controlled perturbation injected at an intermediate diffusion time $t_{\rm bias}\in(0,T]$ on the final generated sample, and studies how this effect varies with the perturbation time and the perturbation magnitude.

Fix $t_{\rm bias}\in(0,T]$ and consider the backward reverse-time transition represented by the Markov kernel $\bwdker{t_{\rm bias}}[T]$, together with its numerical approximation $\bwdkerdiscr{t_{\rm bias}}[T]$ obtained by time-discretization.
We investigate whether we observe a robustness behavior for the \emph{discretized} reverse chain with respect to the true kernel by monitoring the discrepancy
\begin{align*}
\wasserstein[2]{p_{t_{\rm bias}} \bwdker{{t_{\rm bias}}}[T]}{\tilde p_{t_{\rm bias}} \bwdkerdiscr{t_{\rm bias}}[T]}
= \wasserstein[2]{\pi_{\rm data}}{\tilde p_{t_{\rm bias}} \bwdkerdiscr{{t_{\rm bias}}}[T]} \eqsp.
\end{align*}
We define a biased initialization at time $t_{\rm bias}$by applying a shift in a random  unit direction $u$:
\begin{align*}
\x_{{t_{\rm bias}},\lambda}^{(i)} = \x_{t_{\rm bias}}^{(i)} +  \lambda u,
\qquad \|u\|=1,
\end{align*}
where $u$ is drawn once per replication and $\lambda\ge 0$ is chosen to control the bias magnitude. For each $({t_{\rm bias}},\lambda)$, we run an Euler--Maruyama discretization of the reverse-time dynamics from time $t_{\rm bias}$ down to $T$ and obtain samples $\hat \x_T^{(i)}$. We fix the step-size to $h$ throughout the experiment. A  pseudocode of the protocol is given in~\Cref{alg:init-bias}. On Gaussian targets, the \emph{exact} reverse semigroup admits closed-form expressions and satisfies a strict contraction in $\mathcal W_2$, which provides a clean reference behavior. The goal here is not to re-derive these formulas, but to assess whether the \emph{numerical} reverse chain exhibits the same qualitative robustness to perturbations.

\begin{algorithm}[t]
\caption{Initialization-bias robustness at time ${t_{\rm bias}} \in \mathcal T_{\rm perturb} $.}
\label{alg:init-bias}
\begin{algorithmic}[1]
\State  Target distribution $\pi_{\rm data}$, forward SDE, reverse SDE discretization (Euler--Maruyama) with step size $h$, horizon $T$, number of points to generate $M$, perturbation times $\mathcal T_{\rm perturb} \subset(0,T]$, bias magnitudes $\Lambda\subset\mathbb R_+$, number of replications $\rm{rep}$.
\For{$r=1,\dots,\rm{rep}$}
    \State Sample a unit vector $u^{(r)} \sim \mathrm{Unif}(\mathbb{S}^{d-1})$ 
    \For{each $t \in \mathcal T_{\rm perturb}$}
        \For{each $\lambda \in \Lambda$}
        \State Sample $\x_t^{(i)} = \fwdmean{0}{t} \x_0^{(i)} + \fwdstd{0}{t} G^{(i)}$ for $i=1,\dots,M$ 
        \State Set biased initialization $\x_{t,\lambda}^{\mathrm{bias} (i)} \gets \x_t^{(i)} + \lambda u^{(r)}$ for $i=1,\dots,M$
        \State Run discretization of the reverse dynamics from time $t$ to $T$ with initial points $\{\x_{t,\lambda}^{\mathrm{bias}(i)}\}_{i=1}^M$
        \State Obtain $\{\hat \x_{T,\lambda}^{(i)}\}_{i=1}^M$ and empirical law $\widehat\pi_{T,\lambda}^{\rm bias}$
        \State Compute $\hat{\mathcal W}_2 (t,\lambda) \gets \mathcal W_2\!\left(\widehat\pi_{T,\lambda}^{\rm bias}, \pi_{\rm data}\right)$
        \EndFor
    \EndFor
\EndFor
\State \textbf{Plot:} for each $\lambda\in\Lambda$, plot $t\mapsto \frac{1}{\mathrm{rep}}\sum_{r=1}^{\mathrm{rep}} \hat{\mathcal W}_2^{(r)}(t,\lambda)$ (with $\pm$ one std).
\end{algorithmic}
\end{algorithm}

Since the target is Gaussian, we quantify errors by the closed-form Gaussian $\mathcal{W}_2$ (Bures) distance between the empirical output (via its sample mean/covariance) and $\pi_{\rm data}=\mathcal{N}(\mu,\Sigma)$. Each curve in Figure \ref{fig:comparison_gaussian_initial_bias} is obtained by averaging over 5 independent replications and 30 000 points are generated for each bias at each perturbation time. We use for constant step size and linear (non optimized) schedule with maximal value $\beta_{\max} = 20$  and minimal value $\beta_{\min} = 0.1$ in both the VP and VE case. We use the forward-time convention: $t=0$ corresponds to $\pidata$ and $t=T$ to the noisiest level used to initialize the reverse sampler. Across all covariance structures and for both VE and VP dynamics, we observe a clear forgetting trend: perturbations injected at larger diffusion times (\ie, early along the reverse trajectory) are substantially discounted. This behavior is consistent with the contraction/robustness mechanism established in the main text: the reverse-time dynamics progressively contracts discrepancies as it evolves toward the data distribution.

\paragraph{Mixture of Gaussian targets.} The experiment is similar to the one for Gaussien targets. However, as there is no analytical formula for the Wasserstein distance in this case, we employ the Maximum sliced Wasserstein distance (\maxsw) as in \citet[Section 4.3]{kolouri2019generalized} using a total of $10^{6}$ samples. The optimization procedure employs the Adam optimization algorithm with learning rate $10^{-3}$ and is stop when either the last optimization update is smaller than $10^{-7}$ or the optimization has reached a total of $10^5$ iterations.

In this setting, we consider exclusively the \emph{Variance Exploding} framework and define for $\lambda \in \rsetpos$ an initialization perturbation of the type 
\begin{align*}
\x_{{t_{\rm bias}},\lambda}^{(i)} = \x_{t_{\rm bias}}^{(i)} + \fwdvar{0}{t_{\rm bias}}\lambda u_{\operatorname{max}}
\end{align*}
where $u_{\operatorname{max}} \in \sphere{49}$ is obtained as the solution of the optimization problem for the {\maxsw} distance for two sets of $2\times 10^5$ independent samples of $\pdata$. All the calculations have been replicated $20$ times using different seeds and the reported results consist of the mean and the standard deviations of those $20$ replicates. For the scheduling, we used the scheduler from \citet[Equation 5]{karras2022elucidating} with $N=100$, $\sigma_{\operatorname{min}}=0.002$, $\sigma_{\operatorname{max}}=80$ and $\rho=3$. Results are reported in
\Cref{fig:gmm:maxsw} (Left)
and \Cref{fig:gmm:pert:init}.

\begin{figure}[t]
\centering
\begin{subfigure}{0.32\textwidth}
    \centering
    \includegraphics[width=\linewidth]{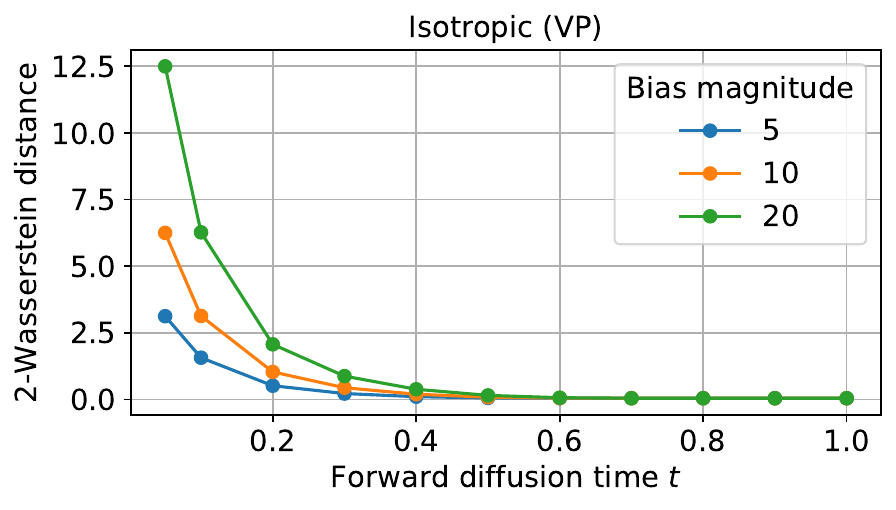}
\end{subfigure}
\hfill
\begin{subfigure}{0.32\textwidth}
    \centering
    \includegraphics[width=\linewidth]{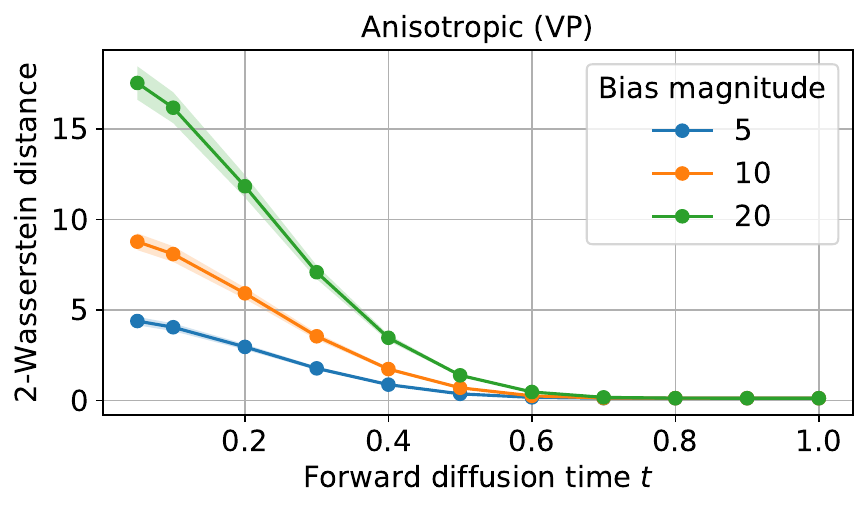}
\end{subfigure}
\hfill
\begin{subfigure}{0.32\textwidth}
    \centering
    \includegraphics[width=\linewidth]{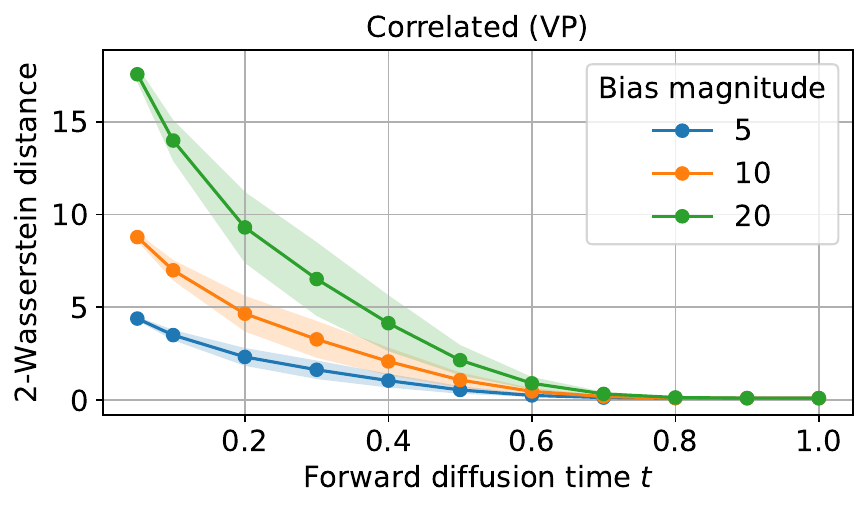}
\end{subfigure}

\vspace{0.3cm}

\begin{subfigure}{0.32\textwidth}
    \centering
    \includegraphics[width=\linewidth]{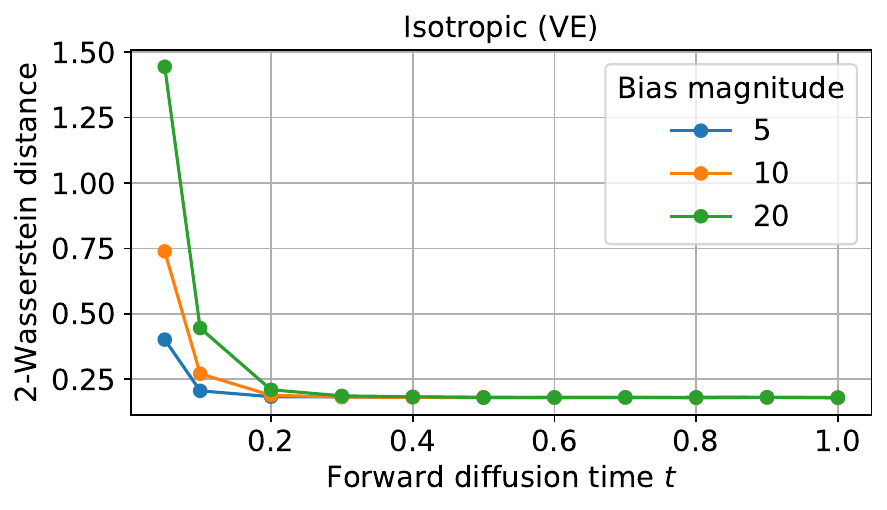}
\end{subfigure}
\hfill
\begin{subfigure}{0.32\textwidth}
    \centering
    \includegraphics[width=\linewidth]{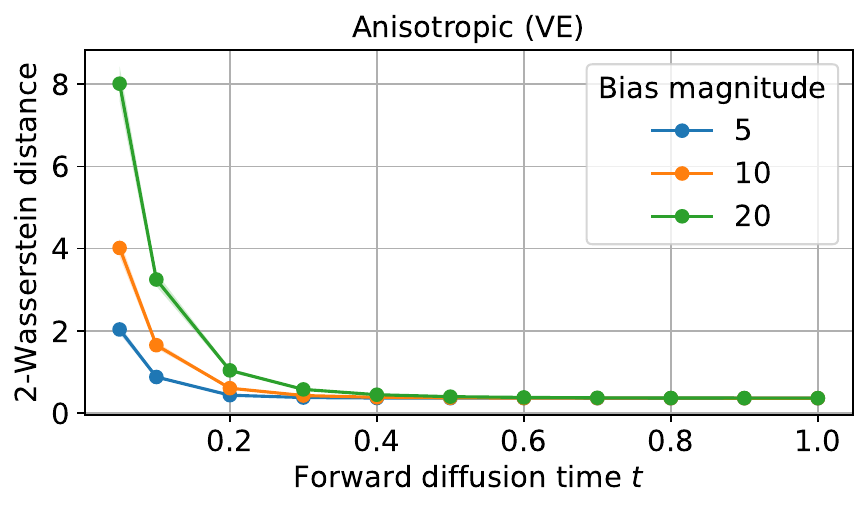}
\end{subfigure}
\hfill
\begin{subfigure}{0.32\textwidth}
    \centering
    \includegraphics[width=\linewidth]{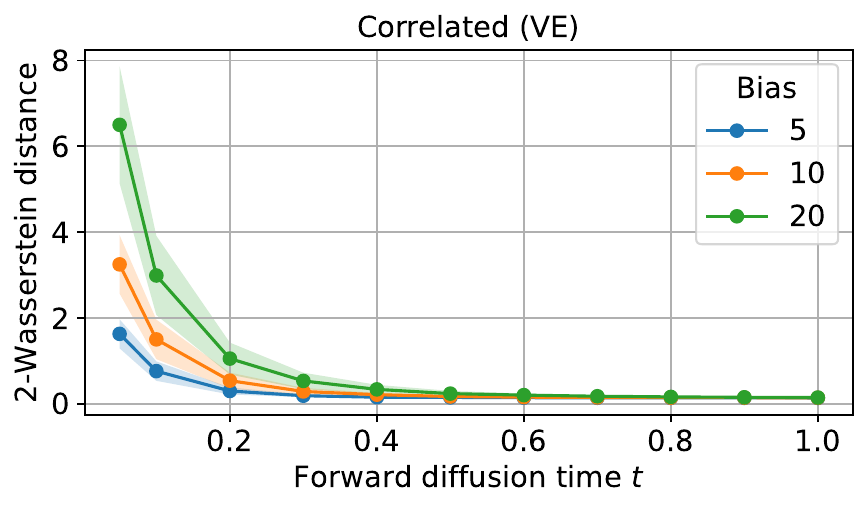}
\end{subfigure}

\caption{Sensitivity to initialization for isotropic, heteroscedastic, and correlated Gaussian targets. Curves report the mean over 5 independent replications of~\Cref{alg:init-bias} ($T=1$, $M=30{,}000$, $h=2.5\times10^{-3}$); shaded regions indicate $\pm 1$ standard deviation across replications. The horizontal axis is the forward diffusion time $t$. Each dotted point on the lines correspond to a perturbation time and the bias magnitude corresponds to the choice of $\lambda$.}
\label{fig:comparison_gaussian_initial_bias}
\end{figure}

\subsection{Sensitivity to approximation errors} \label{subapp:sensitivity_score_error}

This section describes a second numerical protocol designed to analyse the \emph{local-to-global error propagation} predicted by the stability results in the main text. Instead of modifying the initialization, we introduce a \emph{single local perturbation} in the reverse-time dynamics: at one prescribed time $t_{\mathrm{err}}\in[0,T]$, we perturb the score term used by the Euler--Maruyama discretization, and we measure how the resulting output error depends on the location $t_{\mathrm{err}}$ and magnitude of this perturbation.

Fix $n_{\mathrm{err}}\in \{ 0, \cdots, N-1 \}$ and consider the Euler--Maruyama discretization of the reverse-time SDE on the uniform grid. We run the reverse chains starting from the same terminal initialization $\x_0^{(i)}\sim \pi_{\infty}$ with the exact score function except at the unique step corresponding to $t_{\mathrm{err}}$ (\ie, $t_{n_{\mathrm{err}}})$, where we replace the score by a perturbed version only at the step $t_{n_{\mathrm{err}}}$.

\paragraph{Gaussian targets.}
We consider the following perturbation with  $u$ a random unit direction 
\begin{align*}
\tildescore[t_{{\mathrm{err}}}][\x]
=
\score[t_{{\mathrm{err}}}][\x] +  \frac{\lambda}{\fwdvar{0}{t_{\mathrm{err}}}}u \eqsp,
\end{align*}
A  pseudocode of the protocol is given in~\Cref{alg:local-score-error}. Once again, because the target distribution is Gaussian, we quantify the output error using the closed-form Gaussian $\mathcal W_2$ (Bures) distance between the empirical output distribution (via its sample mean and covariance) and $\pi_{\rm data}$. Repeating the experiment over independent replications (fresh draws of the terminal samples and Brownian increments) yields mean and standard deviation curves as functions of $t_{\mathrm{err}}$ and $\lambda$ in Figure \ref{fig:comparison_gaussian_score_error}.

\begin{algorithm}[t]
\caption{Local score-perturbation study (one-step score error at index $n_{\rm err}$).}
\label{alg:local-score-error}
\begin{algorithmic}[1]
\State \textbf{Inputs:} target $\pi_{\rm data}$; reverse-time Euler--Maruyama sampler with step size $h$ and grid $\{t_n\}_{n=0}^{N-1}$; number of points to generate $M$; error indices $\mathcal N_{\rm err}\subset\{0,\dots,N-1\}$; magnitudes $\Lambda\subset\mathbb R_+$; replications $\mathrm{rep}$.
\For{$r=1,\dots,\mathrm{rep}$}
    \State Sample $\{\x_0^{(i)}\}_{i=1}^M \sim \pi_\infty$.
    \State Sample $u^{(r)} \sim \mathrm{Unif}(\mathbb S^{d-1})$.
    \For{each $n_{\rm err}\in\mathcal N_{\rm err}$}
        \For{each $\lambda\in\Lambda$}
            \State Initialize $\x^{(i)}\gets \x_T^{(i)}$ for $i=1,\dots,M$.
            \For{$n=0,\dots,N-1$}
                \State Euler--Maruyama update:
                \begin{align*}
                \x^{(i)}_{n+1} \gets \x^{(i)}_n +
                \Big( \isvp \bwdnoisesch{t_n} \x^{(i)}_n+ 2 \bwdnoisesch{t_n} \tilde s^{(i)}_{T-t_n} (\x^{(i)}_n) \Big)h
                + \sqrt{2 \bwdnoisesch{t_n} } \sqrt{h} \xi_n^{(i)} \eqsp,
                \qquad \xi_n^{(i)}\sim\mathcal N(0,\Id_d) \eqsp,
                \end{align*}
                with
                \begin{align*}
                \tilde s^{(i)}_{T-t_n} (\x^{(i)}_n) \gets \score[T-t_n]{(\x^{(i)}_n)} + \mathbf 1_{\{n=n_{\rm err}\}} \frac{\lambda}{\fwdvar{0}{t_n}} u^{(r)} \qquad \text{ for} i=1,\dots,M \eqsp.
                \end{align*}
            \EndFor
            \State Let $\hat\pi_{N,\lambda,n_{\rm err}}$ be the empirical law of $\{\x^{(i)}_N\}_{i=1}^M$.
            \State Compute $\hat{\mathcal W}_2(\lambda,n_{\rm err}) \gets \mathcal W_2 \left(\hat\pi_{N,\lambda,n_{\rm err}}, \pi_{\rm data}\right)$ 
        \EndFor
    \EndFor
\EndFor
\State \textbf{Plot:} for each $\lambda\in\Lambda$, plot $t\mapsto \frac{1}{\mathrm{rep}}\sum_{r=1}^{\mathrm{rep}} \hat{\mathcal W}_2^{(r)}(\lambda, n_{\rm err})$ (with $\pm$ one std).
\end{algorithmic}
\end{algorithm}

\begin{figure}[t]
\centering
\begin{subfigure}{0.32\textwidth}
    \centering
    \includegraphics[width=\linewidth]{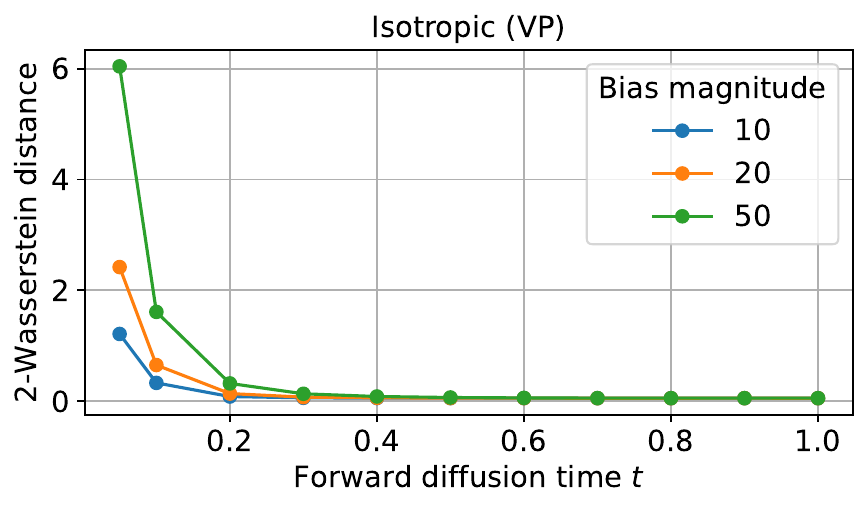}
\end{subfigure}
\hfill
\begin{subfigure}{0.32\textwidth}
    \centering
    \includegraphics[width=\linewidth]{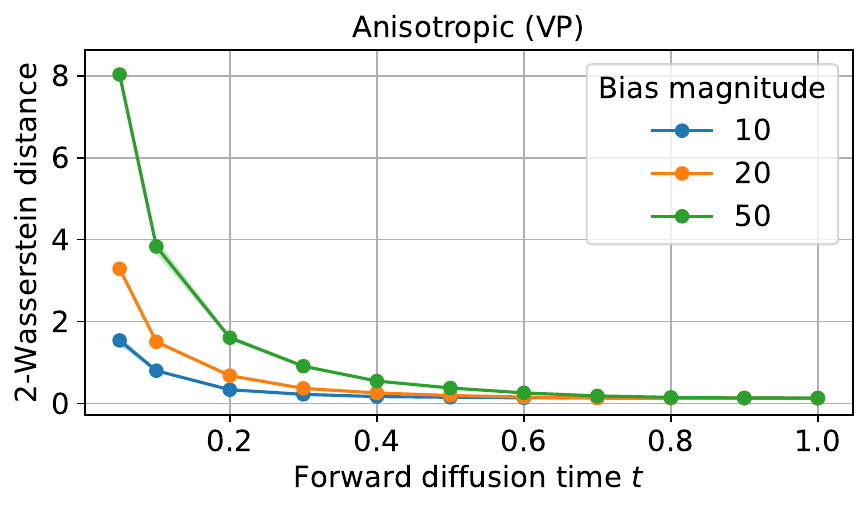}
\end{subfigure}
\hfill
\begin{subfigure}{0.32\textwidth}
    \centering
    \includegraphics[width=\linewidth]{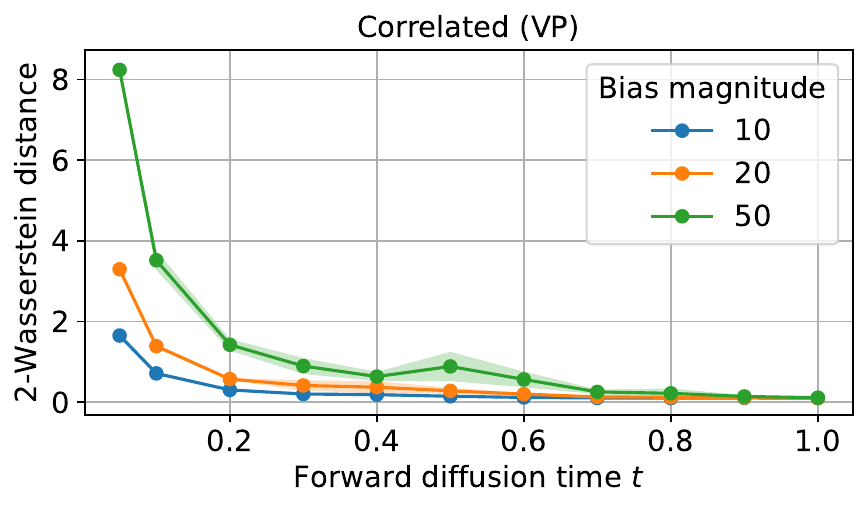}
\end{subfigure}

\vspace{0.3cm}

\begin{subfigure}{0.32\textwidth}
    \centering
    \includegraphics[width=\linewidth]{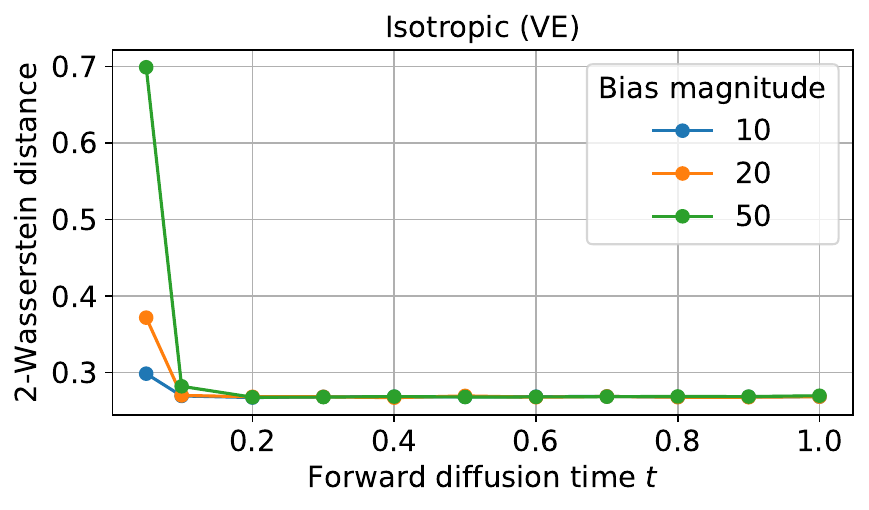}
\end{subfigure}
\hfill
\begin{subfigure}{0.32\textwidth}
    \centering
    \includegraphics[width=\linewidth]{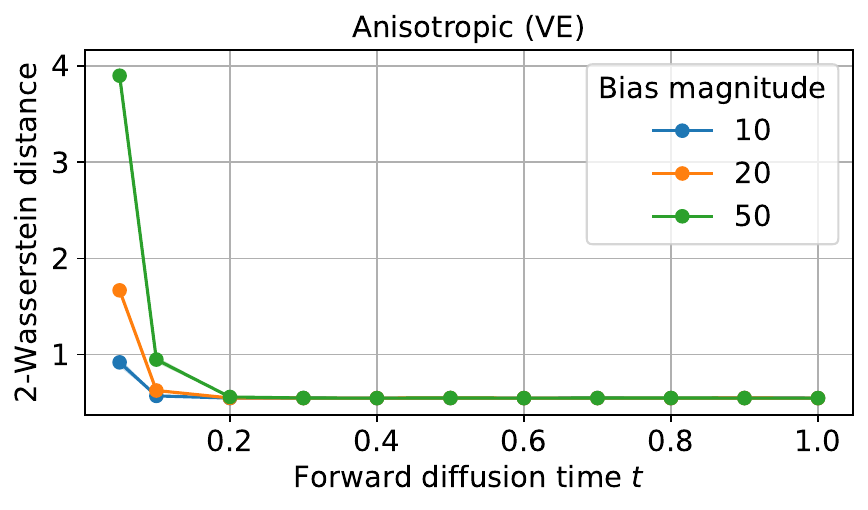}
\end{subfigure}
\hfill
\begin{subfigure}{0.32\textwidth}
    \centering
    \includegraphics[width=\linewidth]{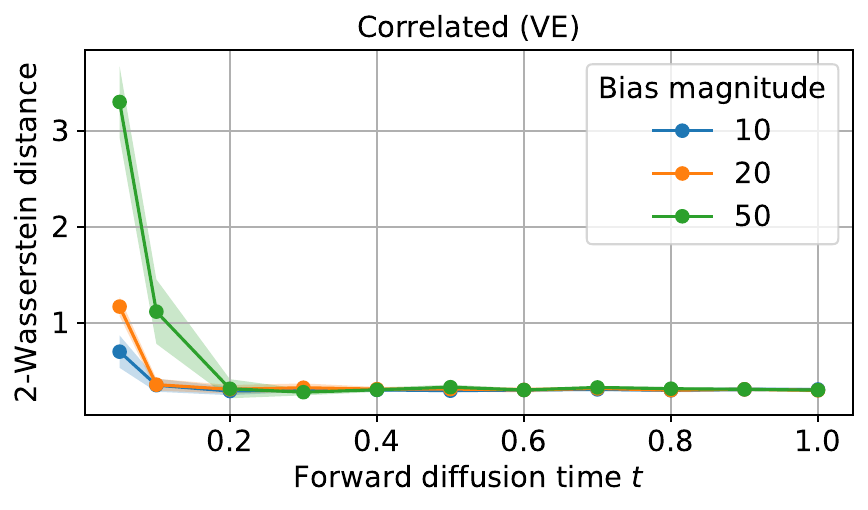}
\end{subfigure}

\caption{Sensitivity to local score error for isotropic, heteroscedastic, and correlated Gaussian targets. Curves report the mean over 5 independent replications of~\Cref{alg:init-bias} ($T=1$, $M=30{,}000$, $h=5\times10^{-3}$); shaded regions indicate $\pm 1$ standard deviation across replications. The horizontal axis is the forward diffusion time $t$. Each dotted point on the lines correspond to a perturbation time and the bias magnitude corresponds to the choice of $\lambda$.}
\label{fig:comparison_gaussian_score_error}
\end{figure}


\paragraph{Mixture of Gaussian targets.}
Once again, the experiment is similar to the one for Gaussian targets. But with the use of the Maximum sliced Wasserstein distance (\maxsw). The rest of the conditions is similar to the previous experiment on initialization bias. Also, in this setting, we consider exclusively the \emph{Variance Exploding} framework and define for $\lambda \in \rsetpos$ and a score perturbation of the type
\begin{align}\label{eq:local-score-perturb:mixt:ve}
\tildescore[t_{{\mathrm{err}}}][\x]
=
\score[t_{{\mathrm{err}}}][\x] +  \frac{\lambda}{\fwdvar{0}{t_{\mathrm{err}}}}u_{\operatorname{max}} \eqsp,
\end{align}
where $u_{\operatorname{max}} \in \sphere{49}$ is obtained as the solution of the optimization problem for the {\maxsw} distance for two sets of $2\times 10^5$ independent samples of $\pdata$. All the calculations have been replicated $20$ times using different seeds and the reported results consist of the mean and the standard deviations of those $20$ replicates. Results are reported in
\Cref{fig:gmm:maxsw} (Right)
and~\Cref{fig:gmm:pert:score}.
\begin{figure*}[t]
    \centering
    \setlength{\tabcolsep}{2pt}
    \renewcommand{\arraystretch}{1}
    \includegraphics[width=0.19\textwidth]{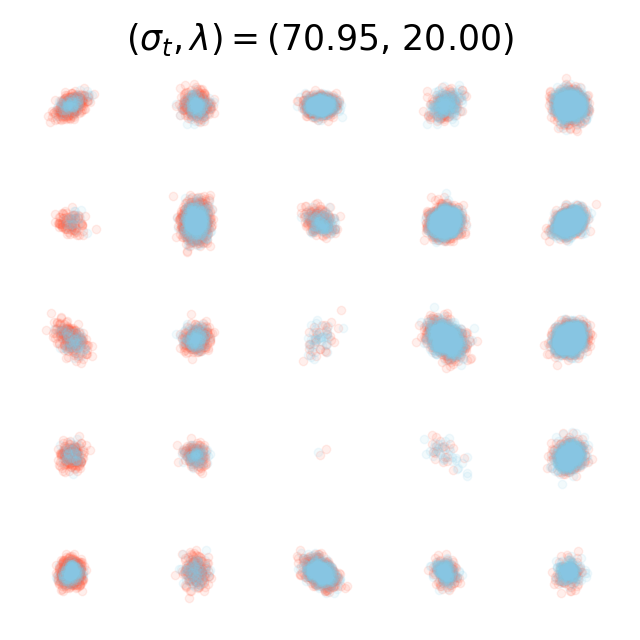}
    \includegraphics[width=0.19\textwidth]{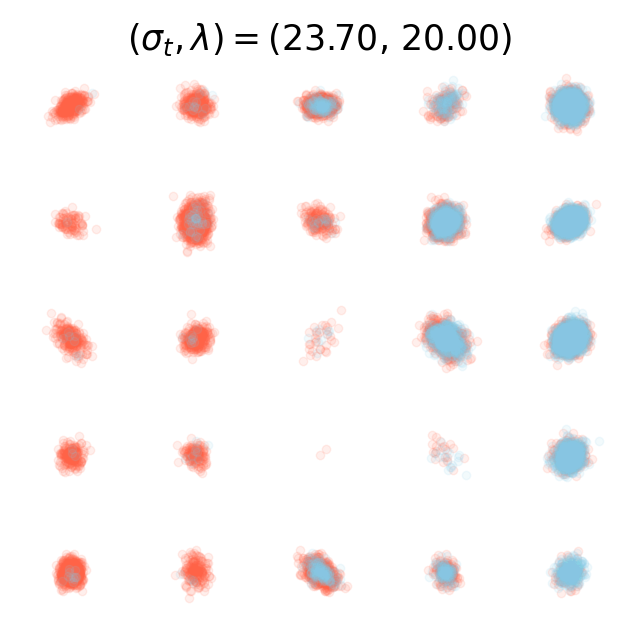}
    \includegraphics[width=0.19\textwidth]{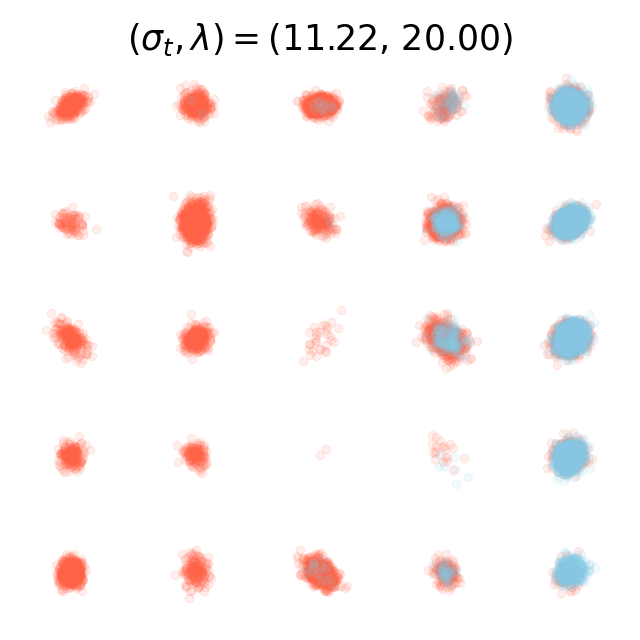}
    \includegraphics[width=0.19\textwidth]{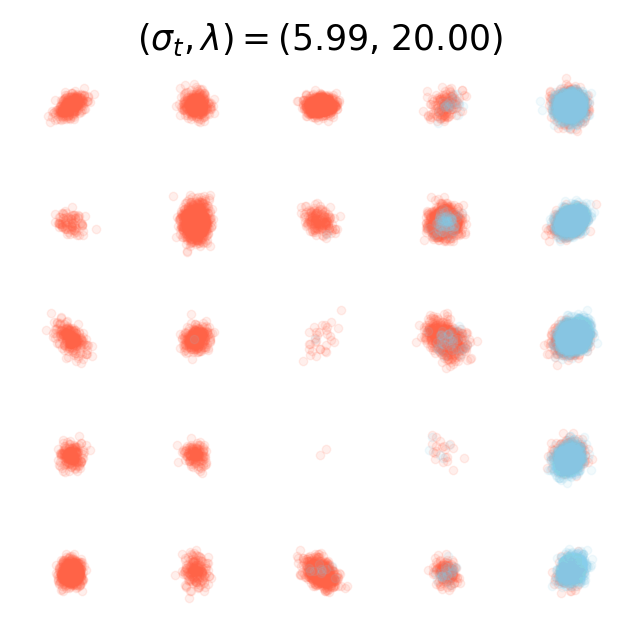}
    \includegraphics[width=0.19\textwidth]{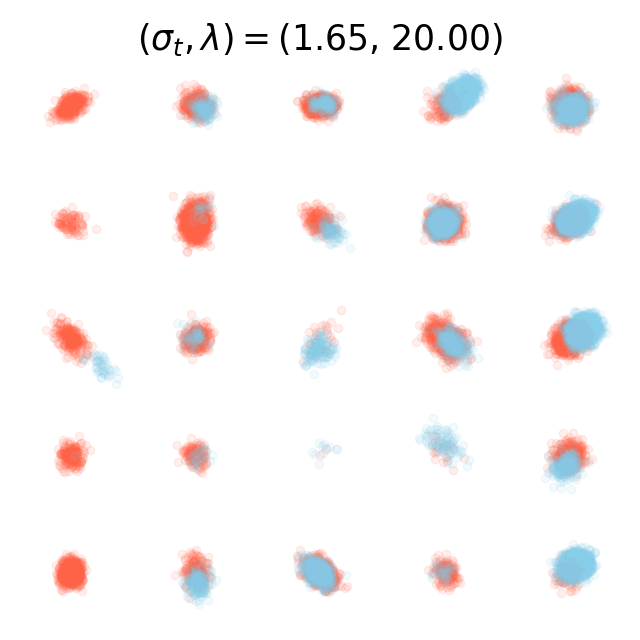}
    \caption{Initialization bias in the GMM case for several noise levels $\sigma_t$ and $\lambda=20$. Red points are samples from $\pidata$ and blue points are samples obtained after the biased initialization experiment.}
    \label{fig:gmm:pert:init}
\end{figure*}

{

\subsection{Real-data illustration on CIFAR-10}
\label{subapp:cifar10}

The experiments in the main text are intentionally conducted in controlled settings, such as Gaussian distributions and Gaussian mixture models, in order to isolate the forgetting effect without confounding it with optimization errors or score-learning errors. To complement these synthetic illustrations, we also run experiments on CIFAR-10 using the pretrained EDM VP model of \citet{karras2022elucidating} with an Euler--Maruyama sampler.

We perturb the denoiser at a single prescribed sampling step and measure the resulting degradation in the generated samples. More precisely, we first fix a perturbation direction in image space by approximately maximizing a sliced-Wasserstein discrepancy between two independently generated batches. We then perturb the denoiser once along this direction, with a prescribed perturbation magnitude, and complete the reverse sampling procedure. We use the convention that small step indices correspond to early reverse times, close to Gaussian initialization, while large step indices correspond to later reverse times, closer to the data distribution. The quality of the generated samples is evaluated using FID and maximum sliced-Wasserstein distance, computed over $50{,}000$ generated samples; for the latter we use $512$ projections.

\begin{table}[htbp]
\centering
\caption{CIFAR-10 denoiser perturbation experiment using the pretrained EDM VP model of \citet{karras2022elucidating} with $100$ Euler--Maruyama steps. The denoiser is perturbed once at the indicated step.}
\label{tab:cifar-perturbation-100}
\begin{tabular}{c|ccccccccc}
\toprule
Step & 0 & 25 & 50 & 70 & 75 & 80 & 85 & 90 & 95 \\
\midrule
FID & 13.3 & 13.0 & 13.1 & 13.6 & 14.4 & 16.4 & 28.3 & 153 & 304 \\
{\maxsw} & 0.011 & 0.014 & 0.016 & 0.033 & 0.044 & 0.060 & 0.094 & 0.266 & 0.779 \\
\bottomrule
\end{tabular}
\end{table}

The results in \Cref{tab:cifar-perturbation-100} show the same qualitative trend as in the controlled experiments: perturbations introduced early in the reverse trajectory have a limited effect on the final samples, whereas perturbations introduced late lead to a sharp degradation in both FID and maximum sliced-Wasserstein distance. This behavior is consistent with the geometric discounting predicted by the stability analysis.

We also checked the robustness of this observation using $200$ Euler--Maruyama steps and several perturbation magnitudes. The corresponding results are reported in \Cref{tab:cifar-perturbation-200}. The number in parentheses denotes the perturbation magnitude.

\begin{table}[htbp]
\centering
\caption{CIFAR-10 denoiser perturbation experiment using $200$ Euler--Maruyama steps and several perturbation magnitudes. The number in parentheses denotes the perturbation magnitude.}
\label{tab:cifar-perturbation-200}
\begin{tabular}{c|cccccc}
\toprule
Step & 0 & 50 & 100 & 150 & 175 & 190 \\
\midrule
FID (100) & 6.1 & 6.4 & 6.4 & 6.2 & 9.2 & 194 \\
FID (150) & 6.1 & 6.4 & 6.4 & 6.8 & 16.8 & 257 \\
FID (200) & 6.1 & 6.3 & 6.3 & 6.8 & 36.6 & 292 \\
{\maxsw} (100) & 0.010 & 0.011 & 0.015 & 0.021 & 0.045 & 0.357 \\
{\maxsw} (150) & 0.010 & 0.017 & 0.024 & 0.027 & 0.076 & 0.520 \\
{\maxsw} (200) & 0.010 & 0.019 & 0.027 & 0.038 & 0.117 & 0.659 \\
\bottomrule
\end{tabular}
\end{table}

These additional experiments confirm the same qualitative phenomenon: errors introduced close to the Gaussian initialization are largely forgotten, whereas errors introduced near the data distribution are much more persistent. This suggests that the forgetting mechanism is not specific to the synthetic distributions used in the main text and remains observable in a realistic image-generation setting.

}

\end{document}